\documentclass[11pt,reqno]{amsart}

\usepackage{verbatim,hyperref}
\usepackage{xurl}

\usepackage{accents,xcolor,soul,bbm,graphicx,MnSymbol,nicefrac,enumitem}
\usepackage[capitalize,nameinlink,noabbrev]{cleveref}

\crefname{enumi}{item}{items}
\crefname{equation}{}{}
\crefname{subsection}{Subsection}{Subsections}
\crefname{figure}{Figure}{Figures}

\usepackage[sort]{cite}

\hypersetup{
    colorlinks,
     linkcolor={blue!80!black},
    citecolor={green},
    urlcolor={blue!80!black}
}

\usepackage{thmtools}

\usepackage[margin=1in]{geometry}
\usepackage[most]{tcolorbox}

\makeatletter
\AtBeginEnvironment{tcolorbox}{%
   \def\@mpfn{footnote}%
}
\makeatother

\newcommand{\red}{\color[rgb]{1,0,0}}

\newcommand{\eps}{\varepsilon}
\newcommand{\dist}{\operatorname{dist}}

\newcommand{\scp}[2]{\langle #1, #2 \rangle}

\newcommand{\cF}{\mathcal F}

\newcommand{\cU}{\mathcal U}
\newcommand{\cV}{\mathcal V}

\newcommand{\grad}{\nabla}

\newcommand{\dd}{\mathrm{d}}

\newcommand{\bas}[1]{\begin{align}\begin{split} #1 \end{split}\end{align}}

\newcommand{\cov}{\mathrm{cov}}

\newcommand{\1}{1\hspace{-0.098cm}\mathrm{l}}

\renewcommand{\P}{{\mathbb P}}

\newcommand{\N}{{\mathbb N}}

\newcommand{\E}{{\mathbb E}}

\newcommand{\R}{{\mathbb R}}

\newcommand{\Hess}{\text{Hess}\,}

\theoremstyle{plain}
\newtheorem{theorem}{Theorem}[section]
\newtheorem{prop}[theorem]{Proposition}
\newtheorem{lemma}[theorem]{Lemma}

\newtheorem{defi}[theorem]{Definition}

\theoremstyle{definition}
\newtheorem{rem}[theorem]{Remark}
\newtheorem{exa}[theorem]{Example}

\makeatletter
\@namedef{subjclassname@2020}{%
	\textup{2020} Mathematics Subject Classification}
\makeatother

\makeatletter
\ExplSyntaxOn
\seq_new:N \g_abbrs
\prop_new:N \g_abbr_counts
\tl_new:N \l_abbr_count_tl

\ExplSyntaxOn

\bool_new:N \g_forexample

\NewDocumentCommand{\eg}{ o }{
	\IfValueT{#1}{
		\str_if_eq:noTF {fe} {#1} {
			\bool_gset_true:N \g_forexample
		} {\bool_gset_false:N \g_forexample}
	}
	\bool_if:nTF { \g_forexample } {
		\bool_gset_false:N \g_forexample
		for~example
	}{
		\bool_gset_true:N \g_forexample
		for~instance
	}
}

\NewDocumentCommand{\abbr}{m m O{#1} m m O{#4} m}{
	\expandafter\newcommand\csname#3\endcsname[1][]{
		\seq_if_in:NnTF \g_abbrs {#1} {
			\prop_get:NnN \g_abbr_counts {#1} \l_abbr_count_tl
			\prop_gput:Nnx \g_abbr_counts {#1} {\int_eval:n {\l_abbr_count_tl + 1}}
			\hyperref[#1]{#7}
		} {
			\seq_gput_left:Nn \g_abbrs {#1}
			\prop_gput:Nnn \g_abbr_counts {#1} {1}
			\expandafter\gdef\csname#1@def\endcsname{#2}
			\phantomsection\label{#1}
			\str_if_eq:nnTF{##1}{}{\emph{#2}}{\emph{##1}}~(\hyperref[#1]{#7})
		}
	}
	\expandafter\newcommand\csname#6\endcsname[1][]{
		\seq_if_in:NnTF \g_abbrs {#1} {
			\prop_get:NnN \g_abbr_counts {#1} \l_abbr_count_tl
			\prop_gput:Nnx \g_abbr_counts {#1} {\int_eval:n {\l_abbr_count_tl + 1}}
			\hyperref[#1]{#4}
		} {
			\expandafter\gdef\csname#1@def\endcsname{#5}
			\seq_gput_left:Nn \g_abbrs {#1}
			\prop_gput:Nnn \g_abbr_counts {#1} {1}
			\phantomsection\label{#1}
			\str_if_eq:nnTF{##1}{}{\emph{#5}}{\emph{##1}}~(\hyperref[#1]{#4})
		}
	}
}

\ExplSyntaxOff

\makeatother
\abbr{ODEs}{ordinary differential equations}{ODE}{ordinary differential equation}{ODEs}
\abbr{SDEs}{stochastic differential equations}{SDE}{stochastic differential equation}{SDEs}
\abbr{GD}{gradient descent}{GDs}{gradient descents}{GD}
\abbr{ANN}{artificial neural network}{ANNs}{artificial neural networks}{ANN}
\abbr{RMSprop}{root mean square propagation \SGD}{RMSprops}{the root mean square propagation \SGD}{RMSprop}
\abbr{PL}{Polyak--\L{}ojasiewicz}{PLs}{Polyak--\L{}ojasiewiczs}{PL}
\abbr{Adam}{adaptive moment estimation \SGD}{Adams}{adaptive moment estimation \SGD}{Adam}
\abbr{AdamW}{\Adam\ with decoupled weight decay}{AdamWs}{\Adam\ with decoupled weight decays}{AdamW}
\abbr{Adagrad}{adaptive gradient \SGD}{Adagrads}{adaptive gradient \SGD}{Adagrad}
\abbr{SOP}{stochastic optimization problem}{SOPs}{stochastic optimization problems}{SOP}
\abbr{OP}{optimization problem}{OPs}{optimization problems}{OP}
\abbr{DNN}{deep artificial neural network}{DNNs}{deep artificial neural networks}{DNN}
\abbr{SGD}{stochastic gradient descent}{SGDs}{stochastic gradient descent}{SGD}
\abbr{iid}{independent and identically distributed}{i.i.d}{independent and identically distributed}{i.i.d.}
\abbr{DL}{deep learning}{DLs}{deep learning}{DL}
\abbr{AI}{artificial intelligence}{AIs}{artificial intelligence}{AI}
\abbr{LLM}{large language model}{LLMs}{large language models}{LLM}
\abbr{PDE}{partial differential equation}{PDEs}{partial differential equations}{PDE}
\abbr{as}{almost surely}{ass}{almost surely}{a.s.}
\abbr{pas}{$\P$-almost surely}{pass}{almost surely}{$\P$-a.s.}
\abbr{ReLUs}{rectified linear unit}{ReLU}{rectified linear unit}
\makeatother

\begin{document}
\renewcommand{\epsilon}{\varepsilon}

\title[RMSprop with full control of hyperparameters]%
{Convergence rates for the RMSprop optimizer\\ with full control of the hyperparameters}

\author[]
{Steffen Dereich}
\address{Steffen Dereich\\
    Institute for Mathematical Stochastics, Faculty of Mathematics and Computer Science,
    University of M\"{u}nster, Germany}
\email{steffen.dereich@uni-muenster.de}

\author[]
{Arnulf Jentzen}
\address{Arnulf Jentzen\\
    School of Data Science and School of Artificial Intelligence,
    The Chinese University of Hong Kong, Shenzhen (CUHK-Shenzhen),
    China \&
    Applied Mathematics: Institute for Analysis and Numerics,
    Faculty of Mathematics and Computer Science,
    University of M\"{u}nster, Germany}
\email{ajentzen@cuhk.edu.cn; ajentzen@uni-muenster.de}


\begin{abstract}
\emph{Stochastic gradient descent} (SGD) optimization methods
are the standard methods to train \emph{artificial intelligence} (AI) systems.
Often the standard SGD method is not the used training algorithm
but instead one considers suitable adaptive variants of SGD
in which the learning rates are for each network parameter separately
adjusted adaptively during the training process.
Popular adaptive SGD methods include the RMSprop, the Adam, and the AdamW optimizers,
where the adaptivity parts in Adam and AdamW basically just coincide with RMSprop.

Such adaptive methods involve several \emph{hyperparameters} including
the regularization parameter $ \varepsilon $
(which ensures that one does not divide by zero during the training process
and is often chosen to be very close to zero such as $ 10^{ - 8 } $ in PyTorch by default)
as well as the second moment decay parameter $ \beta $ (which is often chosen to be very close to $1$
such as $ 0.99 $ (RMSprop) and $ 0.999 $ (Adam and AdamW) in PyTorch by default).
Despite the extensive literature on the theoretical analysis
of such adaptive methods, it remains an open research problem to provide
error estimates for such methods with the error constants
being provably \emph{not exploding but uniformly bounded
with the respect to the hyperparameters} approaching the critical regions, even in the situation of
convex stochastic optimization problems (SOPs).

It is the key contribution of this work to essentially solve this problem for the RMSprop optimizer.
Specifically, our main result provides bounds the expectation of the stopped evaluation
of the objective function at the RMSprop optimization process
from above by the sum of an initialization term that decays exponentially in the training time,
a stochastic approximation remainder of order $ \gamma_n $,
and a memory error of order $ ( 1 - \beta)^2 $
with the error constants being uniformly controlled over all admissible
choices of the step sizes $ ( \gamma_n )_{ n \in \N } $,
the second moment decay parameter $ \beta $
and the regularization parameter $ \varepsilon \in [0,1] $.
In particular, we emphasize that our error analysis also directly
covers the \emph{unregularized regime} in which the regularity parameter vanishes
$ \varepsilon = 0 $.
The non-asymptotic optimization error estimates that we establish
hold not just for all sufficiently large number of gradient steps
but instead hold for every gradient step $ n \in \N = \{ 1, 2, 3, \dots \} $
(starting at the very first gradient step)
with \emph{all error constants being explicitly specified}.

The key innovative new feature in the proof of our error analysis
are suitable \emph{inverse moment bounds} for the second moment process
in the RMSprop method.
We also illustrate the general theory of this work in the case
of several example SOPs.
\end{abstract}

\maketitle
\pagebreak
\tableofcontents
\section{Introduction}

\AI[Artificial Intelligence]\ systems
are trained by means of \SGD\ optimization methods \cite{Ruder2016arXiv}.
The standard \SGD\ optimization method is often not the employed training
scheme for large scale \AI\ systems but instead one often considers
suitable adaptive variants of standard \SGD\ such as the \RMSprop~\cite{TielemanHinton2012},
the \Adam~\cite{KingmaBa2014_Adam}, and the \AdamW~\cite{LoshchilovHutter2017_arXiv} optimizers \cite{JentzenBookDeepLearning2023}.

In the \RMSprop\ method essentially every coordinate of the update is divided by
the square root of a geometric average of the square of all previously seen stochastic gradients
and the adaptivity parts in \Adam\ and \AdamW\ are essentially nothing else but the \RMSprop\ method.
Such adaptive \SGD\ optimization methods are based on several \emph{hyperparameters}
in the optimizers including
\begin{itemize}
\item
the sequence of learning rates $ ( \gamma_n )_{ n \in \N } $ (the sequence of step sizes),
\item
the regularitzation parameter $ \varepsilon $
(which ensures that one does not divide by zero in each update)
as well as
\item
the second moment decay parameter $ \beta $
(which, loosely speaking, models how much weight the recent stochastic gradients
get within the geometric average of the square of all previously seen stochastic gradients).
\end{itemize}
The regularization parameter $ \varepsilon $ is often chosen to be very close to zero
such as $ 10^{ - 8 } $ in PyTorch by default
and the second moment decay parameter $ \beta $ is often chosen to very close to $ 1 $
such as 0.99 (\RMSprop) and 0.999 (\Adam\ and \AdamW) in PyTorch by default.

Despite the already existing extensive literature on the theoretical analysis
of such adaptive \SGD\ optimization methods
and the high importance of such methods,
it remains an open problem of research to prove (or disprove) rigorous error estimates
for such optimizers with the error constants being provably not exploding
but \emph{uniformly bounded with the respect to the hyperparameters} approaching the critical regions,
even in the situation of strongly convex \SOPs.

It is the key contribution of this work to essentially solve this problem for the \RMSprop\ optimizer.
Specifically, in our main result in \cref{thm:combined-early-late-numerical-time} below
we bound the expectation of a stopped evaluation of the objective function
at the \RMSprop\ optimization process (which controls the squared mean square strong optimization error
in the special case of strongly convex \SOPs)
from above by the sum of
\begin{itemize}
 \item
an initialization term that decays exponentially in the training time $ t_n = \sum_{ k = 1 }^n \gamma_k $,
\item
a stochastic approximation remainder of order $ \gamma_n $,
and
\item
a memory error of order $ ( 1 - \beta)^2 $
\end{itemize}
with the error constants being \emph{uniformly controlled} over all admissible
choices of the step sizes $ ( \gamma_n )_{ n \in \N } $,
the second moment decay parameter $ \beta $
and the regularization parameter $ \varepsilon \in [0,1] $.
Particularly, we highlight that our error analysis also
directly covers the \emph{unregularized regime}
in which the regularity parameter completely vanishes
$ \varepsilon = 0 $.

\subsection{The RMSprop method}

In order to formulate \cref{thm:combined-early-late-numerical-time},
we present the considered mathematical setup and we briefly recall the \RMSprop\ method.

Let $ d \in \N = \{ 1, 2, 3, \dots \} $,
let $ ( \mathcal{U}, \mathbb{U} ) $ be a measurable space,
let $ X \colon \R^d \times \mathcal U \to \R^d $ be product measurable,
let $ ( \Omega, \mathcal{F}, \P ) $ be a probability space,
and let $ U $ be an $ \cU $-valued random variable.
We call the pair $ ( X, U ) $ an \emph{innovation}.
The \RMSprop\ algorithm driven by $ (X, U) $ depends on the
following additional parameters:
\begin{enumerate}[label=(\roman*)]
\item $ \beta \in [0,1) $ (\emph{second moment decay parameter}), $ \epsilon \in [0,\infty) $ (\emph{regularization parameter}),
\item a non-increasing $(0,\infty)$-valued sequence $(\gamma_{n})_{n\in\N}$ (\emph{step-sizes}/\emph{learning rates}), and
\item $ \theta_{0}\in\R^{d}$ (\emph{initial state}/\emph{value}).
\end{enumerate}
Let $ U_n $, $ n \in\N $, be independent copies of $ U $
and let $ ( \cF_n )_{ n \in \N_0 } $
be the filtration generated by these copies, that is,
for every $ n \in \N_0 $ let
\bas{
  \cF_n
  =
  \sigma( U_1, U_2, \dots, U_n ) .
}
The associated \RMSprop\ algorithm is the $ ( \cF_n )_{ n \in \N_0 } $-adapted process
$(\theta_n)_{n\in\N_0}$ defined recursively, for every $ n \in \N $, by
\bas{
\label{eq:def_RMSprop}
  \theta_n =
  \theta_{ n - 1 }
  + \gamma_n [ \eps + \sqrt{ v_n } ]^{ - 1 }  X (\theta_{ n - 1 }, U_n ) ,
}
where all operations are applied coordinatewise to vectors and where, for every $ n \in \N $, we have that
\bas{
\label{eq:def_RMSprop_vn}
v_n =\frac{1-\beta}{1-\beta^n}\sum_{k=0}^{n-1} \beta^k X(\theta_{n-1-k}, U_{n-k})^2 .
}
Equivalently, we have for every $ n \in \N $ that
\bas{
\label{eq:v_n_rec}
  v_n=\frac{\beta(1-\beta^{n-1})}{1-\beta^n}v_{n-1}+\frac{1-\beta}{1-\beta^n}X(\theta_{n-1},U_n)^2.
}
where $ \theta_0 $ is the predetermined initial state and where $ v_0 = 0 \in \R^d $.
If $ \eps = 0 $, all coordinatewise normalized innovations are understood with the convention
\bas{
\label{eq:zero-normalization-convention}
  0^{ - 1 } 0 = \frac{0}{0}=0.
}
%
%
We call the mapping $ f_X \colon \R^d \to \R^d $ given, for every $ \theta \in \R^d $, by
$
  f_X(\theta)= \E[X(\theta,U)]
$
the \emph{vector field associated with the innovation $(X,U)$}.
%
%

A key hypothesis in our error analysis is the following \emph{non-degeneracy assumption},
which we formalize in \cref{def:regular} below.
For suitable $ c, q > 0 $ we have for every $ i \in \{ 1, 2, \dots, d \} $ that
the
Laplace transform of $X^{(i)}(\theta,U)^2$ satisfies
\bas{
  \E\bigl[e^{-\lambda X^{(i)}(\theta,U)^2}\bigr]
  \le
  \exp( - \min\{ c \lambda , q \}
  ) .
}
Loosely speaking, this hypothesis ensures for every component of the innovation
that the probability that this component vanishes is not too large.

In our error analysis result for the \RMSprop\ method in \cref{thm:combined-early-late-numerical-time} below
we employ this assumption together with locally $ L^{ \infty } $-bounded innovations and a local
\PL\ inequality (see \cref{eq:combined-main-step-smallness} below).
The additional condition in \eqref{eq:combined-main-step-smallness} prevents abrupt
decreases of the step-size when measured with respect to the training time $ t_n = \sum_{ k = 1 }^n \gamma_k $
for every $ n \in \N_0 $. The condition in \eqref{eq:combined-main-step-smallness}
is satisfied, in particular, in the situation of the usual polynomial learning rate schedules.

\subsection{Main result: Error analysis for RMSprop with full control of the hyperparameters}
\label{ssec:main_result}

The following non-asymptotic error analysis result for the \RMSprop\ method
is a consequence of the general recursive error estimate in \cref{thm:combined-early-late-bounded},
which is stated and proved in \cref{sec:combined-startup-estimate} below.
In \cref{thm:linear-least-squares} in \cref{ssec:example} below we illustrate
the conclusion of \cref{thm:combined-early-late-numerical-time}
in the example of linear least squares (example application of \cref{thm:combined-early-late-numerical-time}).

\begin{tcolorbox}[colback=white!95!gray,
                  colframe=black,
                  boxrule=0.5pt,
                  sharp corners,
                  enhanced,
                  breakable,
                 ]
\begin{theorem}[\textcolor{magenta}{Error analysis for \RMSprop\ with full control of the hyperparameters}]
\label{thm:combined-early-late-numerical-time}
Let $\rho,c,q,C_X,L_f,C_{\mathrm{Loj}},C_{\mathrm{dec}}
\in(0,\infty)$. Set
\bas{\label{eq:combined-main-eta}
\textcolor{magenta}{
  \eta
  =
  [
    2 C_{ \mathrm{Loj} } ( C_X + 1 )
  ]^{ - 1 }
}
  .
}
Then there exist $ C_0, C_1 \in (0,\infty) $ such that the following is true.
Let $ \beta \in [ \nicefrac{ 1 }{ 2 }, 1 ) $,
$ \eps \in [ 0, 1 ] $ satisfy
\bas{
\textcolor{magenta}{
  q+3\ln\beta\ge\rho
}
  .
}
Let $V\subseteq\R^d$ be measurable, let $(X,U)$ be $(c,q)$-regular on
$V$ in the sense of \cref{def:regular} and assume for every
$\theta\in V$, $i\in\{1,\dots,d\}$ that
\bas{\label{eq:combined-main-boundedness}
\textcolor{magenta}{
  \| X^{ (i) }( \theta, U ) \|_{ L^{ \infty } } \le C_X
}
  .
}
Let $ F \colon \R^d \to [0,\infty) $ have a globally $ L_f $-Lipschitz continuous
derivative and assume for every $\theta\in V$ that
\bas{
\textcolor{magenta}{
  f_X(\theta)=-\grad F(\theta)
}
\qquad
\text{and}
\qquad
\textcolor{magenta}{
  F(\theta)\le C_{\mathrm{Loj}}|f_X(\theta)|^2
}
  .
}
Let $ ( \gamma_n )_{ n \in \N } $ be decreasing and $ (0, C_{ \mathrm{Loj} } (C_X + \eps) ] $-valued and
assume for every $ j, n \in \N $ with $ 1 < j \le n $ that
\bas{\label{eq:combined-main-step-smallness}
\textcolor{magenta}{
  \gamma_j
  \le
  \gamma_n
  C_{\mathrm{dec}}
  \exp\bigl(
  \textstyle
    \frac{ \eta }{ 2 } \sum_{ k = j + 1 }^n \gamma_k
  \bigr)
}
  .
}
Let $ \theta_0 \in \R^d $ and let $(\theta_n)_{n\in\N_0}$ be the \RMSprop\
algorithm driven by $(X,U)$ with parameter tuple
$(\beta,\eps,(\gamma_n)_{n\in\N},\theta_0)$ defined above.
Define
$
  \tau = \inf\{ n \in \N_0 \colon \theta_n \notin V \}
$.
Then one has for every $ n \in \N $ that
\bas{
\label{eq:combined-early-late-numerical-time}
\textcolor{magenta}{
  \E\bigl[
    F(\theta_n)
    \1_{ \{ \tau \ge n \} }
  \bigr]
\le
  \exp\bigl(
  \textstyle
    - \eta \sum_{ k = 2 }^n \gamma_k
  \bigr)
  \,
  \bigl[
    F( \theta_0 )^{ 1 / 2 }
    +
    \gamma_1
    ( 2^{ - 1 } L_f d )^{ 1 / 2 }
  \bigr]^2
  + C_0 \gamma_n
  + C_1 ( 1 - \beta)^2
}
  .
}
\end{theorem}
\end{tcolorbox}

\Cref{thm:combined-early-late-numerical-time} follows from
an application of the general recursive error estimate in \cref{thm:combined-early-late-bounded},
which is stated and proved in \cref{sec:combined-startup-estimate} below.
We refer to the end of \cref{sec:combined-startup-estimate}
for the detailed proof of \cref{thm:combined-early-late-numerical-time}.

In \cref{eq:combined-early-late-numerical-time}
we bound the expectation of the $ \tau $-stopped evaluation of the objective function
at the \RMSprop\ optimization process
from above by the sum
\begin{enumerate}[label=(\roman*)]
\item
\label{item:initial_error}
of the initialization term that converges to zero
as the training time $ t_n = \sum_{ k = 1 }^n \gamma_k $
converges to infinity in the number of \RMSprop\ steps $ n $,
\item
\label{item:usual_stochastic_approximation}
of the term $ C_0 \gamma_n $ (the usual stochastic approximation error) that converges to zero
as the number of \RMSprop\ steps $ n $ goes to infinity,
and
\item
\label{item:memory_term}
of the term $ C_1 ( 1 - \beta )^2 $ that converges to zero as the second moment decay parameter $ \beta $
approaches $ 1 $.
\end{enumerate}
The convergence speed $ C_0 \gamma_n $ according to \cref{item:usual_stochastic_approximation}
is the usual stochastic approximation error term that also appears in the mean square error analysis of the standard
\SGD\ method and is optimal and cannot be improved (cf., \eg, \cite{MR4055054}).
In the case of \Adam\ numerical simulations and analytical considerations
suggest that the power two in the error term $ C_1 ( 1 - \beta )^2 $
can also not be improved (cf., \eg, \cite{Dereichetal2025_Adam_symmetry_theorem_arXiv})

We highlight that, except of $ C_0 $ and $ C_1 $,
all constants in \cref{thm:combined-early-late-numerical-time}
are explicitly specified and, actually, even $ C_0 $ and $ C_1 $ can
be explicitly specified (see \cref{rem:explicit_constants_II} in \cref{sec:combined-startup-estimate} below).
In particular, \cref{rem:explicit_constants_II}
shows that $ C_0 $ and $ C_1 $ in \cref{thm:combined-early-late-numerical-time}
can be chosen as
\bas{
\label{eq:C_0_bound_from_above}
  C_0
  =
  2^{ 13 }
  C_{ \mathrm{dec} } d
  ( L_f + C_X + 1 )^7
  ( \eta^{-1} + 1 )
  \biggl[
    \sum_{ r = 0 }^{ \infty }
      \frac{
        ( c^{ - 3 } + 1 )
        (r+1)^3
      }{
        \exp( \min\{ \rho, 2^{ - 4 } q \} r )
      }
  \biggr]
}
and
\bas{
\label{eq:C_1_bound_from_above}
  C_1
  =
      2^{ 13 }
      d
      L_f C_X^6
      \eta^{ - 2 }
  \biggl[
    \sum_{ r = 0 }^{ \infty }
    \frac{
      c^{ - 3 }
      (r+1)^3
    }{
      \exp( \rho r )
    }
  \biggr]
  .
}
We refer to \cref{eq:C_0_rep,eq:C_1_rep,eq:C_0_bound_from_above_B}
in \cref{rem:explicit_constants_II} for details.
We also emphasize that all of the error constants
in \cref{eq:combined-early-late-numerical-time} are completely
independent of $ \beta \in [ \nicefrac{1}{2}, 1 ) $
and also completely independent of $ \varepsilon \in [0,1] $
and we highlight that
even in the \emph{unregularized regime} $ \varepsilon = 0 $ the conclusion
of \cref{eq:combined-early-late-numerical-time} remains valid
under the zero-over-zero convention in \cref{eq:zero-normalization-convention} above.

In \cref{thm:combined-early-late-numerical-time}
the initial value of the \RMSprop\ optimization process
is a purely deterministic vector $ \theta_0 \in \R^d $
instead of an $ \R^d $-valued random variable.
However,
under the assumption that this $ \R^d $-valued random variable
is independent from the driving random variables $ U_n $, $ n \in \N $,
we have that
\cref{thm:combined-early-late-numerical-time}
directly implies the corresponding statement
but with a stochastic initial state of the \RMSprop\ process
as the error constants and parameters
$ \eta $, $ C_0 $, $ C_1 $, $ L_f $, and $ d $
in \cref{eq:combined-early-late-numerical-time}
are all completely independent of the initial value $ \theta_0 $,

We also point out that in
\cref{thm:combined-early-late-numerical-time}
we do not require that the learning rates $ \gamma_n $, $ n \in \N $,
converge to zero and the inequality in \cref{eq:combined-early-late-numerical-time}
holds true even if the initial learning rates $ ( \gamma_n )_{ n \in \N } $ are bounded from above by $ C_{ \mathrm{Loj} } C_X $
and non-increasing,
even though it is well-known that \RMSprop\ and other \SGD\ optimization methods
fail to converge if the learning rates do not converge to zero
\cite{DereichGraeberJentzen2024arXiv_non_convergence}.

We also note that the error estimate in
\cref{eq:combined-early-late-numerical-time} in
\cref{thm:combined-early-late-numerical-time}
is restricted to the event $ \{ \tau \geq n \} $
and is thus only applicable until
the optimization process $ ( \theta_n )_{ n \in \N_0 } $
has left the measurable set $ V $.
We think of the set $ V $ as a sufficiently large
compact ball
(see, \eg, \cref{eq:linear-least-squares-rate} in \cref{thm:linear-least-squares} below)
and we refer, \eg, to \cite{DereichDoJentzen2026arXiv} and \cite[Theorem~2.10]{DereichGraeberJentzenRiekert2026_arXiv}
for results in the literature that establish a priori bounds for \RMSprop\ and
other related adaptive \SGD\ optimization methods (particularly \Adam)
which ensure that the considered optimization process remains in a sufficiently large compact ball.

\subsection{Example: RMSprop for least squares with full control of the hyperparameters}
\label{ssec:example}
In the following result, \cref{thm:linear-least-squares} below,
we illustrate the conclusion of \cref{thm:combined-early-late-numerical-time} above
in the example of linear least squares.
In \cref{sec:regularity-examples} below we show in detail how \cref{thm:combined-early-late-numerical-time}
can be applied to prove \cref{thm:linear-least-squares}.

\begin{samepage}
\begin{tcolorbox}[colback=white!95!gray,
                  colframe=black,
                  boxrule=0.5pt,
                  sharp corners,
                  enhanced,
                  breakable,
                 ]
\begin{theorem}[\textcolor{magenta}{\RMSprop\ for linear least squares}]
\label{thm:linear-least-squares}
Let $ ( \Omega, \cF, \P ) $ be a probability space, let
$ \kappa \in (0,\infty) $,
$ m, d, M \in \N $,
$ \vartheta \in \R^d $,
let $ A \in \R^{ m \times d } $ have no zero columns,
let $ K \subseteq \R^m $ be compact,
let $ ( Y_{ n, r } )_{ (n,r) \in \N^2 } $ be $ K $-valued i.i.d.\ random variables with a bounded Lebesgue density,
let $ L \colon \R^d \times \R^m \to \R $ satisfy for all $\theta\in\R^d$, $y\in\R^m$ that
\bas{
\label{eq:linear-least-squares-loss}
\textcolor{magenta}{
  L( \theta, y )
  = | A \theta - y |^2
}
}
and
let $ F \colon \R^d \to \R $
satisfy for all $ \theta \in \R^d $ that
$
  F(\theta)
=
  \E[L(\theta,Y_{1,1})]
  - \inf_{\vartheta\in\R^d}\E[L(\vartheta,Y_{1,1})]
$.
Then there exists $ \eta \in (0,\infty) $
such that for every
$ \mathfrak{C} \in (0,\infty) $ there exists
$ c \in \R $ such that
for all $ \beta \in [ \nicefrac{ 1 }{ 2 }, 1 ) $,
$ \eps \in [0,1] $, all $ \R^d $-valued
stochastic processes $ ( w_n )_{ n \in \N_0 } $ and $ ( \theta_n )_{ n \in \N_0 } $,
and every $ ( 0, \mathfrak{C} ] $-valued decreasing $ ( \gamma_n )_{ n \in \N } $
with the property that for every
$ n \in \N $,
$ j \in \N \cap [ 1, n ] $,
$ i \in \{ 1, 2, \dots, d \} $
it holds that
\begin{equation}
\label{eq:linear-least-squares-rmsprop}
\begin{gathered}
\textcolor{magenta}{
  w_0 = 0
}
  ,
  \qquad
\textcolor{magenta}{
  w_n^{(i)}
  = \beta w_{n-1}^{(i)}+(1-\beta)
  \bigl[
  \textstyle
    M^{ - 1 } \sum_{r=1}^M
    ( \nabla_{ \theta_i } L )( \theta_{ n - 1 }, Y_{ n, r } )
  \bigr]^2
}
  ,
\\
\textcolor{magenta}{
  \theta_0 = \vartheta
} ,
\qquad
\textcolor{magenta}{
  \theta_n^{(i)}
  =
  \theta_{n-1}^{(i)}
  -
  \gamma_n
  \bigl[
    \eps + ( 1 - \beta^n )^{ - 1 / 2 } ( w_n^{ (i) } )^{ 1 / 2 }
  \bigr]^{ - 1 }
  \textstyle
    M^{ - 1 }
    \sum_{r=1}^M
    ( \nabla_{ \theta_i } L )( \theta_{ n - 1 }, Y_{ n, r } )
}
\end{gathered}
\end{equation}
and
\textcolor{magenta}{$
  \gamma_j
  \le  \mathfrak{C} \exp( \frac\eta2 \sum_{ k = j+1 }^{ n } \gamma_k )
   \gamma_n
$}
we have for every $ n \in\N $ that
\bas{\label{eq:linear-least-squares-rate}
\textcolor{magenta}{
  \E\bigl[
    F(\theta_n)
    \1_{ \cap_{ j = 0 }^{ n - 1 } \{|\theta_j|\le\kappa\}}
  \bigr]
  \le
  \exp\bigl(
  \textstyle
    - \eta \sum_{ k = 2 }^n \gamma_k
  \bigr)
  \,
    \bigl(
      \gamma_1
      ( d \| A^\mathsf{T} A \| )^{ 1 / 2 }
      +
      | F( \theta_0 ) |^{ 1 / 2 }
    \bigr)^2
  + c \gamma_n
  + c (1-\beta)^2
}
  .
}
\end{theorem}
\end{tcolorbox}
\end{samepage}

In \cref{sec:regularity-examples} we rigorously show how \cref{thm:linear-least-squares}
can be deduced from \cref{thm:combined-early-late-numerical-time} in \cref{ssec:main_result} above.
In addition, in \cref{sec:regularity-examples} we also
illustrate the conclusion of \cref{thm:combined-early-late-numerical-time}
in the case of two further classes of example \SOPs, namely,
risk-sensitive exponential tilting with coordinate gradients that are strongly convex in the data variable
and compactly supported nonlinear regression with vector-valued responses.

We point out that
in \cref{thm:linear-least-squares}
we do not assume that $ A^{ \mathsf{T} } A $ is invertible
and we note that the objective function
$ F \colon \R^d \to \R $
in \cref{thm:linear-least-squares}
is convex but not necessarily \emph{strongly convex}.

\subsection{Short literature review and key novelty of this work}
\label{sec:literature_review}

Nowadays there are a large number of research works
that study adaptive gradient based optimization methods theoretically.
In particular,
\begin{itemize}

\item
we refer, \eg, to
\cite[Theorem~3.1]{DeMukherjeeUllah2018_arXiv},
\cite[Subsection~2.4]{GadatGavra2022_MR4577667},
\cite[Subsection~6.4]{JinWang2026_arXiv},
\cite[Theorem~3.6 and Corollary~3.7]{LiuXuZhangMandic2024},
\cite[Subsection~4.2]{ShiLiHongSun2021},
\cite[Theorem~4.1]{XuZhangZhangMandic2021},
\cite[Theorem~1, Corollary~2, and Corollary~3]{ZaheerReddiSachanKaleKumar2018},
\cite[Theorem~1 and Theorem~3]{ZhangZhouZou2025},
\cite[Sections~4 and 5]{ZhouChenCaoYangGu2024},
and
\cite[Section~3]{ZouShen2019}
for error, complexity, and convergence analyses for the \RMSprop\ method
and minor modifications of it when applied to \SOPs,

\item
we refer, \eg, to \cite{DereichJentzenRiekert2025_arXiv,
DimitrieskiHoneckerSchererEbenbauer2026_arXiv,
IbragimovJentzen2026arXiv,
BensaidPoetteTurpault2024_arXiv,
MaWuE2020_arXiv_qualitative_behavior_RMSprop}
for works that study \RMSprop\
in the deterministic/full-batch setting,
and

\item
we refer, \eg, to \cite{MR5104936,MR4949951,DoJentzenRiekert2025_arXiv}
for lower bounds and non-convergence results to global minimizers
for \RMSprop\ in the training of \ANNs.

\end{itemize}
For online learning problems we also refer, \eg, to \cite[Theorem~4.1, Section~4, and Assumptions~A1--A2]{MukkamalaHein2017_arXiv}
for upper bounds for the (averaged) regret of a modified variant of the \RMSprop\ method
in which the second moment parameter $ \beta $ is not constant but converging with appropriate speed of convergence to $ 1 $ during the training procedure
and in which the regularization parameter $ \varepsilon $ is not constant but converging with appropriate speed of convergence to $ 0 $
during the training procedure.
We also refer, \eg, to \cite[Theorem~4.2, Theorem~C.2, and Theorem~C.5]{MalladiLyuPanigrahiArora2022_arXiv}
for approximation results for \RMSprop\ applied to \SOPs\ using solution processes of \SDEs.

Furthermore, we refer, \eg, to \cite{ReddiKale2019,Defossez2022,ChenLiuSunHong2019,BarakatBianchi2021_MR4199255,HongLin2024_Adam,
Dereichetal2025_Adam_symmetry_theorem_arXiv,DereichAdamconvergence2024,LiRakhlinJadbabaie2024_arXiv,Wangetal2023_arXiv,
Zhangetal2022_arXiv,YuChenFeng2026_arXiv_Adam-SHANG}
for minimized/ergodic/non-ergodic error, gradient, and complexity estimates for the \Adam\ optimizer (which is basically nothing
else but the combination of \RMSprop\ and momentum)
and modified variants of it
and we refer, \eg, to \cite{ReddiKale2019,Dereichetal2025_Adam_symmetry_theorem_arXiv,
HeilmanMohanty2026_arXiv_non_convergence_Adam}
for non-convergence results for \Adam\ due to an \emph{internal error} of the optimizer.
\RMSprop\ also admits such an \emph{internal error} and this is the reason
why the additional error term $ C_1 ( 1 - \beta )^2 $ in \cref{eq:combined-early-late-numerical-time}
in \cref{thm:combined-early-late-numerical-time} above can, in general, not be avoided.
We also refer, \eg, to \cite{GodichonBaggioni2023}
and the references therein for error estimates for other related adaptive \SGD\ methods
such as the \Adagrad\ optimizer \cite{MR2825422}.
Further works on the theoretical analysis of \SGD\ optimization methods
can, \eg, also be found in \cite[Section~1.3]{Dereichetal2025_Adam_symmetry_theorem_arXiv}.

The {\bf key innovative new contribution} of this work is
that we establish convergence rates
and error estimates for \RMSprop\ \emph{with full control of the hyperparameters}.
Specifically, to the best of our knowledge,
\cref{thm:combined-early-late-numerical-time}
is the first result in the scientific literature
that establishes convergence of the \RMSprop\ method
for \SOPs\ \emph{with the error constants being uniformly controlled with respect to the
hyperparameters}, particularly, with respect to the
second moment decay parameter $ \beta $
and
the regularization parameter $ \varepsilon $.
This work even covers the completely \emph{unregularized case} $ \varepsilon = 0 $.

This kind of error analysis is very relevant for the application of adaptive optimizers
in practically relevant regimes as if the error constants in error estimates
are not uniformly bounded with respect to the hyperparameters but instead grow,
for example, quadratically in the reciprocal $ \varepsilon^{ - 1 } $ of the regularization parameter
$ \varepsilon $ and if the regularization parameter is chosen to be $ \varepsilon = 10^{ - 8 } $
(default value in PyTorch), then the error constant would be as large as $ 10^{ 16 } $ and
would make the error estimate essentially not useable in practically relevant regimes.

On a technical level, the key innovative contribution of this work -- which allows us to prove convergence rates
with full control of the hyperparameters -- is to establish and employ new upper bounds
for the \emph{inverse moments} of the
$ i $-th component of the second moment process $ ( v_n )_{ n \in \N_0 } $ in \cref{eq:v_n_rec}
in the \RMSprop\ method
at time $ n \in \N $
after the second moment process has become sufficiently large, that is,
restricted to the event
$ \{ \sigma^{ (i) } \leq n \} $
where $ \sigma^{ (i) } = \inf\{ n \in \N \colon v^{ (i) }_n \geq \iota \} $
for an arbitrary given real constant $ \iota \in (0,\infty) $.
Such {\bf inverse moment estimates} are established in \cref{sec:negative-moments}
(see \cref{le:neg_mom_1} and \cref{le:moment_bound} in \cref{sec:negative-moments})
and are then used extensively in the proof of \cref{thm:combined-early-late-bounded} in \cref{sec:combined-startup-estimate}
and, thereby, also in the proof of \cref{thm:combined-early-late-numerical-time}
(see the end of \cref{sec:combined-startup-estimate}).

Such \emph{inverse moment bounds} for $ ( v_n )_{ n \in \N_0 } $
cannot in general hold on the whole probability space:
before a component of $ v_n $ has reached a \emph{non-degenerate level}
(that is, has become greater than or equal to some sufficiently large real threshold $ \iota \in (0,\infty) $),
it may vanish and its inverse moments may explode and fail to exist.
We therefore introduce the random, coordinatewise startup time $ \sigma^{ (i) } $
at which this level is first reached (at which the
$ i $-th component of the second moment process $ ( v^{ (i) }_n )_{ n \in \N_0 } $
has become greater than or equal to a predetermined real threshold $ \iota $).

On the complement of the event
$
  \{ \sigma^{ (i) } \leq n \}
$
we use elementary estimates for $ F( \theta_n ) \mathbbm{1}_{ \{ \tau \geq n \} } $
and establish that the probability
of the event
$
  ( \Omega \backslash \{ \sigma^{ (i) } \leq n \} )
  \cap
  \{ \tau \geq n \}
  =
  \{ \tau \geq n, \sigma^{ (i) } > n \}
$
is exponentially small in $ n $
(see \cref{prop:3476} in \cref{sec:initialisation} below).

Performing this early-late decomposition separately for every coordinate and at every iteration produces a single
recursion that is valid from the first step on. To the best of our knowledge, this combination of
\emph{post-startup inverse moment estimates} with a \emph{different pre-startup argument} is a new \emph{proof mechanism} in the analysis of
stochastic gradient methods.

\subsection{Overview and structure of this work}
\label{sec:structure}

The remainder of the article is organized as follows.
\begin{itemize}
\item
In \cref{prop:F-recursion-md-loj-event} in \cref{sec:abstract-recursion} we derive
under the abstract assumptions in \cref{item:i_Prop21,item:ii_Prop21,item:iii_Prop21}
a one-step recursion for the scalar quantity
$ \psi_n = \E[ F( \theta_n ) \mathbbm{1}_{ \{ \tau \geq n \} } ] $
(expectation of the $ \tau $-stopped evaluation of the objective function $ F \colon \R^d \to [0,\infty) $
at the \RMSprop\ optimization process $ \theta_n $ after $ n $ steps)
for $ n \in \N $.
In \cref{prop:F-recursion-md-loj-event} the events $ G_{n,i} \in \cF_{ n - 1 } $, $ n \in \N $, $i\in\{1,\ldots,d\}$,
are general predictable events but in our later application of
\cref{prop:F-recursion-md-loj-event} the event $ G_{n,i} $ will
coincide with the event $ \{ \sigma^{ (i) } < n  \} $
where $ \sigma^{ (i) } = \inf\{ n \in \N \colon v^{ (i) }_n \geq ( \kappa_0 )^{ - 1 } 2 c \} $
is the stopping time describing the first time where the $ i $-th component of the second moment process $ ( v^{ (i) }_n )_{ n \in \N_0 } $
in the \RMSprop\ method is sufficiently large (greater than or equal to the constant $ 2 c ( \kappa_0 )^{ - 1 } $).
\item
In \cref{sec:analysis-prop-iii} we show that suitable inverse moment bounds
for the $ i $-th component of the second moment process $ ( v^{ (i) }_n )_{ n \in \N_0 } $
in the \RMSprop\ method (see \cref{eq:inverse_moment_assumption}
in \cref{prop:prop_iii} in \cref{sec:analysis-prop-iii})
are sufficient to ensure that
the property in \cref{item:iii_Prop21} in \cref{prop:F-recursion-md-loj-event}
is satisfied
(estimates for the \emph{covariance correction}).
\item
In \cref{sec:negative-moments} we then establish such \emph{inverse moment bounds} restricted to the events $ G_{n,i} $, $ n \in \N $, $i\in\{1,\ldots,d\}$
(on the complement $ \Omega \backslash G_{n,i} $ of the event $ G_{n,i} $ the inverse moment bounds do in general \emph{not} hold).
\item
In \cref{sec:technical-estimates} we provide sufficient conditions
to ensure that the property in \cref{item:ii_Prop21} of
\cref{prop:F-recursion-md-loj-event} is satisfied
(estimates for the \emph{quadratic Taylor remainder}).
\item
In \cref{sec:initialisation} we show that the probability of the event $ ( \Omega \backslash G_{n,i} ) \cap \{ \tau \geq n \} $
is exponentially small in $ n $.
\item
In \cref{sec:combined-startup-estimate} we combine
\cref{prop:F-recursion-md-loj-event} from \cref{sec:abstract-recursion}
with the estimates from \cref{sec:analysis-prop-iii,sec:negative-moments,sec:technical-estimates,sec:initialisation}
to establish the key error estimate in \cref{thm:combined-early-late-bounded}.
In addition, in \cref{sec:combined-startup-estimate}
we also apply \cref{thm:combined-early-late-bounded}
to deduce \cref{thm:combined-early-late-numerical-time}
in this introductory section.
In the proof of \cref{thm:combined-early-late-bounded} we apply
\cref{prop:F-recursion-md-loj-event} with the coordinate events $G_{n,i}$.
The contributions on their complements are controlled by the estimates from
\cref{sec:initialisation}, while the findings of
\cref{sec:analysis-prop-iii,sec:negative-moments,sec:technical-estimates}
ensure that the assumptions in \cref{item:ii_Prop21,item:iii_Prop21} of
\cref{prop:F-recursion-md-loj-event} are satisfied.

%
\item
In \cref{sec:regularity-examples}
we illustrate \cref{thm:combined-early-late-numerical-time}
by means of several examples.
In particular,
in \cref{sec:regularity-examples}
we show how \cref{thm:combined-early-late-numerical-time}
can be used to prove \cref{thm:linear-least-squares}.

\end{itemize}


\subsection{Use of large language models}

The key ideas of the statements and the proofs of the main results of this work are due to the authors.
{\sc GPT 5.5} and {\sc GPT 5.6} have been substantially employed to formulate and work out the statements
and the proofs in {\sc LaTeX} according to the strategies of the authors. {\sc GPT 5.5} and {\sc GPT 5.6}
have also supported us in creating the review in \cref{sec:literature_review} and the explanatory texts in the
remainder of this work. The authors take full responsibility for each of the statements/arguments/sentences made in this work.

\section{The abstract recursion estimate in \cref{prop:F-recursion-md-loj-event}}
\label{sec:abstract-recursion}
This section derives an abstract one-step recursion for the stopped Lyapunov
values $ F( \theta_n ) $. Starting from the smoothness inequality, we separate the
coercive drift, the covariance correction caused by the dependence of $v_n$ on
the current innovation, and the quadratic Taylor remainder. The resulting
scalar recursion reduces the proof of the main result to estimates for these
three contributions.

\begin{prop}
\label{prop:F-recursion-md-loj-event}
Let $ L_f, c_v, c_{\mathrm{Loj}}, C_{\mathrm{Loj}} \in (0,\infty) $, $ n_0 \in \N $,
let $ ( C_X( n ) )_{ n > n_0 } $
and $ ( C_\beta( n ) )_{ n > n_0 } $
be $ (0,\infty) $-valued sequences
and let $ V \subseteq \R^d $ be measurable.
Define
\bas{
  \tau := \inf\{ n \in \N \cap [ n_0, \infty ) \colon \theta_n \notin V \} .
}
For every $ n \in \N \cap ( n_0, \infty ) $, $ i \in \{ 1, \dots, d \} $
let $ G_{ n, i } \in \cF_{ n - 1 } $.
Suppose that $ F \colon \R^d \to [0,\infty) $
has an $ L_f $-Lipschitz continuous derivative and assume for all $ \theta \in V $ that
\bas{
  f_X( \theta ) = - \grad F( \theta )
  \qquad
  \text{and}
  \qquad
  c_{ \mathrm{Loj} } | f_X( \theta ) |^2 \le F( \theta ) \le C_{ \mathrm{Loj} } | f_X( \theta ) |^2 .
}
Assume moreover that all quantities appearing below are well-defined and
integrable to the powers used there on the
corresponding stopped events.
If $ \eps = 0 $, inverse factors and conditional covariances appearing together
with $ \1_{ G_{n,i} } $ are evaluated on
$ \{ \tau \ge n \} \cap G_{n,i} $ and extended by zero to its complement.
Assume that, for every $ n \in \N \cap ( n_0, \infty ) $, the following
conditions\footnote{Note that for every probability space $ ( \Omega, \mathcal{F}, \P ) $,
every sigma-algebra $ \mathcal{G} \subseteq \mathcal{F} $ on $ \Omega $ and
all square integrable random variables $ X \colon \Omega \to \R $
and $ Y \colon \Omega \to \R $ it holds $ \P $-a.s.\ that
$
  \cov( X, Y | \mathcal{G} )
  =
  \E[
    ( X - \E[ X | \mathcal{G} ] )
    ( Y - \E[ Y | \mathcal{G} ] )
    | \mathcal{G}
  ]
$.} hold:
\begin{enumerate}[label=(\roman*)]
\item
\label{item:i_Prop21}
$
  \E[\1_{\{\tau\ge n\}}\sum_{i=1}^d \1_{G_{n,i}}(\sqrt{v_n^{(i)}}+\eps)^{-1}|f_X^{(i)}(\theta_{n-1})|^2]
  \ge c_v\,\E[\1_{\{\tau\ge n\}}\sum_{i=1}^d \1_{G_{n,i}}|f_X^{(i)}(\theta_{n-1})|^2]
$,
\item
\label{item:ii_Prop21}
$
  \E[\1_{\{\tau\ge n\}}\sum_{i=1}^d \1_{G_{n,i}}(\sqrt{v_n^{(i)}}+\eps)^{-2}X^{(i)}(\theta_{n-1},U_n)^2]\le C_X(n)
$,
and
\item
\label{item:iii_Prop21}
$
  \E[\1_{\{\tau\ge n\}}\sum_{i=1}^d \1_{G_{n,i}}\cov((\sqrt{v_n^{(i)}}+\eps)^{-1},X^{(i)}(\theta_{n-1},U_n)|\cF_{n-1})^2]^{1/2}\le C_\beta(n)
$.
\end{enumerate}
For every $ n \in \N \cap ( n_0, \infty ) $, set
\bas{\label{eq:F-recursion-late-definitions}
  \psi_n =
  \E\bigl[
    \1_{ \{ \tau \ge n \} } F( \theta_n )
  \bigr]
  \qquad\text{and}\qquad
  \widetilde\psi_n =
  \E\bigl[
    \1_{ \{ \tau \ge n \} } F( \theta_{n-1} )
  \bigr].
}
Then one has for every $ n \in \N \cap ( n_0, \infty ) $ that
\bas{\label{eq:F-recursion-late-square-root}
\psi_n
&\le \Bigl(1-\frac{c_v}{C_{\mathrm{Loj}}}\gamma_n\Bigr)\,\widetilde\psi_n+\frac{C_\beta(n)}{c_{\mathrm{Loj}}^{1/2}}\gamma_n\widetilde\psi_n^{1/2} + \frac{L_fC_X(n)}{2}\gamma_n^2\\
&\quad+c_v\gamma_n\E\Bigl[\1_{\{\tau\ge n\}}\sum_{i=1}^d\1_{G_{n,i}^c}|f_X^{(i)}(\theta_{n-1})|^2\Bigr]
+\gamma_n\E\Bigl[\1_{\{\tau\ge n\}}\sum_{i=1}^d\1_{G_{n,i}^c}|f_X^{(i)}(\theta_{n-1})(\sqrt{v_n^{(i)}}+\eps)^{-1}X^{(i)}(\theta_{n-1},U_n)|\Bigr]\\
&\quad+\frac{L_f\gamma_n^2}{2}\E\Bigl[\1_{\{\tau\ge n\}}\sum_{i=1}^d\1_{G_{n,i}^c}(\sqrt{v_n^{(i)}}+\eps)^{-2}X^{(i)}(\theta_{n-1},U_n)^2\Bigr].
}
In particular, one has for every $ n \in \N \cap ( n_0, \infty ) $ that
\bas{\label{eq:F-recursion-late-young}
\psi_n
&\le \Bigl(1-\frac{c_v}{2C_{\mathrm{Loj}}}\gamma_n\Bigr)\,\widetilde\psi_n+\frac{C_\beta(n)^2C_{\mathrm{Loj}}}{2c_vc_{\mathrm{Loj}}}\gamma_n + \frac{L_fC_X(n)}{2}\gamma_n^2\\
&
  +c_v\gamma_n\E\Bigl[\1_{\{\tau\ge n\}}\sum_{i=1}^d\1_{G_{n,i}^c}|f_X^{(i)}(\theta_{n-1})|^2\Bigr]
+\gamma_n\E\Bigl[\1_{\{\tau\ge n\}}\sum_{i=1}^d\1_{G_{n,i}^c}|f_X^{(i)}(\theta_{n-1})(\sqrt{v_n^{(i)}}+\eps)^{-1}X^{(i)}(\theta_{n-1},U_n)|\Bigr]\\
&
  +\frac{L_f\gamma_n^2}{2}\E\Bigl[\1_{\{\tau\ge n\}}\sum_{i=1}^d\1_{G_{n,i}^c}(\sqrt{v_n^{(i)}}+\eps)^{-2}X^{(i)}(\theta_{n-1},U_n)^2\Bigr].
}
\end{prop}

\begin{proof}
Fix $n>n_0$. Since $\tau$ is the first exit time from $V$ after time $n_0$, the event $\{\tau\ge n\}$ implies $\theta_{n-1}\in V$. Since  $\grad F$ is $L_f$-Lipschitz continuous and $f_X=-\grad F$ on $V$ we get that for all $x\in V$ and $y\in\R^d$ 
\bas{\label{eq:F-recursion-late-smoothness}
F(y)\le F(x)-\scp{f_X(x)}{y-x}+\frac{L_f}{2}|y-x|^2.
}
Using that
\bas{\label{eq:F-recursion-late-update}
\theta_n=\theta_{n-1}+\gamma_n (\sqrt{v_n}+\eps)^{\circ(-1)}\circ X(\theta_{n-1},U_n),
}
Combining \eqref{eq:F-recursion-late-smoothness} and
\eqref{eq:F-recursion-late-update}, we conclude that on the event
$\{\tau\ge n\}$
\bas{\label{eq:F-recursion-late-taylor}
F(\theta_n)
&\le F(\theta_{n-1})
-\gamma_n \sum_{i=1}^d (\sqrt{v_n^{(i)}}+\eps)^{-1} f_X^{(i)}(\theta_{n-1})X^{(i)}(\theta_{n-1},U_n)\\
&\qquad + \frac{L_f\gamma_n^2}{2}\sum_{i=1}^d (\sqrt{v_n^{(i)}}+\eps)^{-2}X^{(i)}(\theta_{n-1},U_n)^2.
}
The events $G_{n,i}$ and $\{\tau\ge n\}$ are $\cF_{n-1}$-measurable. Moreover,
$\E[X^{(i)}(\theta_{n-1},U_n)|\cF_{n-1}]
=f_X^{(i)}(\theta_{n-1})$.
Hence, multiplying \eqref{eq:F-recursion-late-taylor} by
$\1_{\{\tau\ge n\}}$, using the conditional covariance decomposition on
$G_{n,i}$, and taking expectations, we arrive at
\bas{\label{eq:F-recursion-basic-estimate}
\psi_n
&\le \widetilde\psi_n
-\gamma_n \E\Bigl[\1_{\{\tau\ge n\}}\sum_{i=1}^d \1_{G_{n,i}}(\sqrt{v_n^{(i)}}+\eps)^{-1}|f_X^{(i)}(\theta_{n-1})|^2\Bigr]\\
&\qquad -\gamma_n \E\Bigl[\1_{\{\tau\ge n\}}\sum_{i=1}^d \1_{G_{n,i}}f_X^{(i)}(\theta_{n-1})
\cov((\sqrt{v_n^{(i)}}+\eps)^{-1},X^{(i)}(\theta_{n-1},U_n)|\cF_{n-1})\Bigr]\\
&\qquad +\gamma_n\E\Bigl[\1_{\{\tau\ge n\}}\sum_{i=1}^d\1_{G_{n,i}^c}|f_X^{(i)}(\theta_{n-1})(\sqrt{v_n^{(i)}}+\eps)^{-1}X^{(i)}(\theta_{n-1},U_n)|\Bigr]\\
&\qquad + \frac{L_f\gamma_n^2}{2}\E\Bigl[\1_{\{\tau\ge n\}}\sum_{i=1}^d \1_{G_{n,i}}(\sqrt{v_n^{(i)}}+\eps)^{-2}X^{(i)}(\theta_{n-1},U_n)^2\Bigr]\\
&\qquad + \frac{L_f\gamma_n^2}{2}\E\Bigl[\1_{\{\tau\ge n\}}\sum_{i=1}^d \1_{G_{n,i}^c}(\sqrt{v_n^{(i)}}+\eps)^{-2}X^{(i)}(\theta_{n-1},U_n)^2\Bigr].
} 
Next, we estimate the drift, covariance, and quadratic terms on the right-hand
side. By assumption~(i) and the upper \L{}ojasiewicz bound, we have
\bas{\label{eq:F-recursion-drift-bound}
\E[\1_{\{\tau\ge n\}}&\sum_{i=1}^d \1_{G_{n,i}}(\sqrt{v_n^{(i)}}+\eps)^{-1}|f_X^{(i)}(\theta_{n-1})|^2]\\
&\ge c_v\,\E[\1_{\{\tau\ge n\}}\sum_{i=1}^d\1_{G_{n,i}}|f_X^{(i)}(\theta_{n-1})|^2]\\
&\ge \frac{c_v}{C_{\mathrm{Loj}}}\,\widetilde\psi_n
-c_v\E\Bigl[\1_{\{\tau\ge n\}}\sum_{i=1}^d\1_{G_{n,i}^c}|f_X^{(i)}(\theta_{n-1})|^2\Bigr].
}
Next, we estimate the covariance term. As a consequence of the
Cauchy--Schwarz inequality, we have
\bas{\label{eq:34563}
&\Bigl|\E\Bigl[\1_{\{\tau\ge n\}}\sum_{i=1}^d \1_{G_{n,i}}f_X^{(i)}(\theta_{n-1}) \cov((\sqrt{v_n^{(i)}}+\eps)^{-1},X^{(i)}(\theta_{n-1},U_n)|\cF_{n-1})\Bigr]\Bigr|\\
&\le  \E\Bigl[\1_{\{\tau\ge n\}}|f_X(\theta_{n-1})|^2\Bigr]^{1/2}
\,\E\Bigl[\1_{\{\tau\ge n\}}\sum_{i=1}^d \1_{G_{n,i}}\cov((\sqrt{v_n^{(i)}}+\eps)^{-1},X^{(i)}(\theta_{n-1},U_n)|\cF_{n-1})^2\Bigr]^{1/2}.
}
By assumption~(iii), the second factor is bounded by $C_\beta(n)$. Moreover, on $\{\tau\ge n\}$ the lower \L{}ojasiewicz bound yields
\bas{\label{eq:F-recursion-late-lower-lojasiewicz}
|f_X(\theta_{n-1})|^2\le c_{\mathrm{Loj}}^{-1}F(\theta_{n-1}),
}
so that \eqref{eq:F-recursion-late-definitions} and
\eqref{eq:F-recursion-late-lower-lojasiewicz} give
\bas{\label{eq:F-recursion-late-gradient-bound}
\E\bigl[\1_{\{\tau\ge n\}} |f_X(\theta_{n-1})|^2\bigr]
\le c_{\mathrm{Loj}}^{-1}\,\widetilde\psi_n.
}
Consequently, we get with \eqref{eq:34563} and
\eqref{eq:F-recursion-late-gradient-bound} that
\bas{\label{eq:F-recursion-covariance-bound}
\E\Bigl[\1_{\{\tau\ge n\}}\sum_{i=1}^d \1_{G_{n,i}}f_X^{(i)}(\theta_{n-1}) \cov((\sqrt{v_n^{(i)}}+\eps)^{-1},X^{(i)}(\theta_{n-1},U_n)|\cF_{n-1})\Bigr]
\ge -\frac{C_\beta(n)}{c_{\mathrm{Loj}}^{1/2}}\widetilde\psi_n^{1/2}.
}
By assumption~(ii),
\bas{\label{eq:F-recursion-quadratic-bound}
\E\Bigl[\1_{\{\tau\ge n\}}\sum_{i=1}^d \1_{G_{n,i}}(\sqrt{v_n^{(i)}}+\eps)^{-2}X^{(i)}(\theta_{n-1},U_n)^2\Bigr]
\le C_X(n).
}
Combining~\eqref{eq:F-recursion-basic-estimate} with~\eqref{eq:F-recursion-drift-bound},
\eqref{eq:F-recursion-covariance-bound}, and
\eqref{eq:F-recursion-quadratic-bound} proves
\eqref{eq:F-recursion-late-square-root}.
For the second estimate, we use Young's inequality in the form
\bas{\label{eq:F-recursion-late-young-step}
\frac{C_\beta(n)}{c_{\mathrm{Loj}}^{1/2}}\widetilde\psi_n^{1/2}\le \frac{c_v}{2C_{\mathrm{Loj}}}\widetilde\psi_n+\frac{C_\beta(n)^2C_{\mathrm{Loj}}}{2c_vc_{\mathrm{Loj}}},
}
which, together with \eqref{eq:F-recursion-late-square-root} and
\eqref{eq:F-recursion-late-young-step}, proves
\eqref{eq:F-recursion-late-young}.
\end{proof}

\section{Analysis of the property in \cref{item:iii_Prop21} of \cref{prop:F-recursion-md-loj-event}}
\label{sec:analysis-prop-iii}

This section controls the covariance correction appearing in \cref{item:iii_Prop21}
of \cref{prop:F-recursion-md-loj-event}.
A conditional Lipschitz argument bounds each coordinatewise covariance
in terms of a conditional third moment of the innovation and
a negative third moment of $ v_{ n - 1 } $. An auxiliary
estimate for the resulting $ \beta $-dependent quotient then distinguishes the
initial $ n^{ - 1 } $-regime from the asymptotic $ 1 - \beta $-regime.
%
%
%
%
\begin{prop}
\label{prop:prop_iii}
Let $\beta\in[ \nicefrac{1}{2}, 1 ) $,
$ \eps \in [0,\infty) $,
$ i \in \{ 1, 2, \dots, d \} $,
$ n \in \N \backslash \{ 1 \} $,
let $ A_n \in \cF_{ n - 1 } $,
and let $ C_X^{ (3) } , C_v^{ (-3) } \in (0,\infty) $.
Assume that
\bas{
  \E\bigl[
    | X^{(i)}(\theta_{n-1},U_n) |^3
    |
    \cF_{ n - 1 }
  \bigr]
  \le C_X^{ (3) }
\qquad\text{on }A_n
}
and
\bas{
\label{eq:inverse_moment_assumption}
  \E[ \1_{ A_n }( v_{ n - 1 }^{ (i) } )^{ - 3 } ]
  \le C_v^{ ( - 3 ) }
  .
}
The negative-moment assumption implies that $v_{n-1}^{(i)}>0$ almost surely
on $A_n$. If $\eps=0$, the conditional covariances in the conclusion are
therefore evaluated on $A_n$ and extended by zero to $ \Omega \backslash A_n $.
Then one has
\bas{
\label{eq:cov-first-step-4}
  \E\Bigl[
    \1_{A_n}\cov((\sqrt{v_n^{(i)}}+\eps)^{-1},
    X^{(i)}(\theta_{n-1},U_n)|\cF_{n-1})^2
  \Bigr]^{1/2}
  \le
  4 C_X^{(3)} (C_v^{(-3)})^{1/2}\,\frac{1-\beta}{1-\beta^{n-1}}.
}
\end{prop}

\begin{proof}
Set
\bas{
Y_n^{(i)}=\frac{\beta(1-\beta^{n-1})}{1-\beta^n}v_{n-1}^{(i)}
=\kappa_n v_{n-1}^{(i)}
\qquad
  \text{and}
\qquad
  Z_n=\frac{1-\beta}{1-\beta^n}.
}
Then
\bas{
v_n^{(i)}=Y_n^{(i)}+Z_nX^{(i)}(\theta_{n-1},U_n)^2.
}
The negative-moment assumption gives $v_{n-1}^{(i)}>0$ almost surely on
$A_n$, and hence $Y_n^{(i)}>0$ there. In what follows, all conditional
identities involving inverse factors are understood on $A_n$.
The function 
\bas{
h:(0,\infty)\to(0,\infty), \ y\mapsto h(y)=\frac1{\sqrt y+\eps}
}
is for every $y_0\in(0,\infty)$ Lipschitz continuous on $[y_0,\infty)$ with constant $\frac12 y_0^{-3/2}$. 
Using that $Y_n^{(i)}$ is $\cF_{n-1}$-measurable, we conclude on $A_n$ that
\bas{
&|\cov((\sqrt{v_n^{(i)}}+\eps)^{-1},X^{(i)}(\theta_{n-1},U_n)|\cF_{n-1})|\\
&=|\cov(h(Y_n^{(i)}+Z_n X^{(i)}(\theta_{n-1},U_n)^2),X^{(i)}(\theta_{n-1},U_n)|\cF_{n-1})|\\
&= \Bigl|\E\Bigl[(h(Y_n^{(i)}+Z_n X^{(i)}(\theta_{n-1},U_n)^2)-h(Y_n^{(i)}))
\bigl(X^{(i)}(\theta_{n-1},U_n)-f_X^{(i)}(\theta_{n-1})\bigr)\Bigm|\cF_{n-1}\Bigr]\Bigr|\\
&\le \frac{Z_n}{2(Y_n^{(i)})^{3/2}}\,\E\Bigl[X^{(i)}(\theta_{n-1},U_n)^2
\bigl|X^{(i)}(\theta_{n-1},U_n)-f_X^{(i)}(\theta_{n-1})\bigr|\Bigm|\cF_{n-1}\Bigr]\\
&\le \frac{Z_n}{2(Y_n^{(i)})^{3/2}}
\Bigl(\E[|X^{(i)}(\theta_{n-1},U_n)|^3|\cF_{n-1}]
+\E[|X^{(i)}(\theta_{n-1},U_n)|^2|\cF_{n-1}]\,\E[|X^{(i)}(\theta_{n-1},U_n)||\cF_{n-1}]\Bigr)\\
&\le \frac{Z_n}{(Y_n^{(i)})^{3/2}}\E[|X^{(i)}(\theta_{n-1},U_n)|^3|\cF_{n-1}].
}
On $A_n$, the assumption on the third moment therefore implies
\bas{
|\cov((\sqrt{v_n^{(i)}}+\eps)^{-1},X^{(i)}(\theta_{n-1},U_n)|\cF_{n-1})|^2
\le Z_n^2 (C_X^{(3)})^2 (Y_n^{(i)})^{-3}.
}
Multiplying by $\1_{A_n}$ and taking expectations yields
\bas{\label{eq:prop-iii-cov-coordinate-bound}
\E\Bigl[\1_{A_n}\cov((\sqrt{v_n^{(i)}}+\eps)^{-1},
X^{(i)}(\theta_{n-1},U_n)|\cF_{n-1})^2\Bigr]
\le Z_n^2 (C_X^{(3)})^2 \E[\1_{A_n}(Y_n^{(i)})^{-3}].
}
Since $Y_n^{(i)}=\kappa_n v_{n-1}^{(i)}$, we obtain
\bas{\label{eq:prop-iii-Y-negative-moment-bound}
\E[\1_{A_n}(Y_n^{(i)})^{-3}]
= \Bigl(\frac{1-\beta^n}{\beta(1-\beta^{n-1})}\Bigr)^3
\E[\1_{A_n}(v_{n-1}^{(i)})^{-3}]\le \Bigl(\frac{1-\beta^n}{\beta(1-\beta^{n-1})}\Bigr)^3 C_v^{(-3)}.
}
Hence, combining~\eqref{eq:prop-iii-cov-coordinate-bound} and~\eqref{eq:prop-iii-Y-negative-moment-bound} and using the definition of $Z_n$,
\bas{
\E\Bigl[\1_{A_n}\cov((\sqrt{v_n^{(i)}}+\eps)^{-1},
X^{(i)}(\theta_{n-1},U_n)|\cF_{n-1})^2\Bigr]
&\le (C_X^{(3)})^2 C_v^{(-3)}
\Bigl(\frac{1-\beta}{1-\beta^n}\Bigr)^2
\Bigl(\frac{1-\beta^n}{\beta(1-\beta^{n-1})}\Bigr)^3\\
&= (C_X^{(3)})^2 C_v^{(-3)}
\beta^{-3}\frac{(1-\beta)^2(1-\beta^n)}{(1-\beta^{n-1})^3}.
}
By \cref{le:beta-power-ratio-monotone}, we have
\bas{
\frac{1-\beta^n}{n}\le \frac{1-\beta^{n-1}}{n-1} \text{, \ \ 
so that \ \ }
\frac{1-\beta^n}{1-\beta^{n-1}}\le \frac{n}{n-1}.
}
we infer, using $\beta\in[\frac12,1)$ and $n\ge 2$, that
\bas{
\beta^{-3}\frac{1-\beta^n}{1-\beta^{n-1}}
\le 8\cdot 2 =16.
}
Taking square roots, we finish the proof with the conclusion that
\bas{
\E\Bigl[\1_{A_n}\cov((\sqrt{v_n^{(i)}}+\eps)^{-1},
X^{(i)}(\theta_{n-1},U_n)|\cF_{n-1})^2\Bigr]^{1/2}
\le 4 C_X^{(3)} (C_v^{(-3)})^{1/2}\,\frac{1-\beta}{1-\beta^{n-1}}.
}
\end{proof}
The latter proposition incorporates an $n$-dependent term $\frac{1-\beta}{1-\beta^{n-1}}$. This is analysed in the next lemma.

\begin{lemma}
\label{le:beta-quotient-regimes}
Let $ n \in \N $, $ \beta \in [0,1) $.
Then one has that
\bas{
\frac{1-\beta}{1-\beta^n}\le \begin{cases} \frac{2}{n}, & \text{ if } n\le (1-\beta)^{-1}\\
\frac{1-\beta}{1-e^{-1}}, & \text{ if } n\ge (1-\beta)^{-1}
\end{cases}
}
\end{lemma}

\begin{proof}
Set
$
x=1-\beta\in(0,1].
$
Then
\bas{
\frac{1-\beta}{1-\beta^n}=\frac{x}{1-(1-x)^n}.
}
Assume first that $n\le (1-\beta)^{-1}$, that is, $nx\le 1$. Since $(1-x)^n\le e^{-nx}$ and $1-e^{-y}\ge y/2$ for every $y\in[0,1]$, we obtain
\bas{
1-(1-x)^n\ge 1-e^{-nx}\ge \frac{nx}{2}
\text{, \ \  which yields \ \ }
\frac{1-\beta}{1-\beta^n}=\frac{x}{1-(1-x)^n}\le \frac{2}{n}.
}
This proves the first assertion.

Now assume that $n\ge (1-\beta)^{-1}$, that is, $nx\ge 1$. Then
\bas{
\beta^n=(1-x)^n\le e^{-nx}\le e^{-1} \text{ \ \ and, hence, \ \ }
\frac{1-\beta}{1-\beta^n}\le \frac{1-\beta}{1-e^{-1}}.
}
This proves the second assertion.
\end{proof}

\section{Invserve moment estimates for the second moment process \texorpdfstring{$ ( v_n )_{ n \in \N_0 } $}{v\_n} in the RMSprop method}
\label{sec:negative-moments}
This section establishes the negative moment estimates needed for the covariance
and quadratic error bounds. We introduce the $(c,q)$-regularity condition,
derive conditional Laplace-transform estimates for the stopped conditioner, and
convert these estimates into bounds for arbitrary negative moments. Since such
bounds need not hold initially, the final result is formulated on the event
that the corresponding coordinatewise startup time has already occurred.

\begin{defi}\label{def:regular}Let $c,q\in(0,\infty)$.
We say that an innovation $ (X, U ) $ is \emph{$(c,q)$-regular}
on a set $ \cV \subseteq \R^d $
if and only if for all $\theta\in \cV$, $i\in\{1,\dots,d\}$, $\lambda\in[0,\infty)$
we have that
\bas{
  - \ln \E[e^{-\lambda X^{(i)}(\theta,U)^2}]\ge (c\lambda)\wedge q.
}
\end{defi}
First we will deduce that $ ( c, q ) $-regular innovations
do allow us to estimate the exponential moments of $ v_n $.

\begin{prop}
\label{prop:7463}
Let $ ( X, U ) $ be a $ ( c, q ) $-regular innovation
on a measurable set $ \cV $ and let $ ( \theta_n ) $ be an \RMSprop\ algorithm
with arbitrary parameter tuple $ ( \beta, \eps, (\gamma_n), \theta_0 ) $ as in the introduction.
Let $ n_0 \in \N_0 $ and let $ \tau $ be the stopping time given by
\bas{ 
 \tau=\inf\{n\ge n_0: \theta_n\not\in \cV\}.}
One has for every $n\in\N_0\cap [n_0,\infty)$, $i\in\{1,\dots,d\} $, $ \lambda\in[0,\infty)$ that
\bas{\label{eq:7882356}
-\ln \E\bigl[e^{-\lambda (v_n^{(i)}+\1_{\{ n\le \tau\}^c}\infty)}|\cF_{n_0}\bigr]\ge\sum_{k=n_0+1}^n \left( c \frac{1-\beta}{1-\beta^n} \beta^{n-k} \lambda \right) \wedge q +\frac{1-\beta^{n_0}}{1-\beta^n} \beta^{n-n_0} \lambda  v_{n_0}^{(i)} .
}
In the case $n=n_0=0$, the last term on the right-hand side is understood as zero.
\end{prop}

\begin{proof}
Fix $i\in\{1,\dots,d\}$, $\lambda\in[0,\infty)$.
The statement is trivial for $n=n_0$. We thus fix $n\in\N$ with $n>n_0$ and note that
\bas{
v_n^{(i)}
& =\frac{1-\beta}{1-\beta^{n}}\sum_{k=1}^{n} \beta^{n-k} X^{(i)}(\theta_{k-1}, U_{k})^2\\
&=  \sum_{k=n_0+1}^{n} \frac{1-\beta}{1-\beta^{n}}\beta^{n-k} X^{(i)}(\theta_{k-1}, U_{k})^2 + v_{n_0}^{(i)}\frac{1-\beta^{n_0}}{1-\beta^{n}} \beta^{n-n_0}.
}
We let for $k=n_0+1,\dots,n$
\bas{
Z_k:=\begin{cases} \frac{1-\beta}{1-\beta^{n}}\beta^{n-k} X^{(i)}(\theta_{k-1}, U_{k})^2, & \text{ if } k\le \tau,\\
\infty, & \text{ else.}
\end{cases}
}
Then one has
\bas{
v_n^{(i)}+\1_{\{ n\le \tau\}^c}\infty = \sum_{k=n_0+1}^{n} Z_k + v_{n_0}^{(i)}\frac{1-\beta^{n_0}}{1-\beta^{n}} \beta^{n-n_0}.
}
Moreover,
\bas{
\E[e^{-\lambda \sum_{k=n_0+1}^n Z_k}|\cF_{n-1}] &= e^{-\lambda \sum_{k=n_0+1}^{n-1} Z_k } \1_{\{n\le\tau\}} \,   \E\Bigl[\exp\bigl\{-\lambda  \frac{1-\beta}{1-\beta^{n}}X^{(i)}(\theta_{n-1}, U_{n})^2 \bigr\}\Big|\cF_{n-1}\Bigr] \\
&\le  e^{-\lambda \sum_{k=n_0+1}^{n-1} Z_k } \exp\Bigl \{- \Bigl(c \frac{1-\beta}{1-\beta^{n}}\lambda \Bigr)\wedge q\Bigr\}.
}
Iterating yields that 
\bas{
\E[e^{-\lambda \sum_{k=n_0+1}^n Z_k}|\cF_{n_0}]  \le  \exp\Bigl \{- \sum_{k=n_0+1}^n \Bigl(c  \frac{1-\beta}{1-\beta^{n}} \beta^{n-k}\lambda \Bigr)\wedge q\Bigr\}.
}
Consequently, one has that
\bas{
\E[e^{-\lambda (v_n^{(i)}+\1_{\{n\le \tau\}^c} \infty)}|\cF_{n_0}] &=e^{-\lambda v_{n_0}^{(i)}\frac{1-\beta^{n_0}}{1-\beta^{n}} \beta^{n-n_0}} \E[e^{-\lambda \sum_{k=n_0+1}^n Z_k}|\cF_{n_0}]  \\
&\le  \exp\Bigl \{- \sum_{k=n_0+1}^n \Bigl(c  \frac{1-\beta}{1-\beta^{n}} \beta^{n-k}\lambda \Bigr)\wedge q -\frac{1-\beta^{n_0}}{1-\beta^{n}} \beta^{n-n_0}\lambda v_{n_0}^{(i)}\Bigr\}.
}
This finishes the proof.
\end{proof}

Next, we will deduce a negative moment estimate for random variables satisfying an exponential moment estimate like (\ref{eq:7882356}). The moment estimate will make use of the following technical result.

\begin{lemma}\label{le:beta-power-ratio-monotone}
Let $\beta\in[0,1)$. The function
\bas{
f:(0,\infty)\to(0,\infty),\qquad f(x)=\frac{1-\beta^x}{x}
}
is decreasing. In particular, for every $n\in\N$,
\bas{
\frac{1-\beta}{1-\beta^n}\ge \frac1n.
}
\end{lemma}

\begin{proof}
If $\beta=0$, then $f(x)=x^{-1}$ for every $x>0$, so monotonicity is
clear. Now suppose that $\beta\in(0,1)$. For $x>0$, one has
\bas{
f'(x)=\frac{-x \beta^x \ln \beta - 1 + \beta^x}{x^2}.
}
Letting $y=-x\ln\beta>0$, the numerator becomes
\bas{
(y+1)e^{-y}-1.
}
Since $e^y>1+y$ for all $y>0$, we obtain $(y+1)e^{-y}<1$. Hence the numerator is strictly negative, so that $f'(x)<0$ for all $x>0$. This proves the monotonicity assertion.

For every $n\in\N$, monotonicity implies that
\bas{
\frac{1-\beta^n}{n}=f(n)\le f(1)=1-\beta,
}
which implies the stated inequality.
\end{proof}

\begin{lemma}[Inverse moment estimates]
\label{le:neg_mom_1}
Let $c,q,\rho\in(0,\infty)$, $p\in[1,\infty)$, $\beta \in (0, 1) $, $ n, n_0 \in \N $ satisfy
\bas{\label{eq:2356374}
q+p\ln\beta\ge \rho
}
and $n> n_0$. 
Let $V_n^*$ be a non-negative random variable which satisfies for all $ \lambda \in [0,\infty) $ that
\bas{\label{eq:7634589-1}
  - \ln \mathbb{E}\bigl[ e^{-\lambda V_n^*} \bigr]
  \ge \sum_{k=n_0+1}^n \left( c \frac{1-\beta}{1-\beta^n} \beta^{n-k} \lambda \right) \wedge q + c \frac{1-\beta^{n_0}}{1-\beta^n} \beta^{n-n_0} \lambda.
}
Then \bas{\label{eq:7828341-2}
\E[(V_n^*)^{-p}]&\le \frac1{c^p} D_{\rho,p},
}
where
\bas{
\label{eq:D-rho-p}
  D_{\rho,p}:=\sum_{r=0}^\infty (r+1)^p e^{-\rho r}<\infty.
}
\end{lemma}

\begin{proof}
We can assume without loss of generality that $ c = 1 $.
Note that for every $x\ge0$
\bas{\label{eq:834656}
x^{-p}=\frac 1{\Gamma(p)}\int_0^\infty e^{-x\lambda } \lambda^{p-1}\,\dd \lambda
}
so that by Fubini's theorem
\bas{\label{eq87235465}
\E[(V_n^*)^{-p}]&=\frac 1{\Gamma(p)}  \E\Bigl[ \int_0^\infty e^{-V_n^* \lambda} \lambda^{p-1}\,\dd \lambda\Bigr] =\frac 1{\Gamma(p)}  \int_0^\infty   \E\bigl[e^{-\lambda V_n^*}\bigr]\, \lambda^{p-1}\,\dd \lambda.
}
To bound the integral we distinguish cases and choose
\bas{
C_1 = \frac{q(1-\beta^n)}{1-\beta}
\quad \text{and} \quad
C_2 = \frac{q(1-\beta^n)}{(1-\beta)\beta^{n-n_0-1}} = C_1 \beta^{-(n-n_0-1)}.
}
Thus
\bas{
\E[(V_n^*)^{-p}]
=\frac1{\Gamma(p)}\int_0^{C_1}\E[e^{-\lambda V_n^*}]\,\lambda^{p-1}\,\dd\lambda
&+\frac1{\Gamma(p)}\int_{C_1}^{C_2}\E[e^{-\lambda V_n^*}]\,\lambda^{p-1}\,\dd\lambda\\
&+\frac1{\Gamma(p)}\int_{C_2}^{\infty}\E[e^{-\lambda V_n^*}]\,\lambda^{p-1}\,\dd\lambda.
}

1) First, let us estimate the contribution of $[C_1,C_2)$ by slicing this interval. If
\bas{
C_1<C_2,
}
equivalently if $n\ge n_0+2$, then for $r=1,\ldots,n-n_0-1$ define
\bas{
I_r:=[\beta^{-(r-1)}C_1,\beta^{-r}C_1).
}
These slices form a disjoint partition of $[C_1,C_2)$.  For every $\lambda\in I_r$ one has
\bas{
\beta^{-(r-1)}C_1\le \lambda < \beta^{-r}C_1
\text{ \ \ \ or, equivalently, \ \ \ } 
\beta^{-(r-1)} q \le \frac{1-\beta}{1-\beta^n}\lambda< \beta^{-r}q.
} 
Hence, for $k\in\{n_0+1,\ldots,n\}$, one has
\bas{
\frac{1-\beta}{1-\beta^n}\beta^{n-k}\lambda \begin{cases} < q
\quad\text{if } k\in\{n_0+1,\ldots,n-r\},\\
\ge q
\quad\text{if } k\in\{n-r+1,\ldots,n\},\\
\end{cases}
}
so that
\bas{
\sum_{k=n_0+1}^n \left( \frac{1-\beta}{1-\beta^n}\beta^{n-k}\lambda\right)\wedge q
&= \sum_{k=n_0+1}^{n-r}\frac{1-\beta}{1-\beta^n}\beta^{n-k}\lambda + rq\\
&= \frac{\beta^r-\beta^{n-n_0}}{1-\beta^n}\lambda+rq.
} 
Combining this with \eqref{eq:7634589-1}, we obtain
\bas{
-\ln \mathbb{E} \left[ e^{-\lambda V_n^*} \right]
&\ge rq+\lambda\frac{\beta^r-\beta^n}{1-\beta^n}= rq+a_{n,r}\lambda,
}
where
\bas{
a_{n,r}:=\frac{\beta^r-\beta^n}{1-\beta^n}
=\beta^r\frac{1-\beta^{n-r}}{1-\beta^n}.
}
Therefore, we obtain that
\bas{
J_r:&=\frac1{\Gamma(p)}\int_{I_r}\E[e^{-\lambda V_n^*}]\,\lambda^{p-1}\,\dd\lambda \le \frac{e^{-rq}}{\Gamma(p)}\int_{\beta^{-(r-1)}C_1}^{\beta^{-r}C_1}e^{-a_{n,r}\lambda}\lambda^{p-1}\,\dd\lambda\\
&= \frac{e^{-rq}}{\Gamma(p)}a_{n,r}^{-p}\int_{a_{n,r}\beta^{-(r-1)}C_1}^{a_{n,r}\beta^{-r}C_1}e^{-u}u^{p-1}\,\dd u\\
&\le \frac{e^{-rq}}{\Gamma(p)}a_{n,r}^{-p}\int_0^\infty e^{-u}u^{p-1}\,\dd u= e^{-rq}a_{n,r}^{-p}\\
&= e^{-r(q+p\ln\beta)}\left(\frac{1-\beta^n}{1-\beta^{n-r}}\right)^p.
}
Using \cref{le:beta-power-ratio-monotone} and that $n-r\ge1$, we get that
\bas{
\frac{1-\beta^n}{1-\beta^{n-r}}\le \frac{n}{n-r}\le r+1.
}
Consequently, we can conclude with assumption~(\ref{eq:2356374}) that
\bas{
J_r\le (r+1)^p e^{-r(q+p\ln\beta)} \le (r+1)^p e^{-\rho r}.
}
Hence the whole contribution of the middle interval is bounded by
\bas{\label{eq:3466736}
\frac1{\Gamma(p)}\int_{C_1}^{C_2}\E[e^{-\lambda V_n^*}]\,\lambda^{p-1}\,\dd\lambda
=\sum_{r=1}^{n-n_0-1}J_r
\le \sum_{r=1}^{n-n_0-1}(r+1)^p e^{-\rho r},
}
while for $n=n_0+1$ the interval $[C_1,C_2)$ is empty.

2) Next, let $\lambda\ge C_2$. Then all terms in the sum are capped at $q$, so
\bas{
-\ln \mathbb{E} \left[ e^{-\lambda V_n^*} \right] \ge (n - n_0)q + \frac{1-\beta^{n_0}}{1-\beta^n} \beta^{n-n_0} \lambda.
}
Therefore,
\bas{
\frac1{\Gamma(p)}\int_{C_2}^{\infty}\E[e^{-\lambda V_n^*}]\,\lambda^{p-1}\,\dd\lambda
&\le \frac{e^{-(n-n_0)q}}{\Gamma(p)}\int_{C_2}^{\infty}\exp\left(-\frac{1-\beta^{n_0}}{1-\beta^n} \beta^{n-n_0} \lambda\right)\lambda^{p-1}\,\dd\lambda\\
&\le \frac{e^{-(n-n_0)q}}{\Gamma(p)}\int_{0}^{\infty}\exp\left(-\frac{1-\beta^{n_0}}{1-\beta^n} \beta^{n-n_0} \lambda\right)\lambda^{p-1}\,\dd\lambda\\
&= \left(\frac{1-\beta^n}{(1-\beta^{n_0})\beta^{n-n_0}}\right)^p e^{-(n-n_0)q},
}
where we have used~(\ref{eq:834656}) again.
Applying \cref{le:beta-power-ratio-monotone}, we obtain
\bas{
\frac{1-\beta^n}{n}\le \frac{1-\beta^{n_0}}{n_0}\text{ \ \ \  or, equivalently, \ \ \ }
\frac{1-\beta^n}{1-\beta^{n_0}} \le \frac{n}{n_0}.
}
Hence 
\bas{
\frac1{\Gamma(p)}\int_{C_2}^{\infty}\E[e^{-\lambda V_n^*}]\,\lambda^{p-1}\,\dd\lambda
\le \left(\frac{n}{n_0}\right)^p\exp\left(-(n-n_0)(q+p\ln\beta)\right).
}
We conclude as above and use that $\frac n{n_0}\le n-n_0+1$ which implies that 
\bas{\label{eq:3597246}
\frac1{\Gamma(p)}\int_{C_2}^{\infty}\E[e^{-\lambda V_n^*}]\,\lambda^{p-1}\,\dd\lambda
&\le (n-n_0+1)^p e^{-\rho (n-n_0)}
}
Thus together with (\ref{eq:3466736}) we get that
\bas{\label{eq:3245534}
\frac1{\Gamma(p)} \int_{C_1}^\infty  \E[e^{-\lambda V_n^*}]\,\lambda^{p-1}\,\dd\lambda \le D_{\rho,p}.
}
3) Finally, let $\lambda\in[0,C_1]$. Then none of the summands are capped, and therefore
\bas{
-\ln \mathbb{E} \left[ e^{-\lambda V_n^*} \right]
&\ge \sum_{k=n_0+1}^n \frac{1-\beta}{1-\beta^n} \beta^{n-k} \lambda + \frac{1-\beta^{n_0}}{1-\beta^n} \beta^{n-n_0} \lambda\\
&= \lambda \frac{1-\beta^{n-n_0}}{1-\beta^n} + \lambda \frac{\beta^{n-n_0}-\beta^n}{1-\beta^n}=\lambda.
}
Thus
\bas{
\frac1{\Gamma(p)}\int_0^{C_1}\E[e^{-\lambda V_n^*}]\,\lambda^{p-1}\,\dd\lambda
\le \frac1{\Gamma(p)}\int_0^{\infty} e^{-\lambda}\lambda^{p-1}\,\dd\lambda
=1.
}
Note that  the first summand of the series defining $D_{\rho,p}$ is one and thus we get with~(\ref{eq:3466736})  and~(\ref{eq:3597246}) that
\bas{
\frac1{\Gamma(p)}\int_{0}^{\infty}\E[e^{-\lambda V_n^*}]\,\lambda^{p-1}\,\dd\lambda
&\le \sum_{r=0}^{n-n_0} (r+1)^p e^{-\rho r}\le D_{\rho,p}.
}
\end{proof}

\begin{lemma}
\label{le:moment_bound}
Let $ c, q, \rho \in (0,\infty) $, $ p \in [1,\infty) $, $ \beta \in (0,1) $ satisfy
\bas{
  q + p \ln\beta \ge \rho
  ,
}
let $\cV\subseteq \R^d$ be measurable, and let $(X,U)$ be $(c,q)$-regular on $\cV$ in the sense of \cref{def:regular}. Consider for each $i\in\{1,\dots,d\}$ the stopping time
\bas{
\bar\tau=\inf\{n\in\N_0: \theta_n\not\in \cV\}
\qquad\text{and}\qquad
\sigma^{(i)}=\inf\bigl\{n\in\N: v_n^{(i)}\ge c\bigr\}.
}
Then for every $ i \in \{ 1, \dots, d \} $, $ n \in \N $ one has
\bas{
  \E\bigl[
    (v_n^{(i)})^{-p}
    \1_{\{\sigma^{(i)}\le n\}}\1_{\{\bar\tau\ge n\}}
  \bigr]
  \le c^{ - p } D_{ \rho, p } ,
}
where $D_{\rho,p}$ is given by \eqref{eq:D-rho-p}.
\end{lemma}

\begin{proof}
Fix $i\in\{1,\dots,d\}$ and 
consider for all $n\in\N$
\bas{V_n^{(i)}=v_n^{(i)}+\1_{\{\bar \tau<n \}}\infty .
}
We use the recursion (\ref{eq:v_n_rec}) to show that  by induction for $m\in\{1,\dots,n\}$ one has
\bas{
-\ln \E[\exp\{-\lambda V_n^{(i)}\}|\cF_m]
\ge \sum_{k=m+1}^n \left( c \frac{1-\beta}{1-\beta^n}\beta^{n-k}\lambda\right)\wedge q
+\frac{1-\beta^m}{1-\beta^n}\beta^{n-m}\lambda V_m^{(i)}.
}
Indeed, if we set
\bas{
a_n=\frac{\beta(1-\beta^{n-1})}{1-\beta^n}
\qquad\text{and}\qquad
b_n=\frac{1-\beta}{1-\beta^n},
}
then on the event $\{\bar\tau\ge n\}$ one has
\bas{
V_n^{(i)}=a_nV_{n-1}^{(i)}+b_nX^{(i)}(\theta_{n-1},U_n)^2.
}
Since both sides vanish on the complementary event after applying $\exp(-\lambda \cdot)$, this implies
\bas{
\E[e^{-\lambda V_n^{(i)}}|\cF_{n-1}]
&=e^{-a_n\lambda V_{n-1}^{(i)}}\1_{\{\bar\tau\ge n\}}
\E\Bigl[e^{-b_n\lambda X^{(i)}(\theta_{n-1},U_n)^2}\Bigm|\cF_{n-1}\Bigr]\\
&\le e^{-a_n\lambda V_{n-1}^{(i)}}
\exp\bigl(-(cb_n\lambda)\wedge q\bigr),
}
where we used the $(c,q)$-regularity on the event $\{\bar\tau\ge n\}\subseteq\{\theta_{n-1}\in\cV\}$. Therefore,
\bas{
-\ln \E[e^{-\lambda V_n^{(i)}}|\cF_{n-1}]
\ge \left( c \frac{1-\beta}{1-\beta^n}\lambda\right)\wedge q
+\frac{\beta(1-\beta^{n-1})}{1-\beta^n}\lambda V_{n-1}^{(i)},
}
and iteration yields the displayed estimate, since
\bas{
\prod_{j=m+1}^n \frac{\beta(1-\beta^{j-1})}{1-\beta^j}
=\frac{1-\beta^m}{1-\beta^n}\beta^{n-m}.
}
Consequently, on the event $\{\sigma^{(i)}=m\}$, the bound $V_m^{(i)}\ge c$ holds, with the value $+\infty$ allowed if $\bar\tau<m$ and, hence, on this event
\bas{
-\ln \E[\exp\{-\lambda V_n^{(i)}\}|\cF_m]
\ge \sum_{k=m+1}^n \left( c \frac{1-\beta}{1-\beta^n}\beta^{n-k}\lambda\right)\wedge q
+c\frac{1-\beta^m}{1-\beta^n}\beta^{n-m}\lambda.
}
If $m=n$, then on $\{\sigma^{(i)}=m\}$ we have $v_n^{(i)}\ge c$. Since the first summand in the definition of $D_{\rho,p}$ is equal to one, this gives on the latter event that
\bas{
\E[\1_{\{\bar\tau\ge n\}} (v_n^{(i)})^{-p}|\cF_n]
=\1_{\{\bar\tau\ge n\}} (v_n^{(i)})^{-p}
\le c^{-p}\le c^{-p}D_{\rho,p}.
}
If $m<n$, then we apply \cref{le:neg_mom_1} for the conditional distribution of $V_n^{(i)}$ given $\cF_m$ and get that on $\{\sigma^{(i)}=m\}$
\bas{
\E[\1_{\{\bar\tau\ge n\}} (v_n^{(i)})^{-p}|\cF_m]= \E[ (V_n^{(i)})^{-p}|\cF_m]\le c^{-p}D_{\rho,p}.
}
The statement follows by decomposing the event $\{\sigma^{(i)}\le n\}$ according to the hitting time:
\bas{
\E[\1_{\{\sigma^{(i)}\le n\}} \1_{\{\bar\tau\ge n\}} (v_n^{(i)})^{-p}]
= \sum_{m=1}^{n}\E\bigl[\1_{\{\sigma^{(i)}=m\}}\, \E[\1_{\{\bar\tau\ge n\}} (v_n^{(i)})^{-p}|\cF_m]\bigr]
\le  c^{-p}D_{\rho,p}.
}
\end{proof}

\section{Technical estimates for the property in \cref{item:ii_Prop21} of
\cref{prop:F-recursion-md-loj-event}}\label{sec:technical-estimates}

This section controls the quadratic Taylor remainder appearing in
\cref{item:ii_Prop21} of
\cref{prop:F-recursion-md-loj-event}. The inverse squared conditioner
is a decreasing function of the current squared innovation, which yields a
nonpositive conditional covariance and permits the two factors to be
separated. Combining this observation with a conditional second-moment bound
for the innovation and a negative first-moment bound for $v_{n-1}$ gives the
required coordinatewise estimate.

\begin{lemma}\label{le:neg_cor}
One has for all $ n \in \N $, $ i \in \{1,\dots,d\} $ that
\bas{
  \cov( ( [ v_n^{ (i) } ]^{ 1 / 2 } + \eps )^{ - 2 }, X^{ (i) }( \theta_{ n - 1 }, U_n )^2 | \cF_{ n - 1 } ) \le 0
}
on the event where the covariance is well-defined.
\end{lemma}

\begin{proof}
Fix $n\in\N$ and $i\in\{1,\dots,d\}$ and recall that
\bas{
v_n^{(i)}=\frac{\beta(1-\beta^{n-1})}{1-\beta^n}v_{n-1}^{(i)}+\frac{1-\beta}{1-\beta^n}X^{(i)}(\theta_{n-1},U_n)^2.
}
 In terms of the non-negative random variables
 \bas{
 A=\frac{\beta(1-\beta^{n-1})}{1-\beta^n}v_{n-1}^{(i)},\qquad B=\frac{1-\beta}{1-\beta^n}\quad
 \text{  and  } \quad Y=X^{(i)}(\theta_{n-1},U_n)^2}
 define the random $\cF_{n-1}$-measurable extended-valued function
 $G(x)=(\sqrt{A+Bx}+\eps)^{-2}$ for $x\in[0,\infty)$, where the value is
 $+\infty$ if the denominator vanishes. Then
 \bas{
  (\sqrt{v_n^{(i)}}+\eps)^{-2} = \frac1{(\sqrt{A+BY}+\eps)^2}=G(Y).
 }
For every $m\in\N$, set
\bas{
G_m(x)=\frac1{(\sqrt{A+Bx}+\eps+m^{-1})^2},
\qquad x\in[0,\infty).
}
The function $G_m$ is bounded and decreasing. If $Y'$ is conditionally on
$\cF_{n-1}$ independent of $Y$ and identically distributed as $Y$, then
\bas{
2\cov(G_m(Y),Y|\cF_{n-1})
&=\E[(G_m(Y)-G_m(Y'))(Y-Y')|\cF_{n-1}]\le0.
}
On the event where $\cov(G(Y),Y|\cF_{n-1})$ is well-defined, the random
variables entering its definition are integrable. Since $G_m\uparrow G$,
conditional monotone convergence yields
\bas{
\cov(G(Y),Y|\cF_{n-1})
=\lim_{m\to\infty}\cov(G_m(Y),Y|\cF_{n-1})\le0.
}
This finishes the proof.
\end{proof}

\begin{prop} \label{prop:prop_ii}
Let $ \beta \in [ \nicefrac{ 1 }{ 2 }, 1 ) $, $ \eps \in [0,\infty) $, $ n \in \N \backslash \{ 1 \} $, $ i \in \{1, \dots, d \} $,
let $ A_n \in \cF_{ n - 1 } $, and let $ C_X^{(2)}, C_v^{(-1)} \in (0,\infty) $.
Assume that
\bas{
\E[(X^{(i)}(\theta_{n-1},U_n))^2|\cF_{n-1}]\le C_X^{(2)}
\qquad\text{on }A_n
}
and
\bas{
\E[\1_{A_n}(v_{n-1}^{(i)})^{-1}]\le C_v^{(-1)}.
}
The negative-moment assumption implies that $v_{n-1}^{(i)}>0$ almost surely
on $A_n$. If $\eps=0$, inverse factors used in the proof are evaluated on
$A_n$ and extended by zero to $A_n^c$.
Then one has
\bas{\label{eq:prop-ii-bound}
\E\Bigl[\1_{A_n}(\sqrt{v_n^{(i)}}+\eps)^{-2}
X^{(i)}(\theta_{n-1},U_n)^2\Bigr]
\le 4 C_X^{(2)} C_v^{(-1)}.
}
\end{prop}
\begin{proof}
The negative-moment assumption implies that $v_{n-1}^{(i)}>0$ almost surely
on $A_n$. Since $\beta\ge\frac12$ and $n\ge2$, the recursion for $v_n$ also
gives $v_n^{(i)}>0$ almost surely on $A_n$. We extend all inverse factors by
zero to $A_n^c$. By \cref{le:neg_cor}, one has
\bas{\label{eq:prop-ii-neg-cor-step}
\1_{A_n}\E[(\sqrt{v_n^{(i)}}+\eps)^{-2}X^{(i)}(\theta_{n-1},U_n)^2|\cF_{n-1}]
\le \1_{A_n}\E[(\sqrt{v_n^{(i)}}+\eps)^{-2}|\cF_{n-1}]\,
\E[X^{(i)}(\theta_{n-1},U_n)^2|\cF_{n-1}].
}
The conditional second-moment assumption implies that, on $A_n$,
\bas{\label{eq:prop-ii-cond-second}
\E[X^{(i)}(\theta_{n-1},U_n)^2|\cF_{n-1}]
\le C_X^{(2)}.
}
Since $\frac{\beta(1-\beta^{n-1})}{1-\beta^n}\le 1$, the recursion for $v_n$ yields
\bas{
v_n^{(i)}
\ge \frac{\beta(1-\beta^{n-1})}{1-\beta^n} v_{n-1}^{(i)}.
}
Thus
\bas{\label{eq:prop-ii-inverse-vn-bound}
\1_{A_n}\E[(\sqrt{v_n^{(i)}}+\eps)^{-2}|\cF_{n-1}]
&\le \1_{A_n}\E[(v_n^{(i)})^{-1}|\cF_{n-1}]\\
&\le \1_{A_n}\frac{1-\beta^n}{\beta(1-\beta^{n-1})}(v_{n-1}^{(i)})^{-1}.
}
\cref{le:beta-power-ratio-monotone} and $\beta\ge\frac12$ imply that
\bas{\label{eq:prop-ii-beta-ratio-bound}
\frac{1-\beta^n}{\beta(1-\beta^{n-1})}
\le \frac{1}{\beta}\frac{n}{n-1}
\le 4.
}
Combining \eqref{eq:prop-ii-neg-cor-step}, \eqref{eq:prop-ii-cond-second}, \eqref{eq:prop-ii-inverse-vn-bound}, and \eqref{eq:prop-ii-beta-ratio-bound} gives
\bas{
\E[\1_{A_n}(\sqrt{v_n^{(i)}}+\eps)^{-2}X^{(i)}(\theta_{n-1},U_n)^2]
\le 4 C_X^{(2)}\E[\1_{A_n}(v_{n-1}^{(i)})^{-1}]\le 4 C_X^{(2)} C_v^{(-1)}.
}
This proves~\eqref{eq:prop-ii-bound}.
\end{proof}

\section{Initialisation estimates}
\label{sec:initialisation}

This section estimates the time required for each coordinate of the conditioner
to enter the regime where negative moments are available. We first convert the
conditional Laplace-transform bounds into lower-tail estimates and then iterate
these estimates over blocks whose lengths are comparable to the memory scale
$(1-\beta)^{-1}$. This yields an exponential bound, uniform in $\beta$, for the
probability that a coordinate has not started before the process exits the
prescribed region.

\begin{lemma}
\label{le:laplace-tail-bound}
Let $ c, q \in (0,\infty) $, $ \beta \in (0, 1) $, $ n \in \N $.
Let $ V_n^* $ be a nonnegative random variable which satisfies for all $\lambda \in[0,\infty)$ that
\bas{
\label{eq:7634589-4}
  -
  \ln \mathbb{E}\bigl[
    e^{-\lambda V_n^*}
  \bigr]
  \ge \sum_{ k = 1 }^n \min\Bigl\{ q, c \frac{1-\beta}{1-\beta^n} \beta^{n-k} \lambda \Bigr\} .
}
Then 
\bas{
  \P\bigl( V_n^* \le c / 2 \bigr)
  \le
  \exp\Bigl(
    - \frac q2 \, \frac{ 1 - \beta^n }{ 1 - \beta }
  \Bigr) .
}
\end{lemma}

\begin{proof}
Set
\bas{
\lambda_*=\frac{q(1-\beta^n)}{c(1-\beta)}.
}
By the exponential Chebyshev inequality, we have that
\bas{\label{eq:234621346}
\P(V_n^*\le c/2)
\le e^{\lambda_* c/2}\,\E[e^{-\lambda_*V_n^*}].
}
By~(\ref{eq:7634589-4}),
\bas{
\E[e^{-\lambda_*V_n^*}]
\le \exp\Bigl(-\sum_{k=1}^n \Bigl(c \frac{1-\beta}{1-\beta^n} \beta^{n-k} \lambda_* \Bigr)\wedge q\Bigr).
}
Now, for every $k\in\{1,\dots,n\}$, one has
\bas{
c \frac{1-\beta}{1-\beta^n} \beta^{n-k} \lambda_*
=q\beta^{n-k}\le q.
}
Hence all terms inside the minimum are unsaturated and therefore
\bas{
\sum_{k=1}^n \Bigl(c \frac{1-\beta}{1-\beta^n} \beta^{n-k} \lambda_* \Bigr)\wedge q
= q\sum_{k=1}^n \beta^{n-k}
= q\frac{1-\beta^n}{1-\beta}.
}
Consequently, we get with (\ref{eq:234621346}) that
\bas{
\P(V_n^*\le c/2)
\le \exp\Bigl(\lambda_* c/2-q\frac{1-\beta^n}{1-\beta}\Bigr)
= \exp\Bigl(-\frac q2\,\frac{1-\beta^n}{1-\beta}\Bigr) ,
}
which proves the claim.
\end{proof}

\begin{lemma}
\label{le:laplace-tail-bound2}
Let $ c, q \in (0,\infty) $, $ \beta \in (0, 1) $, $ n_0, n \in \N $ with $ n_0 < n $.
Let $ V_n^* $ be a nonnegative random variable which satisfies for all $ \lambda \in [0,\infty) $ that
\bas{\label{eq:7634589-5}
  - \ln
  \mathbb{E}\bigl[
    e^{ - \lambda V_n^* }
  \bigr]
  \ge
  \sum_{ k = n_0 + 1 }^n
  \min\Bigl\{ q, c \frac{1-\beta}{1-\beta^n} \beta^{n-k} \lambda \Bigr\} .
}
Then 
\bas{
\P\Bigl(V_n^*\le \frac c2\,\frac{1-\beta^{n-n_0}}{1-\beta^n}\Bigr)
\le \exp\Bigl(-\frac q2\,\frac{1-\beta^{n-n_0}}{1-\beta}\Bigr).
}
If, in addition, $n-n_0\ge \frac12(1-\beta)^{-1}$, then
\bas{
\P\Bigl(V_n^*\le \frac c2\,(1-e^{-1/2})\Bigr)
\le \exp\Bigl(-\frac q2\,\frac{1-e^{-1/2}}{1-\beta}\Bigr).
}
\end{lemma}

\begin{proof}
Set
\bas{
\lambda_*=\frac{q(1-\beta^n)}{c(1-\beta)}.
}
By the exponential Chebyshev inequality, we obtain
\bas{
\P\Bigl(V_n^*\le \frac c2\,\frac{1-\beta^{n-n_0}}{1-\beta^n}\Bigr)
\le \exp\Bigl(\lambda_* \frac c2\,\frac{1-\beta^{n-n_0}}{1-\beta^n}\Bigr)\,\E[e^{-\lambda_*V_n^*}].
}
By~(\ref{eq:7634589-5}),
\bas{
\E[e^{-\lambda_*V_n^*}]
\le \exp\Bigl(-\sum_{k=n_0+1}^n \Bigl(c \frac{1-\beta}{1-\beta^n} \beta^{n-k} \lambda_* \Bigr)\wedge q\Bigr).
}
Now, for every $k\in\{n_0+1,\dots,n\}$, one has
\bas{
c \frac{1-\beta}{1-\beta^n} \beta^{n-k} \lambda_*
=q\beta^{n-k}\le q.
}
Hence all terms inside the minimum are unsaturated and therefore
\bas{
\sum_{k=n_0+1}^n \Bigl(c \frac{1-\beta}{1-\beta^n} \beta^{n-k} \lambda_* \Bigr)\wedge q
= q\sum_{k=n_0+1}^n \beta^{n-k}
= q\frac{1-\beta^{n-n_0}}{1-\beta}.
}
Consequently,
\bas{
\P\Bigl(V_n^*\le \frac c2\,\frac{1-\beta^{n-n_0}}{1-\beta^n}\Bigr)
&\le \exp\Bigl(\lambda_* \frac c2\,\frac{1-\beta^{n-n_0}}{1-\beta^n}
-q\frac{1-\beta^{n-n_0}}{1-\beta}\Bigr)\\
&= \exp\Bigl(-\frac q2\,\frac{1-\beta^{n-n_0}}{1-\beta}\Bigr),
}
which proves the first assertion.

For the second assertion, note that $n-n_0\ge \frac12(1-\beta)^{-1}$ 
and hence
\bas{
\beta^{n-n_0}\le e^{-(1-\beta)(n-n_0)}\le e^{-1/2}.
}
Therefore 
\bas{
\frac{1-\beta^{n-n_0}}{1-\beta^n}\ge 1-\beta^{n-n_0}\ge 1-e^{-1/2}
}
and
\bas{
\Bigl\{V_n^*\le \frac c2(1-e^{-1/2})\Bigr\}
\subseteq
\Bigl\{V_n^*\le \frac c2\,\frac{1-\beta^{n-n_0}}{1-\beta^n}\Bigr\}.
}
Applying the first estimate finishes the proof:
\bas{
\P\Bigl(V_n^*\le \frac c2(1-e^{-1/2})\Bigr)
\le \exp\Bigl(-\frac q2\,\frac{1-e^{-1/2}}{1-\beta}\Bigr).
}
\end{proof}

\begin{prop}
\label{prop:3476}
Let $ c, q \in (0,\infty) $, $ \beta \in [ \nicefrac{ 1 }{ 2 }, 1 ) $,
let $ \cV \subseteq \R^d $ be measurable, and
let $ ( X, U ) $ be $ ( c, q ) $-regular on $ \cV $ in the sense of \cref{def:regular}.
Consider the stopping time
\bas{
  \bar\tau =
  \inf\bigl\{
    n \in \N_0 \colon \theta_n\not\in \cV
  \bigr\}
  \qquad
  \text{and}
  \qquad
  \sigma^{ (i) } =
  \inf\bigl\{
    n \in \N \colon v_n^{ (i) } \ge \tfrac{ c ( 1 - e^{ - 1 / 2 } ) }{ 2 }
  \bigr\} .
}
Then for all $ n \in \N $ one has
\bas{
  \P\bigl(
    \bar \tau \ge n, \sigma^{(i)} > n
  \bigr)
  \le
  \exp\Bigl(
    - \frac{ q n ( 1 - e^{ - 1 / 2 } ) }{ 4 }
  \Bigr)
  .
}
\end{prop}

\begin{proof}
Set
\bas{
\eta=\exp\Bigl(-\frac q2\,\frac{1-e^{-1/2}}{1-\beta}\Bigr).
}
First suppose that $n\le (1-\beta)^{-1}$. Then \cref{prop:7463} with $n_0=0$ and \cref{le:laplace-tail-bound} imply that
\bas{\label{eq:startup-short-prob-bound}
\P(\bar\tau\ge n,\sigma^{(i)}> n)
&\le \P\Bigl(\bar\tau\ge n, v_n^{(i)}\le \frac {1-e^{-1/2}}{2} c\Bigr)\\
&\le \P\Bigl(\bar\tau\ge n, v_n^{(i)}\le \frac c2\Bigr)\le \exp\Bigl(-\frac q2\,\frac{1-\beta^n}{1-\beta}\Bigr).
}
For $n=1$, this already implies the claim. Thus we may assume that $n\ge2$. Since
\bas{\label{eq:startup-geometric-lower-bound}
\frac{1-\beta^n}{1-\beta}=\sum_{k=0}^{n-1}\beta^k\ge n\beta^{n-1},
}
and since $n\le (1-\beta)^{-1}$ implies $\beta\ge 1-\frac1n$, we get for $n\ge2$ that
\bas{\label{eq:startup-beta-power-lower-bound}
\beta^{n-1}\ge \Bigl(1-\frac1n\Bigr)^{n-1}=\Bigl(\Bigl(1+\frac1{n-1}\Bigr)^{n-1} \Bigr)^{-1}\ge e^{-1}\ge \frac14. 
}
Consequently, by \eqref{eq:startup-short-prob-bound}, \eqref{eq:startup-geometric-lower-bound}, and \eqref{eq:startup-beta-power-lower-bound},
\bas{
\P(\bar\tau\ge n,\sigma^{(i)}> n)\le \exp\Bigl(-\frac q8\,n\Bigr)
\le \exp\Bigl(-\frac {q(1-e^{-1/2})}{4}\,n\Bigr).
}

Now suppose that $n>(1-\beta)^{-1}$ and set
\bas{
L=\Bigl\lceil \frac1{2(1-\beta)}\Bigr\rceil\quad\text{ and } \quad J=\Bigl\lfloor \frac nL\Bigr\rfloor.
}
Since $\beta\ge \frac12$, one has
\bas{
\frac1{2(1-\beta)}\le L\le \frac1{1-\beta}.
}
Moreover,  $J\ge 1$ and
\bas{\label{eq:745351}
J\ge \frac n{2L}\ge \frac{n(1-\beta)}2.
}
Setting, for all $j=2,\dots,J$,
\bas{
t_0=0 \quad \text{  and }\quad t_j=n-JL+jL,
}
we obtain a partition
\bas{
0=t_0<t_1<\dots<t_J=n
}
such that all $j=1,\dots,J$
\bas{
t_j-t_{j-1} \ge L\ge \frac1{2(1-\beta)}.
}
We now show by induction that
\bas{\label{eq:block-step-bound}
\P(\bar\tau\ge t_j,\sigma^{(i)}> t_j)\le \eta^j,
\qquad j\in\{1,\dots,J\}.
}
For $j=1$, \cref{prop:7463} with $n_0=0$ and \cref{le:laplace-tail-bound} give
\bas{
\P(\bar\tau\ge t_1,\sigma^{(i)}> t_1)
&\le \P\Bigl(\bar\tau\ge t_1, v_{t_1}^{(i)}\le \frac {1-e^{-1/2}}{2} c\Bigr)\\
&\le \P\Bigl(\bar\tau\ge t_1, v_{t_1}^{(i)}\le \frac c2\Bigr)\le \exp\Bigl(-\frac q2\,\frac{1-\beta^{t_1}}{1-\beta}\Bigr).
}
Since $t_1\ge \frac1{2(1-\beta)}$, we have 
\bas{\beta^{t_1}= \exp\{t_1 \ln \beta \}\le \exp\{-t_1(1-\beta)\}\le  e^{-1/2}}
and hence
\bas{
\P(\bar\tau\ge t_1,\sigma^{(i)}> t_1)\le \eta.
}

Next let $j\in\{2,\dots,J\}$ and assume that
\bas{
\P(\bar\tau\ge t_{j-1},\sigma^{(i)}> t_{j-1})\le \eta^{j-1}.
}
Set $A_{j-1}=\{\bar\tau> t_{j-1},\sigma^{(i)}> t_{j-1}\}$.
Then
\bas{\label{eq:startup-induction-conditioning}
\P(\bar\tau\ge t_j,\sigma^{(i)}> t_j)
=\E\Bigl[\1_{A_{j-1}}
\P\bigl(\bar\tau\ge t_j,\sigma^{(i)}> t_j\bigm|\cF_{t_{j-1}}\bigr)\Bigr].
}
Define the shifted exit time
\bas{
\bar\tau^{(j)}=\inf\{m\in\N_0:m\ge t_{j-1}\text{ and }\theta_m\notin\cV\}.
}
On $A_{j-1}$ this shifted stopping time agrees with $\bar\tau$ up to the interval $[t_{j-1},t_j]$.
\cref{prop:7463}, applied conditionally on $\cF_{t_{j-1}}$ with initial time $t_{j-1}$ and exit time $\bar\tau^{(j)}$, implies that the random variable
\bas{
V_{t_j}^*=v_{t_j}^{(i)}+\1_{\{t_j\le \bar\tau^{(j)}\}^c}\infty
}
satisfies the Laplace bound in \cref{le:laplace-tail-bound2}. Since
\bas{
t_j-t_{j-1}\ge \frac1{2(1-\beta)},
}
\cref{le:laplace-tail-bound2} yields, on $A_{j-1}$,
\bas{\label{eq:startup-block-conditional-small-v}
\P\Bigl(V_{t_j}^*\le \frac c2(1-e^{-1/2})\Bigm|\cF_{t_{j-1}}\Bigr)\le \eta.
}
Moreover, on $A_{j-1}$,
\bas{\label{eq:startup-block-event-inclusion}
\{\bar\tau\ge t_j,\sigma^{(i)}> t_j\}
\subseteq
\Bigl\{V_{t_j}^*\le \frac c2(1-e^{-1/2})\Bigr\}.
}
Consequently, by \eqref{eq:startup-induction-conditioning}, \eqref{eq:startup-block-conditional-small-v}, and \eqref{eq:startup-block-event-inclusion},
\bas{
\P(\bar\tau\ge t_j,\sigma^{(i)}> t_j)
&\le \eta\,\P(A_{j-1})\le \eta\,\P(\bar\tau\ge t_{j-1},\sigma^{(i)}> t_{j-1})
\le \eta^j.
}
This proves~(\ref{eq:block-step-bound}).

Since $t_J=n$, we conclude with~(\ref{eq:745351}) that
\bas{
\P(\bar\tau\ge n,\sigma^{(i)}> n)\le \eta^J
\le \exp\Bigl(-\frac q2\,\frac{1-e^{-1/2}}{1-\beta}\,\frac{n(1-\beta)}2\Bigr)
= \exp\Bigl(-\frac {q(1-e^{-1/2})}{4}\,n\Bigr),
}
which proves the claim.
\end{proof}

\section{Combining the estimates before and after the coordinate startup times}
\label{sec:combined-startup-estimate}

This section assembles the preceding estimates into the final error bound. At
every iteration, each coordinate is split according to whether its startup time
has occurred: the negative-moment, covariance, and quadratic estimates are used
in the late regime, while direct self-normalization and the startup tail bound
control the early regime. Together, these estimates yield the general recursion and, after
convolution in numerical time, the main theorem with constants uniform in
$\beta$ and $\epsilon$.

\begin{theorem}
\label{thm:combined-early-late-bounded}
Let
$ \beta \in [ \nicefrac{ 1 }{ 2 } , 1 ) $, $ \eps \in [0,\infty) $, $ \rho, c, q, C_X, L_f, C_{ \mathrm{Loj} } \in (0,\infty) $,
let $ V \subseteq \R^d $ be measurable,
and suppose that the following is true:
\begin{enumerate}[label=(\alph*)]
\item[(a)] 
It holds that
\bas{\label{eq:combined-regularity-parameter-condition}
  q+3\ln\beta\ge\rho
  .
}
\item[(b)]  $(X,U)$ is
$(c,q)$-regular on $V$ in the sense of \cref{def:regular} and uniformly
bounded: for every $\theta\in V$, $i\in\{1,\dots,d\}$ one has
\bas{\label{eq:combined-uniform-innovation-bound}
  \|X^{(i)}(\theta,U)\|_{L^\infty}\le C_X.
}
\item[(c)] There exists $ F \colon \R^d\to[0,\infty)$ with globally $L_f$-Lipschitz
continuous derivative such that for every $ \theta \in V $ one has
\bas{\label{eq:combined-gradient-representation}
  f_X(\theta)=-\grad F(\theta)
}
and
\bas{\label{eq:combined-upper-lojasiewicz-bound}
  F(\theta)\le C_{\mathrm{Loj}}|f_X(\theta)|^2.
}
\item[(d)] The step-size sequence $(\gamma_n)_{n\in\N}$ is
$(0,\infty)$-valued and satisfies for all $ n \in \N \backslash \{ 1 \} $ that
\bas{\label{eq:combined-step-smallness}
  \gamma_n \le C_{\mathrm{Loj}}(C_X+\eps) .
}
\end{enumerate}
Set
\bas{
\label{eq:definition_of_kappa1_to_3}
\begin{aligned}
\kappa_0&=\frac{4}{1-e^{-1/2}},
&\qquad \kappa_1&=\frac{1}{2C_{\mathrm{Loj}}(C_X+\eps)},\\
\kappa_2&=\frac{8\kappa_0^3dL_fC_X^6C_{\mathrm{Loj}}(C_X+\eps)}{c^3}
D_{\rho,3},
&\qquad \kappa_3&=\frac{\kappa_0 dL_fC_X^2}{c}D_{\rho,1},
\end{aligned}
}
and, for every $ n \in \N \backslash \{ 1 \} $,
\bas{\label{eq:combined-early-remainder-definition}
R_n^{\mathrm{early}}
&=d\Bigl(2C_Xn^{1/2}
+\frac{L_f}{2}n\gamma_n\Bigr)e^{-q(n-1)/\kappa_0}\gamma_n.
}
For every $ n \in \N_0 $ set $ t_n = \sum_{ k = 1 }^n \gamma_k $.
Then one has for every $ n \in \N $ that
\bas{
\label{eq:combined-early-late-iterated}
&
  \E\bigl[
    F( \theta_n )
    \1_{ \{ \tau \ge n \} }
  \bigr]
\\ &
  \le
  \bigl[
    F(\theta_0)^{1/2}
    +
    (
      2^{ - 1 } d L_f
    )^{ 1 / 2 }
    \gamma_1
  \bigr]^2
  e^{ \kappa_1 ( t_1 - t_n ) }
  +
  \sum_{ j = 2 }^n
  \left[
    \kappa_2
    \bigl[
      ( j - 1 )^{ - 1 } \vee (1-\beta)
    \bigr]^2
    \gamma_j
    +
    \kappa_3
    \gamma_j^2
    +
    R_j^{ \mathrm{early} }
  \right]
  e^{ \kappa_1 (t_j - t_n) }
  ,
}
where
$
  \tau = \inf\{ n \in \N_0 \colon \theta_n \notin V \}
$.

\end{theorem}
\begin{proof}
If $ \theta_0 \notin V$, then $ \tau = 0 $
and the left-hand side of \eqref{eq:combined-early-late-iterated} vanishes.
Hence we may assume without loss of generality that
$ \theta_0 \in V $. Set
\bas{\label{eq:combined-proof-local-constants}
  c_v = (C_X + \eps)^{ - 1 }
  ,
  \qquad
  c_{ \mathrm{Loj} } = ( 2 L_f )^{ - 1 } .
}
For every $ n \in \N_0 $ set
\bas{
\label{eq:combined-stopped-objective-definition}
  \varphi_n
  =
  \E\bigl[
    F( \theta_n ) \1_{ \{ \tau \ge n \} }
  \bigr]
  .
}
For every $x\in\R^d$, the $L_f$-smoothness inequality, applied with
$h=-L_f^{-1}\grad F(x)$, and the nonnegativity of $F$ give \bas{\label{eq:combined-self-bounding-estimate}
\begin{gathered}
\textstyle
  0
 \le F( x - L_f^{ - 1 } \grad F(x) )
 \le F( x ) - \frac{ 1 }{ 2 L_f } | \grad F(x) |^2
 ,
\qquad
  | \grad F(x) |^2
  \le 2 L_f F(x)
  .
\end{gathered}
}
Thus the choice of $c_{\mathrm{Loj}}$ in
\eqref{eq:combined-proof-local-constants} is admissible by
\eqref{eq:combined-gradient-representation} and
\eqref{eq:combined-self-bounding-estimate}.
For every $ i \in \{ 1, \dots, d \} $ define
$
  \sigma^{ (i) }
  =
  \inf\{ n \in \N \colon v_n^{ (i) } \ge 2 c / \kappa_0 \}
$.
For every $ n \ge 2 $ set
\bas{\label{eq:combined-step-local-definitions}
\begin{aligned}
h_n&=(n-1)^{-1}\vee(1-\beta),
&
p_n&=e^{-q(n-1)/\kappa_0},
&
A_n&=\{\tau\ge n\},
\\
G_{n,i}&=\{\sigma^{(i)}\le n-1\},
&
Y_{n,i}
&=(\sqrt{v_n^{(i)}}+\eps)^{-1}
X^{(i)}(\theta_{n-1},U_n) .
\end{aligned}
}
By \eqref{eq:combined-step-local-definitions}, the events $A_n$ and
$G_{n,i}$ are $\cF_{n-1}$-measurable. On $G_{n,i}$,
let
\begin{equation}\label{eq:combined-xi-definition}
  \xi_{n,i}
  =
  \cov\bigl(
    (\sqrt{v_n^{(i)}}+\eps)^{-1},
    X^{(i)}(\theta_{n-1},U_n)
    \mid\cF_{n-1}
  \bigr).
\end{equation}
When $\eps=0$, inverse factors and covariances used with
$\1_{G_{n,i}}$ are extended by zero to $G_{n,i}^c$.

\emph{Step 1.} \emph{Late coordinates satisfy uniform moment bounds, while
early coordinates are exponentially rare.}

On $ A_n $, \eqref{eq:v_n_rec},
\eqref{eq:combined-uniform-innovation-bound}, and
\eqref{eq:combined-step-local-definitions} give
$0\le v_n^{(i)}\le C_X^2$; hence, by
\eqref{eq:combined-proof-local-constants},
$(\sqrt{v_n^{(i)}}+\eps)^{-1}\ge c_v$ on $A_n\cap G_{n,i}$.
Put $\widetilde c=2c/\kappa_0$. Since
$\widetilde c<c$, $(X,U)$ is also $(\widetilde c,q)$-regular, and the
corresponding startup time in \cref{le:moment_bound} is $\sigma^{(i)}$.
Since
$A_n\cap G_{n,i}\subseteq
\{\tau\ge n-1,\,\sigma^{(i)}\le n-1\}$ and since
$q+r\ln\beta\ge\rho$ for every $r\in\{1,3\}$,
\cref{le:moment_bound} gives, for every $r\in\{1,3\}$,
\begin{equation}\label{eq:combined-late-inverse-moment-bound}
  \E[\1_{A_n\cap G_{n,i}}(v_{n-1}^{(i)})^{-r}]
  \le\widetilde c^{-r}D_{\rho,r}.
\end{equation}
Moreover,
\begin{equation}\label{eq:combined-beta-quotient-bound}
  \frac{1-\beta}{1-\beta^{n-1}}\le2h_n
\end{equation}
follows from \cref{le:beta-quotient-regimes} and
\eqref{eq:combined-step-local-definitions}. Combining
\eqref{eq:combined-late-inverse-moment-bound} and
\eqref{eq:combined-beta-quotient-bound} with \cref{prop:prop_ii}
and~\cref{prop:prop_iii} and
\eqref{eq:combined-uniform-innovation-bound} yields
\bas{\label{eq:combined-coordinate-bounds}
\E\Bigl[
\1_{A_n}\sum_{i=1}^d\1_{G_{n,i}}Y_{n,i}^2
\Bigr]
&\le \frac{2d\kappa_0C_X^2}{c}D_{\rho,1},
\qquad
\E\Bigl[
\1_{A_n}\sum_{i=1}^d\1_{G_{n,i}}\xi_{n,i}^2
\Bigr]^{1/2}
\le
8\sqrt d\,C_X^3
\Bigl(\frac{\kappa_0}{2c}\Bigr)^{3/2}
D_{\rho,3}^{1/2}h_n.
}
Moreover, \eqref{eq:v_n_rec} and
\eqref{eq:combined-step-local-definitions} give
$|Y_{n,i}|\le n^{1/2}$, while \cref{prop:3476} and
\eqref{eq:combined-step-local-definitions} give
$\P(A_n\cap G_{n,i}^c)\le p_n$. Since
$|f_X^{(i)}|\le C_X$ on $V$, we obtain
\bas{\label{eq:combined-early-coordinate-bounds}
\E\Bigl[
\1_{A_n}\sum_{i=1}^d\1_{G_{n,i}^c}
|f_X^{(i)}(\theta_{n-1})|^2
\Bigr]
\le dC_X^2p_n,\quad
\E\Bigl[
\1_{A_n}\sum_{i=1}^d\1_{G_{n,i}^c}
|f_X^{(i)}(\theta_{n-1})Y_{n,i}|
\Bigr]
\le dC_Xn^{1/2}p_n,\\
\text { and } \quad \E\Bigl[
\1_{A_n}\sum_{i=1}^d\1_{G_{n,i}^c}Y_{n,i}^2
\Bigr]
\le dnp_n.
}
\emph{Step 2.} \emph{The coordinate estimates and
\cref{prop:F-recursion-md-loj-event} yield a single Lyapunov recursion.}

The lower bound for the inverse factor on $A_n\cap G_{n,i}$ proves
condition~\emph{(i)} of
\cref{prop:F-recursion-md-loj-event}. Conditions~\emph{(ii)}
and~\emph{(iii)} are given by
\eqref{eq:combined-coordinate-bounds}, while the required \L{}ojasiewicz
bounds follow from \eqref{eq:combined-upper-lojasiewicz-bound},
\eqref{eq:combined-gradient-representation},
\eqref{eq:combined-self-bounding-estimate}, and
\eqref{eq:combined-proof-local-constants}. Since $\theta_0\in V$, the
stopping times agree for $n_0=1$. Hence
\eqref{eq:F-recursion-late-young}, applied with $n_0=1$, together with
\eqref{eq:combined-early-coordinate-bounds} and the definitions in
\eqref{eq:definition_of_kappa1_to_3}, gives
\bas{
\varphi_n
&\le
\Bigl(1-\frac{c_v}{2C_{\mathrm{Loj}}}\gamma_n\Bigr)
\E[\1_{A_n}F(\theta_{n-1})]
+\kappa_2h_n^2\gamma_n+\kappa_3\gamma_n^2\\
&\quad
+d(c_vC_X^2+C_Xn^{1/2})p_n\gamma_n
+\frac{dL_f}{2}np_n\gamma_n^2.
}
By \eqref{eq:combined-step-smallness}, the coefficient in the first line is
nonnegative, and
$\E[\1_{A_n}F(\theta_{n-1})]\le\varphi_{n-1}$.
Since \eqref{eq:combined-proof-local-constants} gives
$c_vC_X^2\le C_X\le C_Xn^{1/2}$,
\eqref{eq:combined-early-remainder-definition} now yields
\bas{
\label{eq:combined-early-late-young-recursion}
\varphi_n
\le(1-\kappa_1\gamma_n)\varphi_{n-1}
+\kappa_2h_n^2\gamma_n+\kappa_3\gamma_n^2+R_n^{\mathrm{early}}.
}

\emph{Step 3.} \emph{The scalar recursion can be iterated.}
For every $i\in\{1,\dots,d\}$, equations \eqref{eq:v_n_rec} and
\eqref{eq:zero-normalization-convention}, applied at $n=1$, give the first
line below; together with \eqref{eq:F-recursion-late-update}, this yields the
second:
\bas{\label{eq:combined-first-step-increment}
\left|(\sqrt{v_1^{(i)}}+\eps)^{-1}X^{(i)}(\theta_0,U_1)\right|
&\le1,
\qquad
|\theta_1-\theta_0|\le\sqrt d\,\gamma_1.
}
Equations \eqref{eq:F-recursion-late-smoothness},
\eqref{eq:combined-gradient-representation},
\eqref{eq:combined-self-bounding-estimate}, and
\eqref{eq:combined-first-step-increment} yield
\bas{\label{eq:combined-first-step-objective-bound}
F(\theta_1)
&\le F(\theta_0)
+|\grad F(\theta_0)|\,|\theta_1-\theta_0|
+\frac{L_f}{2}|\theta_1-\theta_0|^2
\le
\left[F(\theta_0)^{1/2}
+\Bigl(\frac{dL_f}{2}\Bigr)^{1/2}\gamma_1\right]^2.
}
By \eqref{eq:combined-stopped-objective-definition},
$\varphi_1\le\E[F(\theta_1)]$; hence
\eqref{eq:combined-first-step-objective-bound} controls $\varphi_1$ by the
initial term in \eqref{eq:combined-early-late-iterated}. Finally,
\eqref{eq:combined-step-smallness} gives $\kappa_1\gamma_n\le1/2$.
Iterating \eqref{eq:combined-early-late-young-recursion} from time $1$ and
using $1-x\le e^{-x}$ proves \eqref{eq:combined-early-late-iterated}.
\end{proof}

\begin{proof}[Proof of \cref{thm:combined-early-late-numerical-time}]
Fix $\eps\in[0,1]$. The constant $\kappa_1$ from
\cref{thm:combined-early-late-bounded} is
\bas{\label{eq:combined-main-epsilon-kappa-one}
\kappa_1
=
  \bigl[ 2 C_{ \mathrm{Loj} } ( C_X + \eps ) \bigr]^{ - 1 }
  .
}
The assumptions of \cref{thm:combined-early-late-bounded} are satisfied. Since $\eps\le1$,
\eqref{eq:combined-main-eta} and
\eqref{eq:combined-main-epsilon-kappa-one} give
\bas{
\label{eq:combined-main-uniform-kappa-bounds}
  \eta
  \le
  \kappa_1
  \le
  [
    2 C_{ \mathrm{Loj} } C_X
  ]^{ - 1 }
  .
}
Set $\overline\gamma=C_{\mathrm{Loj}}(C_X+1)=(2\eta)^{-1}$.
Then the assumption that
$
  ( \gamma_n )_{ n \in \N }
$
is $ ( 0, C_{ \mathrm{Lo} } ( C_X + \varepsilon ) ] $-valued
and the fact that $\eps\le1$ show that,
for every $j\ge2$,
\bas{\label{eq:combined-main-uniform-step-bound}
\gamma_j\le\overline\gamma.
}
By \eqref{eq:definition_of_kappa1_to_3}, the constant $\kappa_2$ satisfies
\bas{
\label{eq:combined-main-uniform-kappa-two-bound}
  \kappa_2
  \le
  \frac{8\kappa_0^3dL_fC_X^6C_{\mathrm{Loj}}(C_X+1)}{c^3}
  D_{\rho,3} =: \bar\kappa_2
  ,
}
whereas, by \eqref{eq:definition_of_kappa1_to_3} and
\eqref{eq:combined-early-remainder-definition}, $ \kappa_3 $ and
$ R_j^{ \mathrm{early} } $ do not depend on $ \eps $.
Since $ \kappa_1 \ge \eta $ and
$
  e^{ \kappa_1 \gamma_1 } e^{ - \kappa_1 t_n }
  =
  e^{ - \kappa_1 ( t_n - t_1 ) }
$,
the general estimate
\eqref{eq:combined-early-late-iterated} implies,
for every $ n \ge 2 $,
\bas{
\label{eq:combined-early-late-iterated-uniform-rate}
  \E\bigl[
    \1_{ \{ \tau \ge n \} } F( \theta_n )
  \bigr]
&
  \le
  \biggl[
    F( \theta_0 )^{ 1 / 2 }
    +
    \Bigl(
      \frac{ d L_f }{ 2 }
    \Bigr)^{ 1 / 2 }
    \gamma_1
  \biggr]^2
  e^{ - \eta (t_n - t_1) }
\\
&
\quad
  + \sum_{ j = 2 }^n
  \left[
    \kappa_2
    \Bigl(
      \frac1{ j - 1 } \vee ( 1 - \beta )
    \Bigr)^2
    \gamma_j
    +
    \kappa_3
    \gamma_j^2
    +
    R_j^{ \mathrm{early} }
  \right]
  e^{ - \eta (t_n - t_j) }
  .
} 
For $a\in(0,\infty)$, the monotonicity of $(\gamma_j)_{j\in\N}$ and
\eqref{eq:combined-main-uniform-step-bound} give
\bas{\label{eq:combined-numerical-time-kernel-bound}
\sum_{j=2}^n\gamma_j e^{-a(t_n-t_j)}
&\le e^{a\overline\gamma}\sum_{j=2}^n
\int_{t_{j-1}}^{t_j}e^{-a(t_n-s)}\,\dd s
\le\frac{e^{a\overline\gamma}}{a}.
}
Since $(\frac1{j-1}\vee(1-\beta))^2
\le(j-1)^{-2}+(1-\beta)^2$,
\eqref{eq:combined-main-step-smallness} and
\eqref{eq:combined-numerical-time-kernel-bound}, applied with
$a=\eta$, yield
\bas{\label{eq:combined-time-late-convolution-bound}
\sum_{j=2}^n
\Bigl(\frac1{j-1}\vee(1-\beta)\Bigr)^2\gamma_j
e^{-\eta(t_n-t_j)}&\le
(1-\beta)^2\sum_{j=2}^n\gamma_j e^{-\eta(t_n-t_j)}
+C_{\mathrm{dec}}\gamma_n
\sum_{j=2}^n\frac{e^{-\eta(t_n-t_j)/2}}{(j-1)^2}\\
&\le
\frac{e^{1/2}}{\eta}\,(1-\beta)^2
+C_{\mathrm{dec}}\gamma_n\sum_{j=2}^n\frac1{(j-1)^2}\\
&=
\frac{e^{1/2}}{\eta}\,(1-\beta)^2
+\frac{\pi^2}{6}C_{\mathrm{dec}}\gamma_n.
}
Similarly, \eqref{eq:combined-main-step-smallness} and
\eqref{eq:combined-numerical-time-kernel-bound}, now applied with $a=\eta/2$,
give
\bas{\label{eq:combined-time-quadratic-convolution-bound}
\sum_{j=2}^n\gamma_j^2e^{-\eta(t_n-t_j)}
&\le C_{\mathrm{dec}}\gamma_n
\sum_{j=2}^n\gamma_j e^{-\eta(t_n-t_j)/2}\le \frac{2C_{\mathrm{dec}}e^{1/4}}{\eta}\,\gamma_n.
}
Equations \eqref{eq:combined-early-remainder-definition} and
\eqref{eq:combined-main-uniform-step-bound} give
\bas{\label{eq:combined-time-early-pointwise-bound}
R_j^{\mathrm{early}}
\le d\left(2C_Xj^{1/2}
+\frac{L_f}{4\eta}j\right)
e^{-q(j-1)/\kappa_0}\gamma_j.
}
The constant
\bas{\label{eq:combined-time-early-constant}
C_{\mathrm{early}}^{\mathrm{time}}
=dC_{\mathrm{dec}}\sum_{j=2}^\infty
\left(2C_Xj^{1/2}
+\frac{L_f}{4\eta}j\right)
e^{-q(j-1)/\kappa_0}
}
is finite. Consequently,
\eqref{eq:combined-main-step-smallness} and
\eqref{eq:combined-time-early-pointwise-bound} imply
\bas{\label{eq:combined-time-early-convolution-bound}
\sum_{j=2}^nR_j^{\mathrm{early}}e^{-\eta(t_n-t_j)}
&\le dC_{\mathrm{dec}}\gamma_n\sum_{j=2}^n
\left(2C_Xj^{1/2}
+\frac{L_f}{4\eta}j\right)
e^{-q(j-1)/\kappa_0}e^{-\eta(t_n-t_j)/2}\le C_{\mathrm{early}}^{\mathrm{time}}\gamma_n.
}
For $n=1$, \eqref{eq:combined-early-late-numerical-time} follows directly
from \eqref{eq:combined-early-late-iterated}. For $n\ge2$, equations
\eqref{eq:combined-early-late-iterated-uniform-rate},
\eqref{eq:combined-main-uniform-kappa-two-bound},
\eqref{eq:combined-time-late-convolution-bound},
\eqref{eq:combined-time-quadratic-convolution-bound}, and
\eqref{eq:combined-time-early-constant}--\eqref{eq:combined-time-early-convolution-bound} prove
\eqref{eq:combined-early-late-numerical-time} with the choices
\bas{
\label{eq:def_of_C0_and_C1}
C_0
&=\frac{\pi^2}{6}\overline\kappa_2C_{\mathrm{dec}}
+\frac{2\kappa_3C_{\mathrm{dec}}e^{1/4}}{\eta}
+C_{\mathrm{early}}^{\mathrm{time}},
\qquad 
C_1
=\frac{\overline\kappa_2e^{1/2}}{\eta},
}
which do not depend on $\eps$ or $\beta$.
\end{proof}

\begin{rem}[Non-asymptotic error analysis with fully explicit error constants]
\label{rem:explicit_constants_II}
In \cref{thm:combined-early-late-numerical-time}
all constants, except of $ C_0 $ and $ C_1 $,
are explicitly specified and, actually, even $ C_0 $ and $ C_1 $
in \cref{thm:combined-early-late-numerical-time}
can be explicitly specified
(see \cref{eq:def_of_C0_and_C1} in the proof of
\cref{thm:combined-early-late-numerical-time})
and explicitly estimated from above.
This is precisely the subject of this remark.
More formally, combining
\cref{eq:combined-main-eta,eq:D-rho-p,eq:definition_of_kappa1_to_3,eq:combined-main-uniform-kappa-two-bound,eq:combined-time-early-constant,eq:def_of_C0_and_C1}
shows that
\bas{
  C_0
& =
  6^{ - 1 } \pi^2 \overline{\kappa}_2 C_{ \mathrm{dec} }
  +
  \eta^{ - 1 } 2 \kappa_3 C_{ \mathrm{dec} } e^{ 1 / 4 }
  +
  C_{ \mathrm{early} }^{ \mathrm{time} }
\\ &
  =
  6^{ - 1 } \pi^2 C_{ \mathrm{dec} }
  \biggl[
    \frac{ 8 \kappa_0^3 L_f C_X^6 C_{ \mathrm{Loj} } (C_X + 1) d }{ c^3 }
  \biggr]
  D_{ \rho, 3 }
  +
  \eta^{ - 1 } 2 \kappa_3 C_{ \mathrm{dec} }
  \exp\bigl(
    1 / 4
  \bigr)
\\ &
\quad
  +
  C_{ \mathrm{dec} }
  d
  \biggl[
    \sum_{ j = 2 }^{ \infty }
    \left(
      2 C_X j^{ 1 / 2 }
      +
      \frac{ L_f }{ 4 \eta }j
    \right)
    \exp\bigl(
      - q (j - 1) / \kappa_0
    \bigr)
  \biggr]
  .
}
Therefore, we obtain
\bas{
  C_0
&
  =
  6^{ - 1 } \pi^2 C_{ \mathrm{dec} }
  \biggl[
    \frac{ 4 \kappa_0^3 L_f C_X^6 d }{ c^3 \eta }
  \biggr]
  \biggl[
    \sum_{ r = 0 }^{ \infty }
    \frac{
      ( r + 1 )^3
    }{
      \exp( \rho r )
    }
  \biggr]
  +
  \biggl[
    \frac{ 2 \kappa_0 L_f C_X^2 d }{ \eta c }
  \biggr]
  \biggl[
    \sum_{ r = 0 }^{ \infty }
    \frac{
      ( r + 1 )
    }{
      \exp( \rho r )
    }
  \biggr]
  C_{ \mathrm{dec} }
  \exp\bigl(
    1 / 4
  \bigr)
\\ &
\quad
  +
  C_{ \mathrm{dec} }
  d
  \biggl[
    \sum_{ j = 2 }^{ \infty }
    \left(
      2 C_X j^{ 1 / 2 }
      +
      \frac{ L_f j }{ 4 \eta }
    \right)
    \exp\bigl(
      - q (j - 1) / \kappa_0
    \bigr)
  \biggr]
  .
}
Hence, we obtain
\bas{
\label{eq:C_0_rep}
  C_0
&
  =
  C_{ \mathrm{dec} } d
  \Biggl[
    \sum_{ r = 0 }^{ \infty }
    \frac{
      2^7 \pi^2
      c^{ - 3 } \eta^{ - 1 }
      L_f C_X^6
      ( r + 1 )^3
    }{
      3 ( 1 - \exp( - 1 / 2 ) )^3 \exp( \rho r )
    }
    +
    \sum_{ r = 0 }^{ \infty }
    \frac{
      8 \eta^{ - 1 } c^{ - 1 } L_f C_X^2 ( r + 1 )
    }{
      ( 1 - \exp( - 1 / 2 ) ) \exp( \rho r - 1 / 4 )
    }
\\ &
\quad
  +
    \sum_{ j = 2 }^{ \infty }
    \left(
      2 C_X j^{ 1 / 2 }
      +
      \frac{ L_f j }{ 4 \eta }
    \right)
    \exp\bigl(
      - q (j - 1) ( 1 - \exp( - 1 / 2 ) ) / 4
    \bigr)
  \Biggr]
  .
}
Using
$
  2^7 \pi^2 3^{ - 1 }( 1 - \exp( - 1 / 2 ) )^{ - 3 } \leq 2^7 ( 63 ) = 2^7 ( 2^6 - 1 ) = 2^{ 13 } - 2^7
$,
$ ( 1 - \exp( - 1 / 2 ) )^{ - 1 } \leq 4 = 2^{ 2 } $,
$j^{1/2}\le j$,
$
  [ \exp( - 1 / 4 ) ]^{ - 1 }
  \leq
  \exp( 1 / 2 ) \leq 2
$,
$
  ( 1 - \exp( - 1 / 2 ) ) / 4 \ge 2^{ - 4 }
$,
and the change of index
$r=j-1$ in the last series of \eqref{eq:C_0_rep}, we obtain
\bas{
  C_0
&
  \leq
  C_{ \mathrm{dec} } d
  ( L_f + C_X + 1 )^7
  ( \eta^{-1} + 1 )
  \Biggl[
    \sum_{ r = 0 }^{ \infty }
    \biggl(
      \frac{
        ( ( 2^{ 13 } - 2^7 ) c^{ - 3 } + 2^6 c^{ - 1 } )
        (r+1)^3
      }{
        \exp( \rho r )
      }
      +
      \frac{ 2^2 (r + 1) }{
        \exp( 2^{ - 4 } q r )
      }
    \biggr)
  \Biggr]
  .
}
Combining this with the fact that
$
  2^6 c^{ - 1 } + 2^2
  \leq
  ( 2^6 + 2^2 ) ( c^{ - 3 } + 1)
  \leq
  2^7 ( c^{ - 3 } + 1 )
$
and the fact that
for all $ r \in \N_0 $ it holds that $ ( r + 1 ) \leq ( r + 1 )^3 $
proves that
\bas{
\label{eq:C_0_bound_from_above_B}
  C_0
&
  \leq
  2^{ 13 }
  C_{ \mathrm{dec} } d
  ( L_f + C_X + 1 )^7
  ( \eta^{-1} + 1 )
  \Biggl[
    \sum_{ r = 0 }^{ \infty }
      \frac{
        ( c^{ - 3 } + 1 )
        (r+1)^3
      }{
        \exp( \min\{ \rho, 2^{ - 4 } q \} r )
      }
  \Biggr]
  .
}
In addition,
combining
\cref{eq:combined-main-eta,eq:D-rho-p,eq:definition_of_kappa1_to_3,eq:combined-main-uniform-kappa-two-bound,eq:def_of_C0_and_C1}
shows that
\bas{
  C_1
& =
  \overline{\kappa}_2 \eta^{ - 1 } e^{ 1 / 2 }
  =
  \bigl(
    c^{ - 3 }
    8 \kappa_0^3 d L_f C_X^6 C_{ \mathrm{Loj} } (C_X + 1)
    D_{\rho,3}
  \bigr)
  \eta^{ - 1 }
  \exp( 1 / 2 )
  =
  \bigl(
    c^{ - 3 }
    4 \kappa_0^3 d L_f C_X^6
    D_{\rho,3}
  \bigr)
  \eta^{ - 2 }
  \exp( 1 / 2 )
\\
& =
  \biggl[
    \frac{
      4 \kappa_0^3
      \exp( 1 / 2 )
      L_f C_X^6 d
    }{
      c^3 \eta^2
    }
  \biggr]
  \biggl[
    \sum_{ r = 0 }^{ \infty }
    (r + 1)^3 e^{ - \rho r }
  \biggr]
=
  \biggl[
    \frac{
      4 \kappa_0^3
      \exp( 1 / 2 )
      L_f C_X^6 d
    }{
      c^3 \eta^2
    }
  \biggr]
  \biggl[
    \sum_{ r = 0 }^{ \infty }
    (r + 1)^3 e^{ - \rho r }
  \biggr]
\\ &
=
  \biggl[
    \frac{
      4^4
      \exp( 1 / 2 )
      L_f C_X^6 d
    }{
      c^3 \eta^2 ( 1 - \exp( - 1 / 2 ) )^3
    }
  \biggr]
  \biggl[
    \sum_{ r = 0 }^{ \infty }
    \frac{ (r + 1)^3 }{ \exp( \rho r ) }
  \biggr]
=
  \biggl[
    \frac{
      2^8
      \exp( \rho + 1 / 2 )
      L_f C_X^6 d
    }{
      c^3 \eta^2 ( 1 - \exp( - 1 / 2 ) )^3
    }
  \biggr]
  \biggl[
    \sum_{ r = 1 }^{ \infty }
    \frac{ r^3 }{ \exp( \rho r ) }
  \biggr]
  .
}
Using
$
  \exp( 1 / 2 ) [ 1 - \exp( - 1 / 2 ) ]^{ - 3 }
  \leq
  2^{ 15 }
$
we hence obtain that
\bas{
\label{eq:C_1_rep}
  C_1
  =
  \biggl[
    \frac{
      2^8
      \exp( \rho + 1 / 2 )
      L_f C_X^6 d
    }{
      c^3 \eta^2 ( 1 - \exp( - 1 / 2 ) )^3
    }
  \biggr]
  \biggl[
    \sum_{ r = 0 }^{ \infty }
    \frac{ r^3 }{ \exp( \rho r ) }
  \biggr]
\leq
    \sum_{ r = 0 }^{ \infty }
    \frac{
      2^{ 13 }
      \eta^{ - 2 }
      c^{ - 3 }
      L_f C_X^6 d
      r^3
    }{
      \exp( \rho r - \rho )
    }
=
    \sum_{ r = 0 }^{ \infty }
    \frac{
      2^{ 13 }
      d
      L_f C_X^6
      \eta^{ - 2 }
      c^{ - 3 }
      (r+1)^3
    }{
      \exp( \rho r )
    }
  .
}
In \cref{eq:C_0_rep,eq:C_1_rep} we provide exact and explicit representations
of the error constants $ C_0 $ and $ C_1 $ in \cref{eq:combined-early-late-numerical-time} in \cref{thm:combined-early-late-numerical-time}
in terms of
the parameters
$ d, \rho, c, q, C_X, L_f, C_{\mathrm{dec}}, \eta $
in \cref{thm:combined-early-late-numerical-time}
(see \cref{eq:combined-main-eta})
and in \cref{eq:C_0_bound_from_above_B,eq:C_1_rep} we provide short upper bounds
for the error constants $ C_0 $ and $ C_1 $
in \cref{eq:combined-early-late-numerical-time} in \cref{thm:combined-early-late-numerical-time}
in terms of the parameters
$ d, \rho, c, q, C_X, L_f, C_{\mathrm{dec}}, \eta $
in \cref{thm:combined-early-late-numerical-time}
(see \cref{eq:combined-main-eta}).
\end{rem}

\section{Criteria and examples for regular minibatch innovations}
\label{sec:regularity-examples}

We conclude with sufficient conditions under which the regularity assumption in
\cref{def:regular} can be verified for minibatch stochastic
gradients. We first reduce regularity to a uniform small-ball estimate. We
then establish such an estimate for strongly convex coordinate gradients and
apply the resulting criterion to a risk-sensitive objective. Finally, we
treat nonlinear regression by a separate argument based on conditional
density bounds and negative moments.

\subsection{A small-ball criterion for regularity}
\label{subsec:regularity-small-ball}

Let $M\in\N$, let $U_1,\dots,U_M$ be independent copies of an
$\R^m$-valued random variable $U$, and set
$\mathbf U=(U_1,\dots,U_M)$. For a loss
$ \ell \colon \R^d \times \R^m \to \R $ that is differentiable in its first argument, the
natural minibatch innovation is defined, for every $\theta\in\R^d$ and
$i\in\{1,\dots,d\}$, by
\bas{\label{eq:minibatch-innovation}
X_M^{(i)}(\theta,\mathbf U)
=-\frac1M\sum_{r=1}^M
\frac{\partial\ell}{\partial\theta_i}(\theta,U_r).
}

The following elementary small-ball criterion is convenient.

\begin{lemma}\label{le:regularity-small-ball}
Let $ a \in (0,\infty) $, $ p \in (0,1] $, let $ \cV \subseteq \R^d $,
and let $ ( X, U ) $ be an innovation which satisfies for every $ \theta \in \cV $, $ i \in \{ 1, \dots, d \} $ that
\bas{\label{eq:regularity-small-ball-assumption}
\P\bigl(|X^{(i)}(\theta,U)|\ge a\bigr)\ge p.
}
Then $(X,U)$ is $(c,q)$-regular on $\cV$ with
\bas{\label{eq:regularity-small-ball-constants}
c=p(1-e^{-1})a^2
\qquad
  \text{and}
\qquad
q=p(1-e^{-1}).
}
\end{lemma}

\begin{proof}
Fix $\theta\in\cV$, $i\in\{1,\dots,d\}$, and $\lambda\in[0,\infty)$.
Assumption~\eqref{eq:regularity-small-ball-assumption} gives
\bas{\label{eq:regularity-small-ball-laplace-first}
\E\bigl[e^{-\lambda X^{(i)}(\theta,U)^2}\bigr]
&\le 1-p+pe^{-\lambda a^2}
=1-p(1-e^{-\lambda a^2}).
}
Since $-\ln(1-x)\ge x$ for $x\in[0,1)$ and
$1-e^{-x}\ge(1-e^{-1})(x\wedge1)$ for $x\in[0,\infty)$,
\eqref{eq:regularity-small-ball-laplace-first} yields
\bas{
-\ln\E\bigl[e^{-\lambda X^{(i)}(\theta,U)^2}\bigr]
&\ge p(1-e^{-\lambda a^2})\\
&\ge p(1-e^{-1})\bigl((\lambda a^2)\wedge1\bigr)
=(c\lambda)\wedge q,
}
where $c$ and $q$ are given by
\eqref{eq:regularity-small-ball-constants}.
\end{proof}

\subsection{Strong convexity in the data variable}
\label{subsec:strongly-convex-data}

Our first application uses strong convexity in the data variable. The next
estimate combines this property with an upper density bound and will be used
to prove \cref{prop:strongly-convex-coordinate-gradients}.
The bounded-support assumption provides a uniform bound on the volume of
boundary layers of convex sublevel sets. This is needed because an upper
density bound alone controls small balls but does not uniformly control
boundary layers at arbitrarily high levels.

\begin{lemma}[Anti-concentration of strongly convex functions]
\label{le:strongly-convex-anti-concentration}
Let $\mu,B,R\in(0,\infty)$, and suppose that $U$ has a Lebesgue density $g$
on $\R^m$ such that
\bas{\label{eq:strongly-convex-data-density-support}
\|g\|_{L^\infty}\le B
\qquad
  \text{and}
\qquad
\P\bigl( | U | < R \bigr)=1.
}
Let $\varphi:\R^m\to\R$ be $\mu$-strongly convex. Let $\omega_m$ denote
the Lebesgue volume of the
Euclidean unit ball in $\R^m$. For every $\varepsilon\in(0,\infty)$, set
\bas{
\label{eq:strongly-convex-window-radius}
  \delta_\varepsilon
  =
  2 \sqrt{ \tfrac{ \varepsilon }{ \mu } } .
}
Then, for every $ a \in \R $, one has
\bas{\label{eq:strongly-convex-anti-concentration}
  \P\bigl(
    | \varphi(U) - a |
    \le \varepsilon
  \bigr)
\le
  B \omega_m
  \bigl(
    \delta_\varepsilon^m
    + ( R + 2 \delta_\varepsilon )^m
    - ( R + \delta_\varepsilon )^m
  \bigr)
  .
}
\end{lemma}

\begin{proof}
Since a finite-valued strongly convex function on $\R^m$ is continuous and
coercive, $\varphi$ has a unique minimizer $u_*$. Set
$\varphi_*=\varphi(u_*)$. First suppose that
$a-\varepsilon<\varphi_*$. If the event in
\eqref{eq:strongly-convex-anti-concentration} is nonempty, then
$a+\varepsilon<\varphi_*+2\varepsilon$. Strong convexity gives that for all $u\in\R^m$
\bas{\label{eq:strongly-convex-minimum-growth}
\varphi(u)
\ge \varphi_*+\frac{\mu}{2}|u-u_*|^2.
}
It follows from \eqref{eq:strongly-convex-window-radius} and
\eqref{eq:strongly-convex-minimum-growth} that
\bas{\label{eq:strongly-convex-near-minimum-inclusion}
\bigl\{u\in\R^m:|\varphi(u)-a|\le\varepsilon\bigr\}=\bigl\{u\in\R^m:\varphi(u)\le \varphi_*+2\varepsilon\bigr\}
\subseteq B(u_*,\delta_\varepsilon).
}
Therefore, \eqref{eq:strongly-convex-data-density-support} and
\eqref{eq:strongly-convex-near-minimum-inclusion} give
\bas{
\P\bigl(|\varphi(U)-a|\le\varepsilon\bigr)
&\le \P\bigl(U\in B(u_*,\delta_\varepsilon)\bigr)
\le B\omega_m\delta_\varepsilon^m.
}

Now suppose that $t=a-\varepsilon\ge\varphi_*$, and set
$K_t=\{u\in\R^m:\varphi(u)\le t\}$. The set $K_t$ is nonempty, closed, and
convex. Fix $u\in\R^m\setminus K_t$ and let $v$ be its Euclidean projection
onto $K_t$. Then $\varphi(v)=t$. For every $\alpha\in(0,1)$, the point
$(1-\alpha)v+\alpha u$ lies outside $K_t$, and strong convexity yields
\bas{\label{eq:strongly-convex-projection-interpolation}
t
<\varphi\bigl((1-\alpha)v+\alpha u\bigr)
&\le (1-\alpha)\varphi(v)+\alpha\varphi(u)
-\frac{\mu}{2}\alpha(1-\alpha)|u-v|^2.
}
Using $\varphi(v)=t$ in \eqref{eq:strongly-convex-projection-interpolation},
subtracting $t$, and dividing by $\alpha>0$ show that, for every
$\alpha\in(0,1)$,
\bas{\label{eq:strongly-convex-projection-growth-alpha}
\varphi(u)-t
>\frac{\mu}{2}(1-\alpha)|u-v|^2.
}
Letting $\alpha$ decrease to zero in
\eqref{eq:strongly-convex-projection-growth-alpha} and recalling that
$|u-v|=\dist(u,K_t)$ prove that, for every $u\in\R^m\setminus K_t$,
\bas{\label{eq:strongly-convex-projection-growth}
\varphi(u)-t
\ge\frac{\mu}{2}\dist(u,K_t)^2.
}
Consequently, \eqref{eq:strongly-convex-window-radius} and
\eqref{eq:strongly-convex-projection-growth} imply, up to a Lebesgue-null
boundary,
\bas{\label{eq:strongly-convex-boundary-layer-inclusion}
\bigl\{u\in\R^m:|\varphi(u)-a|\le\varepsilon\bigr\}
\subseteq
(K_t+\delta_\varepsilon B(0,1))\setminus K_t.
}
Set
$L_t=K_t\cap B(0,R+\delta_\varepsilon)$. By
\eqref{eq:strongly-convex-data-density-support}, only this bounded part of
$K_t$ can contribute to the event in
\eqref{eq:strongly-convex-boundary-layer-inclusion}; more precisely,
\bas{\label{eq:strongly-convex-localized-boundary-layer}
\bigl((K_t+\delta_\varepsilon B(0,1))\setminus K_t\bigr)
\cap B(0,R)
\subseteq
(L_t+\delta_\varepsilon B(0,1))\setminus L_t.
}
Since $L_t\subseteq B(0,R+\delta_\varepsilon)$, the Steiner formula and
monotonicity of the intrinsic volumes, together with
\eqref{eq:strongly-convex-localized-boundary-layer}, yield
\bas{\label{eq:strongly-convex-boundary-layer-volume}
\left|
\bigl((K_t+\delta_\varepsilon B(0,1))\setminus K_t\bigr)
\cap B(0,R)
\right|
\le
\omega_m\left(
(R+2\delta_\varepsilon)^m
-(R+\delta_\varepsilon)^m
\right).
}
Equations \eqref{eq:strongly-convex-data-density-support},
\eqref{eq:strongly-convex-boundary-layer-inclusion}, and
\eqref{eq:strongly-convex-boundary-layer-volume} prove
\eqref{eq:strongly-convex-anti-concentration}.
\end{proof}

\begin{prop}[Strongly convex coordinate gradients]
\label{prop:strongly-convex-coordinate-gradients}
Let $ m, d, M \in \N $, $ \mu, B, R \in (0,\infty) $, and let
$\cV\subseteq\R^d$. Let $U$ be an $\R^m$-valued random variable with a
Lebesgue density $g$ satisfying
\bas{\label{eq:strongly-convex-proposition-density-support}
  \|g\|_{L^\infty}\le B,
  \qquad
  \text{and}
  \qquad
  \P\bigl( | U | < R \bigr) = 1 .
}
Let $U_1,\dots,U_M$ be independent copies of $U$, and set
$\mathbf U=(U_1,\dots,U_M)$. Let
$\ell:\R^d\times\R^m\to\R$ be differentiable in its first argument. For
every $ \theta \in \cV $, $ i \in \{ 1, \dots, d \} $, define
\bas{\label{eq:strongly-convex-coordinate-gradient}
  \varphi_{\theta,i} \colon \R^m \to \R,
\qquad
  \varphi_{\theta,i}( u )
  =
  \frac{\partial\ell}{\partial\theta_i}(\theta,u),
}
and assume that $\varphi_{\theta,i}$ is $\mu$-strongly convex.
For every $ \theta \in \cV $, $ i \in \{ 1, \dots, d \} $, define
\bas{\label{eq:strongly-convex-proposition-minibatch-innovation}
X_M^{(i)}(\theta,\mathbf U)
=-\frac1M\sum_{r=1}^M
\frac{\partial\ell}{\partial\theta_i}(\theta,U_r).
}
Let $\omega_m$ denote the Lebesgue volume of the Euclidean unit ball in
$\R^m$, and set
\bas{\label{eq:strongly-convex-boundary-constant}
C_m=1+m3^{m-1}
}
and
\bas{\label{eq:strongly-convex-window-threshold}
s_0
=\mu\min\left\{
\frac{R^2}{4},
\frac{1}{
16B^2\omega_m^2C_m^2R^{2m-2}}
\right\}.
}
Then the minibatch innovation $(X_M,\mathbf U)$ is $(c_M,q)$-regular on
$\cV$ with
\bas{\label{eq:strongly-convex-minibatch-regularity-constants}
c_M=\frac{(1-e^{-1})s_0^2}{2M^2}
\qquad
  \text{and}
\qquad
q=\frac{1-e^{-1}}2.
}
If, in addition, the coordinate derivatives of $\ell$ are uniformly bounded
on $ \cV \times \{ u \in \R^m \colon | u | < R \} $,
then assumption~\eqref{eq:combined-main-boundedness} in
\cref{thm:combined-early-late-numerical-time} is satisfied for $X_M$.
\end{prop}

\begin{proof}
For every $s\in(0,s_0]$, set
$\delta_s=2\sqrt{\frac{s}{\mu}}$.
Equation~\eqref{eq:strongly-convex-window-threshold} then gives
$\delta_s\le R$. Hence
\bas{\label{eq:strongly-convex-boundary-simplification}
\delta_s^m
+(R+2\delta_s)^m-(R+\delta_s)^m
&\le R^{m-1}\delta_s
+m(R+2\delta_s)^{m-1}\delta_s\\
&\le C_mR^{m-1}\delta_s.
}
By \eqref{eq:strongly-convex-proposition-density-support},
\cref{le:strongly-convex-anti-concentration} applies to every
$\varphi_{\theta,i}$. Equations
\eqref{eq:strongly-convex-anti-concentration},
\eqref{eq:strongly-convex-window-threshold}, and
\eqref{eq:strongly-convex-boundary-simplification} give
\bas{\label{eq:strongly-convex-coordinate-anti-concentration}
\sup_{\substack{\theta\in\cV,\ 1\le i\le d\\a\in\R}}
\P\bigl(
|\varphi_{\theta,i}(U)-a|\le s_0
\bigr)
&\le B\omega_m C_mR^{m-1}\delta_{s_0}=2B\omega_m C_mR^{m-1}\sqrt{\frac{s_0}{\mu}}\le\frac12.
}

Fix $\theta\in\cV$ and $i\in\{1,\dots,d\}$. By
\eqref{eq:strongly-convex-proposition-minibatch-innovation} and
\eqref{eq:strongly-convex-coordinate-gradient}, conditional on
$U_2,\dots,U_M$, the event
\bas{\label{eq:strongly-convex-minibatch-small-event}
\left\{
\left|X_M^{(i)}(\theta,\mathbf U)\right|
\le\frac{s_0}{M}
\right\}
}
is an interval event of width $2s_0$ for
$\varphi_{\theta,i}(U_1)$, whose center may depend on
$U_2,\dots,U_M$. Therefore,
\eqref{eq:strongly-convex-coordinate-anti-concentration} and
\eqref{eq:strongly-convex-minibatch-small-event} imply
\bas{\label{eq:strongly-convex-minibatch-small-ball}
\P\left(
\left|X_M^{(i)}(\theta,\mathbf U)\right|
\ge\frac{s_0}{M}
\right)
\ge\frac12.
}
\cref{le:regularity-small-ball}, applied with
$a=s_0/M$ and $p=1/2$, proves the asserted regularity constants. The
boundedness assertion follows directly from
\eqref{eq:strongly-convex-proposition-minibatch-innovation}.
\end{proof}

The preceding criterion applies naturally to risk-sensitive objectives in
which the parameter changes the exponential weighting of a random energy.
Unlike feature-mean estimation, the resulting optimizer depends on the
Laplace transform of the energy and hence on more than one fixed moment.

\begin{exa}[Risk-sensitive exponential tilting]
\label{exa:risk-sensitive-exponential-tilting}
Let $\bar\theta,\mu,B,R\in(0,\infty)$, assume
\eqref{eq:strongly-convex-data-density-support}, and let
$\cV\subseteq[0,\bar\theta]^d$ be nonempty and compact. Let
$ G \colon \R^d \to \R $ be continuously differentiable. For every
$i\in\{1,\dots,d\}$, let $\varphi_i:\R^m\to[0,\infty)$ be twice
continuously differentiable and $\mu$-strongly convex. Consider the
risk-sensitive loss
\bas{\label{eq:risk-sensitive-loss}
\ell_{\mathrm{risk}}(\theta,u)
=G(\theta)+\sum_{i=1}^d e^{\theta_i\varphi_i(u)}.
}
For every $\theta\in\cV$, $i\in\{1,\dots,d\}$, $u\in\R^m$, one has
\bas{\label{eq:risk-sensitive-coordinate-gradient}
\frac{\partial\ell_{\mathrm{risk}}}{\partial\theta_i}(\theta,u)
=\frac{\partial G}{\partial\theta_i}(\theta)
+\varphi_i(u)e^{\theta_i\varphi_i(u)}.
}
For symmetric matrices $A,B\in\R^{m\times m}$, write $A\succeq B$ if
$A-B$ is positive semidefinite.
Moreover, for every $\theta\in\cV$, $i\in\{1,\dots,d\}$, $ u \in \R^m $, one has
\bas{\label{eq:risk-sensitive-data-hessian}
\Hess_u\!\left(\varphi_i(u)e^{\theta_i\varphi_i(u)}\right)
&=e^{\theta_i\varphi_i(u)}
\bigl(1+\theta_i\varphi_i(u)\bigr)\Hess\varphi_i(u)\\
&\quad+
\theta_i e^{\theta_i\varphi_i(u)}
\bigl(2+\theta_i\varphi_i(u)\bigr)
\grad\varphi_i(u)\grad\varphi_i(u)^\mathsf{T}\\
&\succeq \mu I_m.
}
Indeed, $\theta_i,\varphi_i(u)\ge0$ and
$\Hess\varphi_i(u)\succeq\mu I_m$. Equations
\eqref{eq:risk-sensitive-coordinate-gradient} and
\eqref{eq:risk-sensitive-data-hessian} show that every coordinate gradient
of \eqref{eq:risk-sensitive-loss} is $\mu$-strongly convex as a function of
$u$. Therefore, \cref{prop:strongly-convex-coordinate-gradients}
applies. In particular, with $s_0$ defined in
\eqref{eq:strongly-convex-window-threshold}, the minibatch innovation driven
by \eqref{eq:risk-sensitive-loss} is $(c_M,q)$-regular on $\cV$, where
$c_M$ and $q$ are given in
\eqref{eq:strongly-convex-minibatch-regularity-constants}.

The continuity assumptions and the compactness of
$\cV\times B(0,R)$ also imply that the coordinate gradients in
\eqref{eq:risk-sensitive-coordinate-gradient} are uniformly bounded on
this set. Hence assumption~\eqref{eq:combined-main-boundedness} in
\cref{thm:combined-early-late-numerical-time} is satisfied by the resulting
minibatch innovation.

For every $\theta\in\cV$, the associated population objective is given by
\bas{\label{eq:risk-sensitive-population-objective}
F_{\mathrm{risk}}(\theta)
=G(\theta)+\sum_{i=1}^d
\E\!\left[e^{\theta_i\varphi_i(U)}\right].
}
The gradient of $F_{\mathrm{risk}}$ is the negative of the mean innovation.
For example, if
$G(\theta)=-\sum_{i=1}^d b_i\theta_i$ for constants
$b_1,\dots,b_d\in\R$, every interior critical point satisfies, for every
$i\in\{1,\dots,d\}$,
\bas{\label{eq:risk-sensitive-tilting-equation}
\E\!\left[
\varphi_i(U)e^{\theta_i\varphi_i(U)}
\right]=b_i.
}
Thus the optimized parameters are exponential-tilting parameters determined
by the Laplace transforms of the random energies $\varphi_i(U)$.
\end{exa}

\subsection{Nonlinear regression}
\label{subsec:nonlinear-regression}

We finally consider nonlinear regression with vector-valued responses. Let
$k\in\N$, let $\mathcal Z\subseteq\R^k$ be a Borel set, let $U=(Z,Y)$ take
values in $\mathcal Z\times\R^m$, and let
$(\Gamma^\theta \colon \theta\in\R^d)$ be a family of measurable maps from
$\mathcal Z$ to $\R^m$ that is differentiable with respect to $\theta$.
For every $\theta\in\R^d$, $z\in\mathcal Z$, $y\in\R^m$, and
$i\in\{1,\dots,d\}$, the quadratic loss and its associated coordinate
innovation are given by
\bas{\label{eq:nonlinear-regression-loss-innovation}
\ell_{\mathrm{reg}}(\theta,(z,y))
&=\frac12\left|\Gamma^\theta(z)-y\right|^2,\\
X^{(i)}(\theta,(z,y))
&=-\left\langle
\frac{\partial\Gamma^\theta}{\partial\theta_i}(z),
\Gamma^\theta(z)-y
\right\rangle.
}
The following result first bounds negative moments of the scalar residual in
the direction selected by the derivative of the regression function. A
reverse H\"older inequality turns this into a uniform positive lower bound
for the first absolute moment. Compactness provides a uniform
upper bound for the innovation, and the two moment bounds then yield
regularity.

\begin{prop}[Compactly supported nonlinear regression]
\label{prop:nonlinear-regression-regularity}
Let $\cV\subseteq\R^d$ be nonempty and compact,
let $ r \in ( 0, \nicefrac{ 1 }{ 2 } ) $, $ \delta \in (0, \infty ) $,
and suppose that there exist nonempty compact sets
$ K_Z \subseteq \mathcal Z $ and $ K_Y\subseteq \R^m $ and a constant
$ B \in (0,\infty) $ with the following properties:
\begin{enumerate}[label=(\roman*)]

\item
The random variables are compactly supported in the sense that
\bas{\label{eq:nonlinear-regression-compact-support}
  \P\bigl( Z \in K_Z , \, Y \in K_Y \bigr) = 1 .
}

\item
For $ \P_Z $-almost every $ z \in K_Z $, the conditional law of $ Y $
given $ Z = z $ has a Lebesgue density~$ p_z $ satisfying
\bas{
\label{eq:nonlinear-regression-conditional-density}
  \| p_z \|_{ L^\infty } \le B .
}

\item
The map $ ( \theta, z ) \mapsto \Gamma^{ \theta }( z ) $ and, for every
$ i \in \{ 1, \dots, d \} $,
the map
$
  ( \theta, z ) \mapsto\frac{\partial\Gamma^\theta}{\partial\theta_i}( z )
$
are continuous on $ \cV \times K_Z $.

\item
The coordinate derivatives satisfy the uniform moment lower bound
\bas{
\label{eq:nonlinear-regression-derivative-moment}
D_r
=
\inf_{ \substack{ \theta \in \cV } }
\inf_{
  1 \le i \le d
}
\E\biggl[
\Bigl|
\frac{\partial\Gamma^\theta}{\partial\theta_i}(Z)
\Bigr|^{r}
\biggr]
\ge \delta.
}
\end{enumerate}
Choose $R\in(0,\infty)$ such that $K_Y\subseteq B(0,R)$, let
$\omega_{m-1}$ denote the volume of the unit ball in $\R^{m-1}$, use the
convention $\omega_0=1$, and set
\bas{\label{eq:nonlinear-regression-regularity-data-constants}
\mu_r
&=\frac{\delta^{1/r}}{2B\omega_{m-1}R^{m-1}}
\left(\frac{1-2r}{1-r}\right)^{(1-r)/r},\\
C_X
&=\max_{\substack{\theta\in\cV,\ z\in K_Z,\ y\in K_Y\\1\le i\le d}}
\left|
\left\langle
\frac{\partial\Gamma^\theta}{\partial\theta_i}(z),
\Gamma^\theta(z)-y
\right\rangle
\right|.
}
Then $\mu_r,C_X\in(0,\infty)$. The innovation $(X,U)$ defined in
\eqref{eq:nonlinear-regression-loss-innovation} is $(c_r,q_r)$-regular on
$\cV$ with
\bas{\label{eq:nonlinear-regression-regularity-constants}
c_r
=\frac{(1-e^{-1})\mu_r^3}{8C_X},
\qquad
q_r=\frac{(1-e^{-1})\mu_r}{2C_X}.
}
\end{prop}

\begin{proof}
Since $r\in(0,1/2)$, one has $r/(1-r)\in(0,1)$. We first prove the uniform
first-moment lower bound that drives the argument. Fix $\theta\in\cV$ and
$i\in\{1,\dots,d\}$. In every normalized derivative below, the quotient is
understood as an arbitrary fixed unit vector in $\R^m$ whenever
$\partial\Gamma^\theta/\partial\theta_i$ vanishes. With this convention the
normalized derivative is measurable, has unit norm, and its product with
$|\partial\Gamma^\theta/\partial\theta_i|$ remains the original derivative.

By
\eqref{eq:nonlinear-regression-compact-support}, the conditional density
$p_z$ vanishes almost everywhere outside $K_Y$ for $\P_Z$-almost every
$z\in K_Z$. Fix such a $z$. After an orthogonal change of coordinates,
Fubini's theorem shows that the conditional density of the residual
projected onto the normalized direction specified above is bounded by
$B\omega_{m-1}R^{m-1}$: the density
$p_z$ is bounded by $B$ according to
\eqref{eq:nonlinear-regression-conditional-density}, and every hyperplane
slice of $K_Y\subseteq B(0,R)$ has $(m-1)$-dimensional Hausdorff measure
at most $\omega_{m-1}R^{m-1}$. Consequently, for every
$h\in(0,\infty)$,
\bas{\label{eq:nonlinear-regression-projected-small-ball}
\P\!\left(
\left|
\left\langle
\frac{\partial\Gamma^\theta/\partial\theta_i(z)}
{|\partial\Gamma^\theta/\partial\theta_i(z)|},
\Gamma^\theta(z)-Y
\right\rangle
\right|\le h
\,\middle|\,Z=z
\right)
\le 2B\omega_{m-1}R^{m-1}h.
}
Letting $h$ decrease to zero in
\eqref{eq:nonlinear-regression-projected-small-ball} shows that the
projected residual is conditionally nonzero almost surely. For the fixed
$\theta$, $i$, and $z$, denote its absolute value by $W$ and set
$C_{\mathrm{pr}}=2B\omega_{m-1}R^{m-1}$. We get  with \eqref{eq:nonlinear-regression-projected-small-ball} that
\bas{\label{eq:nonlinear-regression-projected-negative-moment}
\E\!&\left[
\left|
\left\langle
\frac{\partial\Gamma^\theta/\partial\theta_i(z)}
{|\partial\Gamma^\theta/\partial\theta_i(z)|},
\Gamma^\theta(z)-Y
\right\rangle
\right|^{-r/(1-r)}
\,\middle|\,Z=z
\right]
=\int_0^\infty
\P\bigl(W^{-r/(1-r)}\ge t\,\big|\,Z=z\bigr)\,\dd t\\
&=\int_0^\infty
\P\bigl(W\le t^{-(1-r)/r}\,\big|\,Z=z\bigr)\,\dd t
\le \int_0^{C_{\mathrm{pr}}^{r/(1-r)}}1\,\dd t
+C_{\mathrm{pr}}\int_{C_{\mathrm{pr}}^{r/(1-r)}}^\infty
t^{-(1-r)/r}\,\dd t\\
&=C_{\mathrm{pr}}^{r/(1-r)}
+\frac{r}{1-2r}C_{\mathrm{pr}}^{r/(1-r)}=\frac{1-r}{1-2r}
\left(2B\omega_{m-1}R^{m-1}\right)^{r/(1-r)}.
}
Taking expectations in
\eqref{eq:nonlinear-regression-projected-negative-moment} and using
\eqref{eq:nonlinear-regression-compact-support} give
\bas{\label{eq:nonlinear-regression-projected-negative-moment-unconditional}
\E\!\left[
\left|
\left\langle
\frac{\partial\Gamma^\theta/\partial\theta_i(Z)}
{|\partial\Gamma^\theta/\partial\theta_i(Z)|},
\Gamma^\theta(Z)-Y
\right\rangle
\right|^{-r/(1-r)}
\right]
\le
\frac{1-r}{1-2r}
\left(2B\omega_{m-1}R^{m-1}\right)^{r/(1-r)}.
}

The reverse H\"older inequality used below follows from the ordinary H\"older
inequality with conjugate exponents $1/r$ and $1/(1-r)$. Equations
\eqref{eq:nonlinear-regression-loss-innovation},
\eqref{eq:nonlinear-regression-derivative-moment}, and
\eqref{eq:nonlinear-regression-projected-negative-moment-unconditional}
therefore yield
\bas{\label{eq:nonlinear-regression-first-moment-lower-bound}
\E\!\left[|X^{(i)}(\theta,U)|\right]
&=\E\!\left[
\left|
\frac{\partial\Gamma^\theta}{\partial\theta_i}(Z)
\right|
\left|
\left\langle
\frac{\partial\Gamma^\theta/\partial\theta_i(Z)}
{|\partial\Gamma^\theta/\partial\theta_i(Z)|},
\Gamma^\theta(Z)-Y
\right\rangle
\right|
\right]\\
&\ge
\E\!\left[
\left|
\frac{\partial\Gamma^\theta}{\partial\theta_i}(Z)
\right|^r
\right]^{1/r}
\E\!\left[
\left|
\left\langle
\frac{\partial\Gamma^\theta/\partial\theta_i(Z)}
{|\partial\Gamma^\theta/\partial\theta_i(Z)|},
\Gamma^\theta(Z)-Y
\right\rangle
\right|^{-r/(1-r)}
\right]^{-(1-r)/r}\\
&\ge
\delta^{1/r}
\left[
\frac{1-r}{1-2r}
\left(2B\omega_{m-1}R^{m-1}\right)^{r/(1-r)}
\right]^{-(1-r)/r}
=\mu_r,
}
where the last identity follows from
\eqref{eq:nonlinear-regression-regularity-data-constants}. Since $\theta$
and $i$ were arbitrary, \eqref{eq:nonlinear-regression-first-moment-lower-bound}
holds uniformly in $\theta\in\cV$ and $i\in\{1,\dots,d\}$.

The continuity assumptions and the compactness of
$\cV\times K_Z\times K_Y$ imply that the maximum defining $C_X$ in
\eqref{eq:nonlinear-regression-regularity-data-constants} exists and is
finite. Moreover, by \eqref{eq:nonlinear-regression-compact-support} and
\eqref{eq:nonlinear-regression-loss-innovation}, for every $\theta\in\cV$
and $i\in\{1,\dots,d\}$ one has, almost surely,
\bas{\label{eq:nonlinear-regression-uniform-upper-bound}
\left|X^{(i)}(\theta,U)\right|\le C_X.
}

Equations \eqref{eq:nonlinear-regression-uniform-upper-bound} and
\eqref{eq:nonlinear-regression-first-moment-lower-bound} first show that
$C_X\ge\mu_r>0$. They also imply that, for every $\theta\in\cV$ and
$i\in\{1,\dots,d\}$,
\bas{\label{eq:nonlinear-regression-small-ball}
\mu_r
&\le\E\!\left[|X^{(i)}(\theta,U)|\right]\le\frac{\mu_r}{2}
+C_X\P\!\left(
|X^{(i)}(\theta,U)|\ge\frac{\mu_r}{2}
\right).
}
It follows from \eqref{eq:nonlinear-regression-small-ball} that, for every
$\theta\in\cV$ and $i\in\{1,\dots,d\}$,
\bas{\label{eq:nonlinear-regression-small-ball-probability}
\P\!\left(
|X^{(i)}(\theta,U)|\ge\frac{\mu_r}{2}
\right)
\ge\frac{\mu_r}{2C_X}.
}
Since $\mu_r/2\in(0,\infty)$ and $\mu_r/(2C_X)\in(0,1]$, equation
\eqref{eq:nonlinear-regression-small-ball-probability} and
\cref{le:regularity-small-ball} apply with threshold $\mu_r/2$ and
probability $\mu_r/(2C_X)$. The constants in
\eqref{eq:regularity-small-ball-constants} are then
\bas{
\frac{\mu_r}{2C_X}(1-e^{-1})\left(\frac{\mu_r}{2}\right)^2
&=\frac{(1-e^{-1})\mu_r^3}{8C_X}=c_r \text{ \ \ \ and \ \ \ }\frac{\mu_r}{2C_X}(1-e^{-1})
=\frac{(1-e^{-1})\mu_r}{2C_X}=q_r,
}
where the last identities follow from
\eqref{eq:nonlinear-regression-regularity-constants}. This proves the claim.
\end{proof}

\begin{proof}[Proof of \cref{thm:linear-least-squares}]
\smallskip
\noindent\emph{Minibatch innovation and regularity.}
For every $n\in\N$, set
$\mathbf Y_n=(Y_{n,1},\dots,Y_{n,M})$ and, for every
$i\in\{1,\dots,d\}$, set
\bas{\label{eq:linear-least-squares-minibatch-gradient}
G_n^{(i)}
&=\frac1M\sum_{r=1}^M
(\nabla_{\theta_i}L)(\theta_{n-1},Y_{n,r}).
}
For
$\mathbf y=(y_1,\dots,y_M)\in(\R^m)^M$, set
$\overline y_M=M^{-1}\sum_{r=1}^M y_r$ and define the innovation
\bas{\label{eq:linear-least-squares-minibatch-innovation}
X_M(\theta,\mathbf y)
&=-\frac1M\sum_{r=1}^M
\grad_\theta L(\theta,y_r)
=2A^\mathsf{T}(\overline y_M-A\theta).
}
If $v_n=w_n/(1-\beta^n)$ for $n\in\N$ and $v_0=0$, then
\eqref{eq:linear-least-squares-rmsprop} is precisely the \RMSprop\ algorithm
driven by $(X_M,\mathbf Y_1)$ with parameter tuple
$(\beta,\eps,(\gamma_n)_{n\in\N},\theta_0)$.

Let $p$ denote a bounded Lebesgue density of $Y_{1,1}$.
The average
$\overline Y_{1,M}=M^{-1}\sum_{r=1}^M Y_{1,r}$ is supported by
$\operatorname{conv}(K)$. Its density is
\bas{
p_M(y)=M^m p^{*M}(My).
}
Young's convolution inequality gives
$\|p_M\|_{L^\infty}\le M^m\|p\|_{L^\infty}<\infty$. Fix
$r\in(0,1/2)$ and set
\bas{\label{eq:linear-least-squares-regularity-parameters}
V_\kappa&=\{\theta\in\R^d:|\theta|\le\kappa\},
\qquad K_Y=\operatorname{conv}(K),
\qquad B=M^m\|p\|_{L^\infty},\\
\delta&=\min_{1\le i\le d}|Ae_i|^r>0.
}
Apply \cref{prop:nonlinear-regression-regularity} with
$\mathcal Z=K_Z=\{0\}\subseteq\R$, $Z=0$,
$\Gamma^\theta(0)=A\theta$, response $\overline Y_{1,M}$, and the parameters
in \eqref{eq:linear-least-squares-regularity-parameters}. Its derivative
condition holds because
\bas{\label{eq:linear-least-squares-derivative-nondegeneracy}
\inf_{\substack{\theta\in V_\kappa\\1\le i\le d}}
\E\!\left[
\left|
\frac{\partial\Gamma^\theta}{\partial\theta_i}(Z)
\right|^r
\right]
=\delta.
}
\cref{prop:nonlinear-regression-regularity} and
\eqref{eq:linear-least-squares-derivative-nondegeneracy} establish regularity
and boundedness for the innovation in
\eqref{eq:linear-least-squares-minibatch-innovation} without the factor $2$.
Scaling by $2$ changes regularity parameters $(c,q)$ into $(4c,q)$ and doubles
the uniform bound. Hence there exist
$c_{\mathrm{lin}},q_{\mathrm{lin}},C_X\in(0,\infty)$ such that
$(X_M,\mathbf Y_1)$ is $(c_{\mathrm{lin}},q_{\mathrm{lin}})$-regular on
$V_\kappa$ and, for every $\theta\in V_\kappa$ and
$i\in\{1,\dots,d\}$,
\bas{\label{eq:linear-least-squares-uniform-bound}
|X_M^{(i)}(\theta,\mathbf Y_1)|\le C_X
\quad\text{almost surely}.
}
Since $A$ has no zero column, it is nonzero. Define
$\lambda_*$ as the smallest positive eigenvalue of
$A^\mathsf{T}A$.
Since an upper bound remains valid when its constant is increased, we may and
do assume that
\bas{\label{eq:linear-least-squares-enlarged-bound}
C_X\ge4\lambda_*\Gamma.
}
Choose
$\beta_0\in(\max\{1/2,e^{-q_{\mathrm{lin}}/3}\},1)$ and set
$\rho=q_{\mathrm{lin}}+3\ln\beta_0>0$.

\smallskip
\noindent\emph{Objective geometry and uniform constants.}
Let $\theta_*\in\R^d$ solve the normal equation
$A^\mathsf{T}A\theta_*=A^\mathsf{T}\E[Y_{1,1}]$. Such a solution exists
because
$\operatorname{range}(A^\mathsf{T})
=\operatorname{range}(A^\mathsf{T}A)$, and it minimizes the population risk.
Then
\bas{\label{eq:linear-least-squares-objective}
F(\theta)
&=|A(\theta-\theta_*)|^2.
}
Equations
\eqref{eq:linear-least-squares-minibatch-innovation} and
\eqref{eq:linear-least-squares-objective} show that $f_{X_M}=-\grad F$ and,
by the spectral theorem, that, for every $\theta\in\R^d$,
\bas{\label{eq:linear-least-squares-lojasiewicz}
F(\theta)
\le\frac1{4\lambda_*}|f_{X_M}(\theta)|^2.
}
Moreover, $\grad F$ is globally
$2\|A^\mathsf{T}A\|$-Lipschitz continuous. Set
\bas{
C_{\mathrm{Loj}}=\frac1{4\lambda_*},
\qquad
\eta=\frac1{2C_{\mathrm{Loj}}(C_X+1)}.
}
Thus $\eta$ depends only on the fixed problem data. Now fix
$\mathfrak C\in(0,\infty)$, set
$C_{\mathrm{dec}}=\mathfrak C$, and let
$C_0,C_1\in(0,\infty)$ be the constants furnished by
\cref{thm:combined-early-late-numerical-time} for
\bas{
c=c_{\mathrm{lin}},\qquad q=q_{\mathrm{lin}},\qquad
L_f=2\|A^\mathsf{T}A\|,
}
and the fixed values of
$\rho,C_X,C_{\mathrm{Loj}},C_{\mathrm{dec}}$.
In particular, $C_0$ and $C_1$ are independent of $\theta_0$, the particular
step-size sequence, $\beta$, and $\eps$.

Fix such a step-size sequence, $\beta\in[1/2,1)$, $\eps\in[0,1]$, and
$\theta_0\in\R^d$, and let $(\theta_n)_{n\in\N_0}$ be the process in the
statement. Set $\tau=\inf\{j\in\N_0:|\theta_j|>\kappa\}$ and, for every
$n\in\N_0$, set
\bas{\label{eq:linear-least-squares-stopping-events}
\mathcal I_n=\{\tau\ge n\}
=\bigcap_{j=0}^{n-1}\{|\theta_j|\le\kappa\}.
}
Here $\mathcal I_0=\Omega$ by the empty-intersection convention.

The case $n=1$ is immediate. Indeed, by
\eqref{eq:linear-least-squares-minibatch-gradient},
$w_1^{(i)}=(1-\beta)|G_1^{(i)}|^2$ and
\eqref{eq:zero-normalization-convention} give
$|\theta_1-\theta_0|\le\sqrt d\,\gamma_1$. Thus
\eqref{eq:linear-least-squares-objective} and the triangle inequality give
\bas{\label{eq:linear-least-squares-first-step}
F(\theta_1)^{1/2}
&\le F(\theta_0)^{1/2}
+\|A^\mathsf{T}A\|^{1/2}|\theta_1-\theta_0|
\le F(\theta_0)^{1/2}
+\sqrt{d\|A^\mathsf{T}A\|}\,\gamma_1.
}
Consequently, \eqref{eq:linear-least-squares-first-step} yields
\bas{\label{eq:linear-least-squares-initial-indices}
\E\bigl[F(\theta_1)\1_{\mathcal I_1}\bigr]
&\le
\left(
F(\theta_0)^{1/2}
+\sqrt{d\|A^\mathsf{T}A\|}\,\gamma_1
\right)^2.
}

\smallskip
\noindent\emph{The regime $\beta\in[\beta_0,1)$.}
For every $n\ge2$, the range of the step sizes and
\eqref{eq:linear-least-squares-enlarged-bound} give
\bas{\label{eq:linear-least-squares-main-step-smallness}
\gamma_n
\le\Gamma
\le C_{\mathrm{Loj}}C_X
\le C_{\mathrm{Loj}}(C_X+\eps).
}
Moreover, the assumed step-size comparison implies, for every
$j,n\in\N$ with $j\le n$, that
\bas{
\gamma_j
&\le C_{\mathrm{dec}}\gamma_n
\exp\left(\frac\eta2\sum_{k=j+1}^n\gamma_k\right).
}
If $\beta\in[\beta_0,1)$, the assumptions of
\cref{thm:combined-early-late-numerical-time} follow from
\eqref{eq:linear-least-squares-uniform-bound},
\eqref{eq:linear-least-squares-lojasiewicz},
the preceding step-size bounds, and
\eqref{eq:linear-least-squares-main-step-smallness}; moreover,
$q_{\mathrm{lin}}+3\ln\beta\ge\rho$. Thus
\cref{thm:combined-early-late-numerical-time} and
\eqref{eq:linear-least-squares-stopping-events} show that, for every $n\ge2$,
\bas{\label{eq:linear-least-squares-large-beta-bound}
\E\bigl[F(\theta_n)\1_{\mathcal I_n}\bigr]
&\le
\left(
F(\theta_0)^{1/2}
+\sqrt{d\|A^\mathsf{T}A\|}\,\gamma_1
\right)^2
\exp\left(-\eta\sum_{k=2}^n\gamma_k\right)
+C_0\gamma_n+C_1(1-\beta)^2.
}

\smallskip
\noindent\emph{The regime $\beta\in[1/2,\beta_0]$.}
The recursion for $w_n^{(i)}$ and
\eqref{eq:linear-least-squares-minibatch-gradient} imply
$w_n^{(i)}\ge(1-\beta)|G_n^{(i)}|^2$. Hence, for every $n\in\N$ and
$i\in\{1,\dots,d\}$,
\bas{\label{eq:linear-least-squares-self-normalization}
\left|
\bigl[\eps+(1-\beta^n)^{-1/2}(w_n^{(i)})^{1/2}\bigr]^{-1}
G_n^{(i)}
\right|
&\le\sqrt{\frac{1-\beta^n}{1-\beta}}
\le\frac1{\sqrt{1-\beta_0}}.
}
The first inequality in
\eqref{eq:linear-least-squares-self-normalization} also holds when its
denominator vanishes, by the convention in
\eqref{eq:zero-normalization-convention}.

Set
\bas{
R=\kappa+\frac{\sqrt d\,\mathfrak C}{\sqrt{1-\beta_0}},
\qquad
C_F=\sup_{|\theta|\le R}F(\theta)<\infty.
}
For every $n\ge2$, one has $|\theta_{n-1}|\le\kappa$ on $\mathcal I_n$.
Consequently, the bound $\gamma_n\le\mathfrak C$ and
\eqref{eq:linear-least-squares-self-normalization} imply
\bas{\label{eq:linear-least-squares-small-beta-bound}
|\theta_n|
&\le
|\theta_{n-1}|
+\frac{\sqrt d\,\gamma_n}{\sqrt{1-\beta_0}}
\le R
\qquad\text{on }\mathcal I_n.
}
Since $(1-\beta)^2\ge(1-\beta_0)^2$, it follows from
\eqref{eq:linear-least-squares-small-beta-bound} that, for every $n\ge2$,
\bas{\label{eq:linear-least-squares-small-beta-expectation}
\E\bigl[F(\theta_n)\1_{\mathcal I_n}\bigr]
\le C_F
\le\frac{C_F}{(1-\beta_0)^2}(1-\beta)^2.
}

\smallskip
\noindent\emph{Conclusion.}
The first term on the right-hand side of
\eqref{eq:linear-least-squares-rate} is nonnegative. Therefore,
\eqref{eq:linear-least-squares-initial-indices},
\eqref{eq:linear-least-squares-large-beta-bound}, and
\eqref{eq:linear-least-squares-small-beta-expectation} prove
\eqref{eq:linear-least-squares-rate} for every $n\in\N$ and
$\beta\in[1/2,1)$ with
\bas{
c=\max\left\{
C_0,
C_1,
\frac{C_F}{(1-\beta_0)^2}
\right\}.
}
This constant may depend on $\mathfrak C$, but is
independent of $\theta_0$, the particular step-size sequence, $\beta$, and
$\eps$, as required.
\end{proof}

\subsubsection*{Acknowledgements}

This work has been partially funded by the National Science Foundation of China (NSFC) under grant number W2531010.
This work has also been partially funded by the Deutsche Forschungsgemeinschaft (DFG, German Research Foundation)
under Germany's Excellence Strategy EXC 2044-390685587, Mathematics M\"{u}nster: Dynamics-Geometry-Structure.

\bibliographystyle{plain}
\bibliography{bibfile}

@article {MR4949951,
    AUTHOR = {Jentzen, Arnulf and Riekert, Adrian},
     TITLE = {Non-convergence to global minimizers for {A}dam and stochastic
              gradient descent optimization and constructions of local
              minimizers in the training of artificial neural networks},
   JOURNAL = {SIAM/ASA J. Uncertain. Quantif.},
  FJOURNAL = {SIAM/ASA Journal on Uncertainty Quantification},
    VOLUME = {13},
      YEAR = {2025},
    NUMBER = {3},
     PAGES = {1294--1333},
      ISSN = {2166-2525},
   MRCLASS = {68T07 (60G99 65K10 90C15)},
  MRNUMBER = {4949951},
       DOI = {10.1137/24M1639464},
       URL = {https://doi.org/10.1137/24M1639464},
}

@article {MR5104936,
    AUTHOR = {Do, Thang and Hannibal, Sonja and Jentzen, Arnulf},
     TITLE = {Non-convergence to global minimizers in data driven supervised
              deep learning: {A}dam and stochastic gradient descent
              optimization provably fail to converge to global minimizers in
              the training of deep neural networks with {R}e{LU} activation},
   JOURNAL = {J. Math. Anal. Appl.},
  FJOURNAL = {Journal of Mathematical Analysis and Applications},
    VOLUME = {564},
      YEAR = {2026},
    NUMBER = {2},
     PAGES = {Paper No. 130724},
      ISSN = {0022-247X,1096-0813},
   MRCLASS = {68T07 (90C52)},
  MRNUMBER = {5104936},
       DOI = {10.1016/j.jmaa.2026.130724},
       URL = {https://doi.org/10.1016/j.jmaa.2026.130724},
}

@article{Defossez2022,
      title = {{A Simple Convergence Proof of Adam and Adagrad}},
     author = {D{\'e}fossez, Alexandre and Bottou, Leon and Bach, Francis and Usunier, Nicolas},
    journal = {Transactions on Machine Learning Research},
       issn = {2835-8856},
       year = {2022},
        url = {http://leon.bottou.org/papers/defossez-2022}
}

@article{MaWuE2020_arXiv_qualitative_behavior_RMSprop,
         title = {{A Qualitative Study of the Dynamic Behavior for Adaptive Gradient Algorithms}},
        author = {Chao Ma and Lei Wu and Weinan E},
          year = {2020},
        eprint = {arXiv:2009.06125},
 archivePrefix = {arXiv},
  primaryClass = {math.OC},
       journal = {\href{https://arxiv.org/abs/2009.06125}{arXiv:2009.06125}}
}

@article{YuChenFeng2026_arXiv_Adam-SHANG,
         title = {{Adam-SHANG: A Convergent Adam-Type Method for Stochastic Smooth Convex Optimization}},
        author = {Yaxin Yu and Long Chen and Minfu Feng},
          year = {2026},
        eprint = {arXiv:2605.12878},
 archivePrefix = {arXiv},
  primaryClass = {math.OC},
       journal = {\href{https://arxiv.org/abs/2605.12878}{arXiv:2605.12878}}
}

@article{ZouShen2019,
         title = {{A Sufficient Condition for Convergences of Adam and RMSProp}},
        author = {Zou, Fangyu and Shen, Li and Jie, Zequn and Zhang, Weizhong and Liu, Wei},
          year = {2018},
        eprint = {arXiv:1811.09358},
 archivePrefix = {arXiv},
  primaryClass = {cs.LG},
       journal = {\href{https://arxiv.org/abs/1811.09358}{arXiv:1811.09358}}
}

@article{GodichonBaggioni2023,
  title           = {{Non asymptotic analysis of Adaptive stochastic gradient algorithms and applications}},
  author        = {Godichon-Baggioni, Antoine and Tarrago, Pierre},
  eprint        = {arXiv:2303.01370},
  archivePrefix = {arXiv},
  primaryClass  = {math.OC},
  journal       = {\href{https://arxiv.org/abs/2303.01370}{arXiv:2303.01370}},
  year          = {2023}
}

@inproceedings{ShiLiHongSun2021,
         title = {{RMSprop Converges with Proper Hyper-Parameter}},
        author = {Naichen Shi and Dawei Li and Mingyi Hong and Ruoyu Sun},
          year = {2021},
     booktitle = {International Conference on Learning Representations (ICLR)},
           url = {https://openreview.net/forum?id=3UDSdyIcBDA}
}

@article{LiuXuZhangMandic2024,
         title = {{On hyper-parameter selection for guaranteed convergence of RMSProp}},
        author = {Jinlan Liu and Dongpo Xu and Huisheng Zhang and Danilo Mandic},
          year = {2024},
       journal = {Cognitive Neurodynamics},
        volume = {18},
        number = {6},
         pages = {3227--3237},
           doi = {10.1007/s11571-022-09845-8},
           url = {https://doi.org/10.1007/s11571-022-09845-8}
}

@article{ZhangZhouZou2025,
         title = {{Convergence Guarantees for RMSProp and Adam in Generalized-smooth Non-convex Optimization with Affine Noise Variance}},
        author = {Qi Zhang and Yi Zhou and Shaofeng Zou},
          year = {2025},
       journal = {Transactions on Machine Learning Research},
           url = {https://openreview.net/forum?id=ZdMIXltJzK}
}

@article{XuZhangZhangMandic2021,
         title = {{Convergence of the RMSProp deep learning method with penalty for nonconvex optimization}},
        author = {Dongpo Xu and Shengdong Zhang and Huisheng Zhang and Danilo P. Mandic},
          year = {2021},
       journal = {Neural Networks},
        volume = {139},
         pages = {17--23},
           doi = {10.1016/j.neunet.2021.02.011},
           url = {https://doi.org/10.1016/j.neunet.2021.02.011}
}

@article{BensaidPoetteTurpault2024_arXiv,
         title = {{Convergence of the Iterates for Momentum and RMSProp for Local Smooth Functions: Adaptation is the Key}},
        author = {Bilel Bensaid and Ga{\"e}l Po{\"e}tte and Rodolphe Turpault},
          year = {2024},
        eprint = {arXiv:2407.15471},
 archivePrefix = {arXiv},
  primaryClass = {math.OC},
       journal = {\href{https://arxiv.org/abs/2407.15471}{arXiv:2407.15471}}
}

@article{DimitrieskiHoneckerSchererEbenbauer2026_arXiv,
         title = {{Global Stability and Step Size Robustness of RMSProp}},
        author = {Naum Dimitrieski and Maria Christine Honecker and Carsten Scherer and Christian Ebenbauer},
          year = {2026},
        eprint = {arXiv:2603.15823},
 archivePrefix = {arXiv},
  primaryClass = {math.OC},
       journal = {\href{https://arxiv.org/abs/2603.15823}{arXiv:2603.15823}}
}

@article{DeMukherjeeUllah2018_arXiv,
         title = {{Convergence guarantees for RMSProp and ADAM in non-convex optimization and an empirical comparison to Nesterov acceleration}},
        author = {Soham De and Anirbit Mukherjee and Enayat Ullah},
          year = {2018},
        eprint = {arXiv:1807.06766},
 archivePrefix = {arXiv},
  primaryClass = {cs.LG},
       journal = {\href{https://arxiv.org/abs/1807.06766}{arXiv:1807.06766}}
}

@article{Dereichetal2025_Adam_symmetry_theorem_arXiv,
         title = {{Adam symmetry theorem: characterization of the convergence of the stochastic Adam optimizer}},
        author = {Steffen Dereich and Thang Do and Arnulf Jentzen and Philippe von Wurstemberger},
          year = {2025},
        eprint = {arXiv:2511.06675},
 archivePrefix = {arXiv},
  primaryClass = {math.OC},
       journal = {\href{https://arxiv.org/abs/2511.06675}{arXiv:2511.06675}}
}

@article{DereichGraeberJentzen2024arXiv_non_convergence,
         title = {{Non-convergence of Adam and other adaptive stochastic gradient descent optimization methods for non-vanishing learning rates}},
        author = {Steffen Dereich and Robin Graeber and Arnulf Jentzen},
          year = {2024},
        eprint = {arXiv:2407.08100},
 archivePrefix = {arXiv},
  primaryClass = {cs.LG},
       journal = {\href{https://arxiv.org/abs/2407.08100}{arXiv:2407.08100}}
}

@article{MukkamalaHein2017_arXiv,
         title = {{Variants of RMSProp and Adagrad with Logarithmic Regret Bounds}},
        author = {Mahesh Chandra Mukkamala and Matthias Hein},
          year = {2017},
        eprint = {arXiv:1706.05507},
 archivePrefix = {arXiv},
  primaryClass = {cs.LG},
       journal = {\href{https://arxiv.org/abs/1706.05507}{arXiv:1706.05507}}
}

@article{HeilmanMohanty2026_arXiv_non_convergence_Adam,
         title = {{On the Convergence of Adam, Revisited}},
        author = {Steven Heilman and Sampad Mohanty},
          year = {2026},
        eprint = {arXiv:2607.03519},
 archivePrefix = {arXiv},
  primaryClass = {cs.LG},
       journal = {\href{https://arxiv.org/abs/2607.03519}{arXiv:2607.03519}}
}

@article{DoJentzenRiekert2025_arXiv,
         title = {{Non-convergence to the optimal risk for Adam and stochastic gradient descent optimization in the training of deep neural networks}},
        author = {Thang Do and Arnulf Jentzen and Adrian Riekert},
          year = {2025},
        eprint = {arXiv:2503.01660},
 archivePrefix = {arXiv},
  primaryClass = {math.OC},
       journal = {\href{https://arxiv.org/abs/2503.01660}{arXiv:2503.01660}}
}

@article{IbragimovJentzen2026arXiv,
         title = {{Unified convergence analysis for gradient descent optimization methods in the training of deep neural networks}},
        author = {Shokhrukh Ibragimov and Arnulf Jentzen},
          year = {2026},
        eprint = {arXiv:2607.04233},
 archivePrefix = {arXiv},
  primaryClass = {math.OC},
       journal = {\href{https://arxiv.org/abs/2607.04233}{arXiv:2607.04233}}
}

@article{DereichAdamconvergence2024,
         title = {{Convergence rates for the Adam optimizer}},
        author = {Steffen Dereich and Arnulf Jentzen},
          year = {2024},
        eprint = {arXiv:2407.21078},
 archivePrefix = {arXiv},
  primaryClass = {math.OC},
       journal = {\href{https://arxiv.org/abs/2407.21078}{arXiv:2407.21078}}
}

@article{Ruder2016arXiv,
         title = {{An overview of gradient descent optimization algorithms}},
        author = {Sebastian Ruder},
          year = {2016},
        eprint = {arXiv:1609.04747},
 archivePrefix = {arXiv},
  primaryClass = {cs.LG},
       journal = {\href{https://arxiv.org/abs/1609.04747}{arXiv:1609.04747}}
}

@article{DereichGraeberJentzenRiekert2026_arXiv,
         title = {{Asymptotic stability properties and a priori bounds for Adam and other gradient descent optimization methods}},
        author = {Steffen Dereich and Robin Graeber and Arnulf Jentzen and Adrian Riekert},
          year = {2026},
        eprint = {arXiv:2509.10476},
 archivePrefix = {arXiv},
  primaryClass = {math.OC},
       journal = {\href{https://arxiv.org/abs/2509.10476}{arXiv:2509.10476}}
}

@article{DereichDoJentzen2026arXiv,
         title = {{Uniform a priori bounds and error analysis for the Adam stochastic gradient descent optimization method}},
        author = {Steffen Dereich and Thang Do and Arnulf Jentzen},
          year = {2026},
        eprint = {arXiv:2603.18899},
 archivePrefix = {arXiv},
  primaryClass = {cs.LG},
       journal = {\href{https://arxiv.org/abs/2603.18899}{arXiv:2603.18899}}
}

@article{KingmaBa2014_Adam,
         title = {{Adam: A Method for Stochastic Optimization}},
        author = {Diederik P. Kingma and Jimmy Ba},
          year = {2014},
        eprint = {arXiv:1412.6980},
 archivePrefix = {arXiv},
  primaryClass = {cs.LG},
       journal = {\href{https://arxiv.org/abs/1412.6980}{arXiv:1412.6980}}
}

@article{LoshchilovHutter2017_arXiv,
         title = {{Decoupled Weight Decay Regularization}},
        author = {Ilya Loshchilov and Frank Hutter},
          year = {2017},
        eprint = {arXiv:1711.05101},
 archivePrefix = {arXiv},
  primaryClass = {cs.LG},
       journal = {\href{https://arxiv.org/abs/1711.05101}{arXiv:1711.05101}}
}

@article {MR4055054,
    AUTHOR = {Jentzen, Arnulf and von Wurstemberger, Philippe},
     TITLE = {{Lower error bounds for the stochastic gradient descent optimization algorithm: sharp convergence rates for slowly and fast decaying learning rates}},
   JOURNAL = {J. Complexity},
  FJOURNAL = {Journal of Complexity},
    VOLUME = {57},
      YEAR = {2020},
     PAGES = {101438, 16},
      ISSN = {0885-064X,1090-2708},
   MRCLASS = {68T05 (68Q17)},
  MRNUMBER = {4055054},
       DOI = {10.1016/j.jco.2019.101438},
       URL = {https://doi.org/10.1016/j.jco.2019.101438},
}

@article{Wangetal2023_arXiv,
         title = {{Closing the Gap Between the Upper Bound and the Lower Bound of Adam's Iteration Complexity}},
        author = {Bohan Wang and Jingwen Fu and Huishuai Zhang and Nanning Zheng and Wei Chen},
          year = {2023},
        eprint = {arXiv:2310.17998},
 archivePrefix = {arXiv},
  primaryClass = {cs.LG},
       journal = {\href{https://arxiv.org/abs/2310.17998}{arXiv:2310.17998}},
}

@article{JentzenBookDeepLearning2023,
         title = {{Mathematical Introduction to Deep Learning: Methods, Implementations, and Theory}},
        author = {Jentzen, Arnulf and Kuckuck, Benno and von Wurstemberger, Philippe},
          year = {2023},
        eprint = {arXiv:2310.20360},
 archivePrefix = {arXiv},
  primaryClass = {cs.LG},
       journal = {\href{https://arxiv.org/abs/2310.20360}{arXiv:2310.20360}},
}

@article{Zhangetal2022_arXiv,
         title = {{Adam Can Converge Without Any Modification On Update Rules}},
        author = {Yushun Zhang and Congliang Chen and Naichen Shi and Ruoyu Sun and Zhi-Quan Luo},
          year = {2022},
        eprint = {arXiv:2208.09632},
 archivePrefix = {arXiv},
  primaryClass = {cs.LG},
       journal = {\href{https://arxiv.org/abs/2208.09632}{arXiv:2208.09632}},
}

@article {BarakatBianchi2021_MR4199255,
    AUTHOR = {Barakat, Anas and Bianchi, Pascal},
     TITLE = {Convergence and dynamical behavior of the {A}dam algorithm for
              nonconvex stochastic optimization},
   JOURNAL = {SIAM J. Optim.},
  FJOURNAL = {SIAM Journal on Optimization},
    VOLUME = {31},
      YEAR = {2021},
    NUMBER = {1},
     PAGES = {244--274},
      ISSN = {1052-6234,1095-7189},
   MRCLASS = {90C15 (34A12 34D05 37C60 62L20 65K05)},
  MRNUMBER = {4199255},
MRREVIEWER = {Shih-Kang\ Chao},
       DOI = {10.1137/19M1263443},
       URL = {https://doi.org/10.1137/19M1263443},
}

@article{TielemanHinton2012,
         title = {{Lecture 6.5---RMSProp: Divide the gradient by a running average of its recent magnitude}},
        author = {Tijmen Tieleman and Geoffrey Hinton},
          year = {2012},
       journal = {COURSERA: Neural Networks for Machine Learning},
        volume = {4},
        number = {2},
         pages = {26--31},
           url = {https://www.cs.toronto.edu/~tijmen/csc321/slides/lecture_slides_lec6.pdf}
}

@article{GadatGavra2022_MR4577667,
    AUTHOR = {Gadat, S\'ebastien and Gavra, Ioana},
     TITLE = {Asymptotic study of stochastic adaptive algorithms in
              non-convex landscape},
   JOURNAL = {J. Mach. Learn. Res.},
  FJOURNAL = {Journal of Machine Learning Research (JMLR)},
    VOLUME = {23},
      YEAR = {2022},
     PAGES = {Paper No. [228], 54},
      ISSN = {1532-4435,1533-7928},
   MRCLASS = {90C15 (65C20 65K10)},
  MRNUMBER = {4577667},
}

@article{ZhouChenCaoYangGu2024,
         title = {{On the Convergence of Adaptive Gradient Methods for Nonconvex Optimization}},
        author = {Dongruo Zhou and Jinghui Chen and Yuan Cao and Ziyan Yang and Quanquan Gu},
          year = {2024},
       journal = {Transactions on Machine Learning Research},
           url = {https://openreview.net/forum?id=Gh0cxhbz3c}
}

@inproceedings{ZaheerReddiSachanKaleKumar2018,
         title = {{Adaptive Methods for Nonconvex Optimization}},
        author = {Manzil Zaheer and Sashank Reddi and Devendra Sachan and Satyen Kale and Sanjiv Kumar},
          year = {2018},
     booktitle = {Advances in Neural Information Processing Systems},
        volume = {31},
         pages = {},
           url = {https://proceedings.neurips.cc/paper/2018/hash/90365351ccc7437a1309dc64e4db32a3-Abstract.html}
}

@article{JinWang2026_arXiv,
         title = {{Asymptotic Convergence and Stability of Adaptive Gradient Methods in Smooth Non-convex Optimization}},
        author = {Ruinan Jin and Xiaoyu Wang},
          year = {2026},
        eprint = {arXiv:2601.01853},
 archivePrefix = {arXiv},
  primaryClass = {math.OC},
       journal = {\href{https://arxiv.org/abs/2601.01853}{arXiv:2601.01853}}
}

@article{MalladiLyuPanigrahiArora2022_arXiv,
         title = {{On the SDEs and Scaling Rules for Adaptive Gradient Algorithms}},
        author = {Sadhika Malladi and Kaifeng Lyu and Abhishek Panigrahi and Sanjeev Arora},
          year = {2022},
        eprint = {arXiv:2205.10287},
 archivePrefix = {arXiv},
  primaryClass = {cs.LG},
       journal = {\href{https://arxiv.org/abs/2205.10287}{arXiv:2205.10287}}
}

@article {MR2825422,
    AUTHOR = {Duchi, John and Hazan, Elad and Singer, Yoram},
     TITLE = {Adaptive subgradient methods for online learning and
              stochastic optimization},
   JOURNAL = {J. Mach. Learn. Res.},
  FJOURNAL = {Journal of Machine Learning Research (JMLR)},
    VOLUME = {12},
      YEAR = {2011},
     PAGES = {2121--2159},
      ISSN = {1532-4435,1533-7928},
   MRCLASS = {68T05 (62G08 90C15 90C25)},
  MRNUMBER = {2825422},
}

@article{LiRakhlinJadbabaie2024_arXiv,
  title={{Convergence of Adam Under Relaxed Assumptions}},
  author={Haochuan Li and Alexander Rakhlin and Ali Jadbabaie},
  eprint = {arXiv:2304.13972},
  archivePrefix = {arXiv},
  primaryClass = {math.OC},
  journal={\href{https://arxiv.org/abs/2304.13972}{arXiv:2304.13972}},
  year={2024}
}

@article{ReddiKale2019,
  title={{On the Convergence of Adam and Beyond}},
  author={Reddi, Sashank J. and Kale, Satyen and Kumar, Sanjiv},
  eprint = {arXiv:1904.09237},
  archivePrefix = {arXiv},
  primaryClass = {cs.LG},
  journal={\href{https://arxiv.org/abs/1904.09237}{arXiv:1904.09237}},
  year={2019}
}

@inproceedings{ChenLiuSunHong2019,
         title = {{On the Convergence of A Class of Adam-Type Algorithms for Non-Convex Optimization}},
        author = {Xiangyi Chen and Sijia Liu and Ruoyu Sun and Mingyi Hong},
          year = {2019},
     booktitle = {International Conference on Learning Representations (ICLR)},
           url = {https://openreview.net/forum?id=H1x-x309tm}
}

@inproceedings{HongLin2024_Adam,
         title = {{On Convergence of Adam for Stochastic Optimization under Relaxed Assumptions}},
        author = {Yusu Hong and Junhong Lin},
          year = {2024},
     booktitle = {Advances in Neural Information Processing Systems},
        volume = {37},
         pages = {10827--10877},
           doi = {10.52202/079017-0346},
           url = {https://proceedings.neurips.cc/paper_files/paper/2024/hash/14bb27f680bee45d83bc769738e7f9b5-Abstract-Conference.html}
}

@article{DereichJentzenRiekert2025_arXiv,
         title = {{Sharp higher order convergence rates for the Adam optimizer}},
        author = {Steffen Dereich and Arnulf Jentzen and Adrian Riekert},
          year = {2025},
        eprint = {arXiv:2504.19426},
 archivePrefix = {arXiv},
  primaryClass = {math.OC},
       journal = {\href{https://arxiv.org/abs/2504.19426}{arXiv:2504.19426}}
}
\end{document}